\documentclass{article}
\usepackage{graphicx} 
\usepackage{amssymb, amsthm, mathtools}
\usepackage{hyperref}
\usepackage{enumitem}
\usepackage[a4paper, margin=1in]{geometry}  
\usepackage{amsthm}
\theoremstyle{plain}

\usepackage{xcolor} 
\usepackage{algorithm}
\usepackage{algorithmic}
\usepackage{float}
\usepackage{amsmath}
\usepackage{cleveref}

\theoremstyle{plain}
\newtheorem{theorem}{Theorem}[section]

\newtheorem{lemma}[theorem]{Lemma}
\newtheorem{corollary}[theorem]{Corollary}
\theoremstyle{definition}

\newtheorem{assumption}[theorem]{Assumption}
\theoremstyle{remark}
\newtheorem{remark}[theorem]{Remark}

\usepackage[round,authoryear]{natbib}
\setcitestyle{maxcitenames=2,mincitenames=1}

\usepackage{authblk}

\title{Provable Last-Iterate Convergence for Multi-Objective Safe LLM Alignment via Optimistic Primal-Dual}
\author[1]{Yining Li}
\author[2]{Peizhong Ju}
\author[1]{Ness Shroff}

\affil[1]{The Ohio State University}
\affil[2]{University of Kentucky}

\date{}
\begin{document}
\maketitle

\newcommand{\1}{\mathbm{1}}
\newcommand{\pp}[1]{\left[#1\right]_+} 

\newcommand{\ip}[2]{\left\langle #1,\, #2 \right\rangle}
\newcommand{\D}[2]{D_{#1}\!\left(#2\right)} 

\newcommand{\piref}{\pi_{\mathrm{ref}}}
\newcommand{\pihat}{\hat{\pi}}
\newcommand{\pistar}{\pi^\star}
\newcommand{\X}{\mathcal{X}}
\newcommand{\A}{\mathcal{A}}
\newcommand{\PiSet}{\Pi}
\newcommand{\PisubSet}{\Pi_{\Theta}}
\newcommand{\Real}{\mathbb{R}}

\newcommand{\KL}{\mathrm{KL}}
\newcommand{\E}{\mathbb{E}}
\newcommand{\Var}{Var}
\newcommand{\argmin}{\mathrm{argmin}}
\newcommand{\argmax}{\mathrm{argmax}}
\newcommand{\tref}{\mathrm{ref}}
\newcommand{\old}{\mathrm{old}}
\newcommand{\new}{\mathrm{new}}
\newcommand{\bias}{\mathrm{bias}}
\newcommand{\stat}{\mathrm{stat}}
\newcommand{\gap}{\mathrm{gap}}

\newcommand{\eff}{\mathrm{eff}}
\newcommand{\supp}{\mathrm{supp}}
\newcommand{\proj}{\mathrm{Proj}}

\begin{abstract}
Reinforcement Learning from Human Feedback (RLHF) plays a significant role in aligning Large Language Models (LLMs) with human preferences.
While RLHF with expected reward constraints can be formulated as a primal–dual optimization problem, standard primal–dual methods only guarantee the convergence with a distributional policy where the saddle-point problem is in the convex-concave form. 
Moreover, standard primal-dual methods may exhibit instability or divergence in the last iterations under policy parameterization in practical applications.
In this work, we propose a universal primal–dual framework for safe RLHF that unifies a broad class of existing alignment algorithms, including safe-RLHF, one-shot, and multi-shot based methods. 
Building on the universal primal-dual framework, we introduce an optimistic primal–dual (OPD) algorithm that incorporates predictive updates for both primal and dual variables to stabilize saddle-point dynamics. 
We establish last-iterate convergence guarantees for the proposed method, covering both exact policy optimization in the distributional space and convergence to the neighborhood of the optimal solution whose gap is related to approximate error and bias with parameterized policies. 
Our analysis reveals that optimism plays a crucial role in mitigating the oscillations inherent to constrained alignment objectives, thereby closing a key theoretical gap between constrained RL and practical RLHF.

\end{abstract}


\section{Introduction}

The unsafe behaviors of large language models (LLMs) have raised growing concerns about the need to align safe and useful models.
Although LLMs have shown impressive performance across a wide range of language tasks, such as summarization~\cite{zhang2024comprehensive}, translation~\cite{elshin2024general}, and code generation~\cite{wang2023review}, they can also exhibit harmful behaviors, including generating misleading or incorrect information~\cite{guerreiro2023hallucinations,zhang2025siren}, producing inappropriate or toxic content~\cite{wen2023unveiling}, and leaking sensitive or private data~\cite{feretzakis2024trustworthy}.
As a result, aligning LLMs with human preferences that jointly emphasize helpfulness and safety has become a critical challenge.

In practice, preferences involve multiple attributes, such as helpfulness, conciseness, factuality, and harmlessness, and these attributes are often not perfectly aligned and can even conflict with one another~\cite{sorensen2024roadmap}. 
However, standard Reinforcement Learning from Human Feedback (RLHF) is inherently single-objective and does not fully capture the complexity of human preferences~\cite{ziegler2019fine,stiennon2020learning}. 
In its typical form, RLHF aligns a language model by collecting pairwise comparisons from human annotators, learning a reward model that reflects these preferences, and then optimizing the model to maximize the learned reward.
This observation naturally motivates constrained RLHF, where the model is optimized for helpfulness while explicitly enforcing safety-related constraints. For example, ~\cite{dai2024safe,huang2022constrained,du2025primal} study how to  maximize the helpfulness reward while requiring the expected safety cost to stay below a predefined threshold.

We propose a universal framework that unifies a broad class of constrained RLHF algorithms based on Lagrangian relaxation~\cite{dai2024safe,huang2024one,zhang2025alignment}. These methods formulate constrained RLHF as a saddle-point problem over a policy and a set of non-negative dual variables.
Our framework unifies these approaches by explicitly characterizing how different algorithms (i) approximately solve the primal policy optimization problem induced by the current dual variables, and (ii) update the dual variables using feedback from constraint violations.

This unified perspective highlights several algorithmic characteristics of existing approaches, including the convergence behavior of primal--dual methods and the computational requirements of one-shot and multi-shot procedures.
In practice, naive primal--dual updates~\cite{dai2024safe} can lead to unstable saddle-point dynamics.  Even in simple bilinear saddle-point problems, simultaneous primal--dual updates fail to converge in the last iterate and guarantee only average convergence, meaning that optimality is achieved only when averaging over iterates. 
This is often insufficient in safe RLHF, where the deployed model corresponds to the last iterate of training.
Multi-shot methods can be computationally expensive, as they require repeatedly solving the primal policy optimization problem to near optimality for each dual update~\cite{zhang2025alignment}. 
Meanwhile, one-shot dualization-based approaches rely on choosing the closed-form solution in the space of distributional policies as the primal optimal policy~\cite{huang2024one}. While this assumption enables efficient dual optimization, it does not accurately reflect practical alignment settings, where policies are parameterized by large neural networks and the exact distributional optimum may be unattainable.
These observations bring up an open question:
\textit{Is it possible to design an iterative alignment algorithm for constrained RLHF that is both computationally practical and provably stable in the last iterate, without relying on one-shot dualization or inner-loop optimal policy solvers?}

To address the stability issue, we propose an optimistic primal--dual method for safe RLHF.
Optimistic primal--dual methods are known to stabilize saddle-point dynamics and admit last-iterate convergence guarantees~\cite{ding2023last}. Motivated by this observation, we propose an optimistic safe RLHF algorithm that augments both primal and dual updates with optimistic steps. These steps predict future gradients, and the final updates are obtained by correcting the predicted trajectories, leading to more stable training and improved last-iterate performance.

Our main contributions include two parts. 
First, building on the proposed unified primal--dual framework, we develop an optimistic primal--dual algorithm for safe RLHF. By incorporating optimistic updates for both the policy and the dual variables, the proposed method stabilizes saddle-point dynamics and mitigates the oscillatory behavior in the constrained alignment problems. 
Second, we establish theoretical guarantees for the proposed optimistic primal--dual algorithm. In the distributional policy space, we show that the optimistic primal--dual method achieves last-iterate convergence to an optimal solution. We further extend the analysis to parameterized policy spaces relevant to practical LLM alignment, where we prove that the last iterate converges to a neighborhood of the optimal solution. The resulting residual error is explicitly characterized in terms of statistical estimation error and parameterization bias.

\section{Preliminaries on Constrained RLHF}
\subsection{Constrained RLHF Problem}
To align with human preferences that involve multiple, potentially conflicting objectives, multi-objective or constrained variants of RLHF have been widely studied, where alignment is performed with respect to a primary objective while additional preference dimensions are enforced via constraints.
These variants largely follow the standard RLHF pipeline, which consists of supervised fine-tuning (SFT) to obtain a reference policy, learning reward models from human preference data, and reinforcement-learning-based policy optimization with KL regularization to the reference model~\cite{ziegler2019fine,stiennon2020learning}.

Let $\mathcal{X}$ and $\mathcal{Y}$ denote the sets of prompts and responses, respectively. 
A language model is represented as a stochastic policy mapping from the prompt set $\mathcal{X}$ to the distribution on the response set $\mathcal{Y}$, denoted as $\pi:\mathcal{X}\to\Delta(\mathcal{Y})$, where $\Delta(\mathcal{Y})$ is the set of all distributions on $\mathcal{Y}$.  Denote $\pi_{\tref}$ as the reference policy obtained after SFT.
We consider multiple preference objectives indexed by $\mathcal{K}=\mathcal{S}\cup\mathcal{H}$, where objectives in $\mathcal{S}$ are optimized and those in $\mathcal{H}$ are enforced via constraints.

To learn a reward model for objective $k$, we assume access to a human preference dataset
$\{(x_k^i, y_k^{i,w}, y_k^{i,l})\}_{i=1}^N$,
where $x_k^i$ is a prompt and $(y_k^{i,w}, y_k^{i,l})$ denotes a preferred (indicated by superscript $^w$) and less preferred (indicated by superscript $^l$) response pair annotated by human annotators.
Following standard practice in RLHF, we assume that preferences are generated according to a latent reward function $R_k^*(x,y)$, and that human comparisons follow the Bradley-Terry model~\cite{ouyang2022training}:
\[
P(y^w \succ y^l | x)
= \sigma\left(R_k^*\left(x,y^w\right)-R_k^*\left(x,y^l\right)\right),
\]
where $\sigma(\cdot)$ denotes the sigmoid function.
The reward model is then estimated by maximum likelihood over the preference dataset,
\[
{R}_k
= \arg\max_{R}
\sum_{i=1}^N
\log \sigma\left(R\left(x_k^i,y_k^{i,w}\right)-R\left(x_k^i,y_k^{i,l}\right)\right).
\]

For the constrained objectives indexed by $\mathcal{H}$, we specify a vector of thresholds
$\mathbf{b}=(b_j)_{j\in\mathcal{H}}$, which define minimum performance requirements.
For notational convenience, we absorb the thresholds into the reward definitions by introducing shifted rewards
$\tilde R_j(x,y):= R_j(x,y) - b_j$,
and with a slight abuse of notation, we continue to denote the shifted rewards by $R_j$.
The goal of RLHF is to optimize a policy with respect to the learned reward signals while regularizing it to remain close to a reference policy.
Given a preference weight vector $\mathbf{w} \in \mathbb{R}_+^{|\mathcal{S}|}$ such that
$\sum_{k\in \mathcal{S}} w_k = 1$, encoding the user’s trade-off over the soft objectives in
$\mathcal{S}$, the resulting multi-objective RLHF problem is formulated as 
\[
\begin{aligned}
  \max_{\pi}\quad & \mathbb{E}_{x\sim\mathcal{D}}
\left[
\mathbb{E}_{y\sim\pi(\cdot| x)}
\left[\sum_{k\in\mathcal{S}} w_k R_k(x,y)\right]\right.\\
&\left.-\beta\mathrm{KL}\left(\pi(\cdot| x)\|\pi_{\tref}(\cdot| x)\right)
\right] \\
\text{s.t.}\quad 
& \mathbb{E}_{x\sim\mathcal{D},y\sim\pi(\cdot| x)}
\left[R_j(x,y)\right] \geq 0, \quad \forall j\in\mathcal{H},  
\end{aligned}
\]
where $\mathcal{D}$ denotes the prompt distribution and $\beta>0$ controls the weight of KL regularization to the reference policy $\pi_{\tref}$.

\subsection{Lagrangian Method}

A standard approach to solving constrained RLHF problems is the Lagrangian method.
For each constrained objective $j\in\mathcal{H}$, we introduce a nonnegative Lagrange multiplier $\lambda_j\ge 0$.
Given the preference weights $\mathbf{w}$ over the soft objectives, we define the aggregated reward
\begin{equation}
\label{eq:Sxy}
S_{\lambda}(x,y)
:= \sum_{k \in \mathcal{S}} w_k R_k(x,y)
   + \sum_{j \in \mathcal{H}} \lambda_j R_j(x,y).
\end{equation}
The resulting Lagrangian of the constrained multi-objective RLHF problem is
\[
\begin{aligned}
\mathcal{L}(\pi,\lambda)
=& \mathbb{E}_{x\sim\mathcal{D},\,y\sim\pi(\cdot| x)}
\!\left[S_{\lambda}(x,y)\right]\\
&- \beta\,\mathbb{E}_{x\sim\mathcal{D}}
\left[\mathrm{KL}\!\left(\pi(\cdot| x)\,\|\,\pi_{\tref}(\cdot| x)\right)\right].
\end{aligned}
\]
The corresponding saddle-point problem is
\begin{equation}
\label{eq:dual_problem}
\min_{\lambda\ge 0}\; \max_{\pi}\; \mathcal{L}(\pi,\lambda).
\end{equation}
When optimizing over the space of all stochastic policies, the objective is concave in $\pi$ and linear in $\lambda$, and the problem admits a convex--concave structure.
In this case, for any fixed $\lambda$, the optimal policy has a closed-form solution given by
\begin{equation}
\label{eq:opt_policy}
\pi^\star(y| x)= \pi_{\tref}(y| x)\exp\left(S_{\lambda}(x,y)/\beta\right)/Z(x),
\end{equation}
where $Z(x)$ is the normalization factor $Z(x)=\sum_{y}\pi_{\tref}(y|x)\exp\left(S_{\lambda}(x,y)/\beta\right)$.
Detailed derivations are provided in \cref{lemma:solve_RLHF_opt}.
In practice, however, the policy is restricted to a parameterized family $\{\pi_\theta\}_{\theta\in\Theta}$, under which the optimization becomes non-concave in $\theta$.
As a result, practical constrained RLHF algorithms typically rely on iterative primal--dual updates, alternating between approximate policy optimization for fixed $\lambda$ and gradient-based updates of the dual variables.
The convergence of such methods in the parameterized setting generally requires additional assumptions or specialized algorithmic designs.

Fixing the policy $\pi$, the Lagrangian is differentiable with respect to the dual variables.
The gradient of $\mathcal{L}(\pi,\lambda)$ with respect to $\lambda_j$ is given by
\[
\nabla_{\lambda_j} \mathcal{L}(\pi,\lambda)
= \mathbb{E}_{x\sim\mathcal{D},\,y\sim\pi(\cdot| x)}\big[R_j(x,y)\big].
\]
Accordingly, standard constrained RLHF methods update the dual variables by projected gradient descent,
i.e., moving $\lambda$ in the direction of constraint violation and projecting onto $\mathbb{R}_{\ge 0}$.

\subsection{A universal safe RLHF framework}
We propose a universal framework that unifies a broad class
of constrained RLHF algorithms based on Lagrangian re-
laxation. 
The detailed universal Lagrangian alignment framework is shown in \ref{alg:universal_rlhf_framework}. 
\begin{algorithm}[H]
\caption{Universal Lagrangian Alignment Framework}
\label{alg:universal_rlhf_framework}
\begin{algorithmic}[1]
\REQUIRE Prompt distribution $\mathcal{D}$; Reward models $\{R_k(x,y)\}_{k\in\mathcal{K}}$; soft weights $\{w_j\}_{j\in \mathcal{S}}$;
reference policy $\pi_{\tref}$; KL coefficient $\beta$; 
initial $\theta_0$, $\lambda_0\ge 0$; 

\FOR{$t=0,1,2,\dots,T-1$}
\STATE \textbf{Primal update:}
\STATE $\pi_{\theta_{t+1}} \leftarrow \textsc{PrimalOracle}(\pi_{\theta_{t}},\lambda_t,\pi_{\tref},\mathcal{D},\beta)$

\STATE \textbf{Dual update:}
\STATE $g_t \leftarrow \textsc{GradEst}(\theta_{t+1},\lambda_t,\pi_{\tref},\mathcal{D})$
\STATE $\lambda_{t+1} \leftarrow \left[\lambda_t - \frac{1}{\eta_\lambda} g_t\right]_{+}$

\ENDFOR
\STATE \textbf{Return} $\lambda_T$ and $\theta_T$. 
\end{algorithmic}
\end{algorithm}

For each iteration, the framework alternates between a primal update and a dual update.
(1) The primal update is abstracted as a \textsc{PrimalOracle}, which aims to maximize
the Lagrangian objective for a given dual variable.
Depending on the choice of the oracle, the primal step may correspond to a single-step or multi-step
policy gradient update in the parameter space, an approximate inner-iterations to solve the near-optimal policy in the parameterization space,
or an exact closed-form solution in the distribution space~\cite{huang2024one}.
(2) The dual update use \textsc{GradEst} estimates the expected rewards of the constrained objectives under the current policy, followed by a projected gradient step on the dual variable.

Different existing alignment methods can be recovered by instantiating the primal oracle
and the dual gradient estimator differently, as detailed below.
\begin{itemize}
    \item Finite-step primal--dual updates.
When the primal oracle performs a finite number of stochastic gradient steps, the algorithm reduces to the class of coupled primal--dual methods used in safe RLHF and constrained DPO~\cite{dai2024safe,du2025primal,liu2024enhancing}.
In this regime, the primal policy is updated by a small number of stochastic gradient steps under a non-stationary objective induced by the evolving dual variable.
These methods lack last-iterate convergence guarantees, even when the underlying Lagrangian is convex–concave in the distribution space.

\item Approximate multi-shot variants.
Some recent works decouple the optimization by introducing an outer-loop dual update and an inner-loop primal optimization that approximately maximizes the Lagrangian for a fixed dual variable~\cite{zhang2025alignment}.
While this reduces interference between primal and dual updates, the inner-loop problem remains non-convex in the parameter space and is only solved approximately, which prevents these methods from being interpreted as exact primal oracles.

\item Exact dualization and one-shot alignment.
In contrast, one-shot methods are obtained by analytically eliminating the primal variable in the distribution space, which yields an explicit, closed-form, and convex dual objective~\cite{huang2024one}.
They can be viewed as a degenerate instantiation of the universal framework: the primal oracle returns the closed-form optimal distributional policy for a given dual variable, so no iterative primal updates are required during dual optimization. Therefore, each iteration reduces to a pure dual update step.
\end{itemize}


\section{Optimistic Primal--Dual Method}

Standard primal--dual methods do not guarantee last-iterate convergence in constrained RLHF, and this fundamental limitation motivates the need for alternative primal--dual methods with stronger stability properties.
The universal framework in \cref{alg:universal_rlhf_framework} formulates constrained RLHF as a Lagrangian saddle-point problem, where the primal update optimizes the policy and the dual update adjusts the constraint multipliers.
When optimization is carried out over the distributional policy space, the KL regularization induces strong concavity in the primal variable.
However, the Lagrangian remains linear in the dual multipliers, and hence the resulting saddle-point problem is generally not strongly-convex-strongly-concave.
The gradient descent-ascent methods converge linearly to the unique saddle point only under smooth strongly-convex-strongly-concave conditions with appropriate step sizes~\cite{zamani2024convergence}.
Once these conditions are violated, such guarantees no longer hold, and last-iterate convergence may fail even when a unique saddle point exists.

\paragraph{Example: Failure of Last-Iterate Convergence in a Bilinear Saddle-Point Problem}
We consider a simple convex–concave bilinear problem $\min_{\mathbf{y}}\max_{\mathbf{x}}\mathbf{x}^T\mathbf{A}\mathbf{y}$, where $\mathbf{A}$ is a full-rank matrix whose singular values are $[\sigma_1,\cdots,\sigma_M]$. 
The standard primal--dual gradient method gives $\mathbf{x}_{t+1}=\mathbf{x}_t+\alpha \mathbf{A}\mathbf{y}_t$ and $\mathbf{y}_{t+1}=\mathbf{y}_t-\alpha \mathbf{A}^\top \mathbf{x}_t$, where $\alpha$ is the stepsize.
Let $z_t=[\mathbf{x}_t,\mathbf{y}_i]^\top$. Then the update can be written as a linear iteration
\[
z_{t+1}=(\mathbf{I}-\alpha \mathbf{J})z_t,\quad
\mathbf{J}=\begin{bmatrix}0&-\mathbf{A}\\ \mathbf{A}^\top&0\end{bmatrix}.
\]
The matrix $\mathbf{J}$ has imaginary eigenvalues $\pm i\sigma_i$. 
Hence, $(\mathbf{I}-\alpha \mathbf{J})$ has eigenvalues $1\pm i\alpha\sigma_i$ whose magnitudes are larger than $1$, 
implying that the last iterates do not contract toward the saddle point due to the saddle-point problem's inherently rotational structure~\cite {daskalakis2018limit}.

The aforementioned example implies that, even in constrained RLHF problems where the primal objective is strongly concave over the distributional policy space, standard primal--dual methods generally admit only average convergence guarantees and may fail to converge in the last iterate.
The situation becomes even more challenging in practical RLHF settings with parameterized policies, where the optimization problem is no longer convex in the policy parameters.

Motivated by these challenges, we adopt an optimistic primal--dual (OPD) method, which corrects each update using a prediction of the next-step gradient and is known to suppress the rotational dynamics which can cause oscillations.
In the following, we first analyze OPD in the distributional policy space and establish last-iterate convergence to the optimal primal--dual solution.
We then extend the analysis to parameterized policies, showing that the same guarantees hold up to approximation errors.

\subsection{OPD in Distribution Space}

\begin{algorithm}[!t]
\caption{OPD with Primal Distributional Policies}
\label{alg:distribution_policy}
\begin{algorithmic}[1]
\REQUIRE Prompt distribution $\mathcal{D}$; Reward models $\{R_k(x,y)\}_{k\in\mathcal{K}}$; soft weights $\{w_j\}_{j\in \mathcal{S}}$;
reference policy $\pi_{\tref}$; KL coefficient $\beta$; 
initial $\hat{\pi}_0$, $\lambda_0\ge 0$; 

\FOR{$t=0,1,2,\dots,T-1$}

\STATE Primal Optimistic Update:
\begin{equation}
    \label{eq:pi_t}
\begin{aligned}[t]
\pi_t
= \arg&\max_{\pi}
\left(\mathcal{L}(\pi,\lambda_{t-1})\right.\\
&\left.-\mathbb{E}_{x\sim \mathcal{D}}\left[\eta_{\theta}\mathrm{KL}\left(\pi(\cdot|x)\Vert\hat{\pi}_t(\cdot|x)\right)\right]\right),
\end{aligned}
\end{equation}

\STATE Dual Optimistic Update:
\begin{equation}
  \label{eq:lambda_t}
\lambda_{t}
= \arg\min_{\lambda\geq 0}
\begin{aligned}[t]
&\lambda\mathbb{E}_{x\sim \mathcal{D}, y\sim \pi_{t-1}(\cdot|x)}\left[R(x,y)\right] \\
&+\eta_{\lambda}(\lambda-\hat{\lambda}_{t})^2,
\end{aligned}
\end{equation}
\STATE Primal Actual Update:
\begin{equation}
   \label{eq:hatpi_t}
\begin{aligned}[t]
\hat{\pi}_{t+1}
= \arg&\max_{\pi}
\mathcal{L}(\pi,\lambda_{t})\\
&-\mathbb{E}_{x\sim \mathcal{D}}\left[\eta_{\theta}\mathrm{KL}\left(\pi(\cdot|x)\Vert\hat{\pi}_t(\cdot|x)\right)\right],
\end{aligned}
\end{equation}
\STATE Dual Actual Update:
\begin{equation}
 \label{eq:hatlambda_t}
\hat{\lambda}_{t+1}
= \arg\min_{\lambda\geq 0}
\begin{aligned}[t]
&\lambda\mathbb{E}_{x\sim \mathcal{D}, y\sim \pi_{t}(\cdot|x)}\left[R(x,y)\right] \\
&+\eta_{\lambda}(\lambda-\hat{\lambda}_{t})^2.
\end{aligned}
\end{equation}

\ENDFOR
\STATE \textbf{Return} $\hat{\lambda}_T$ and $\hat{\pi}_T$. 
\end{algorithmic}
\end{algorithm}

OPD update in distribution space is shown in \cref{eq:pi_t,eq:lambda_t,eq:hatpi_t,eq:hatlambda_t} of \cref{alg:distribution_policy}.
OPD introduces predictive iterates $(\pi_t, \lambda_t)$ to approximate the next-step primal and dual variables.
The actual updates $(\hat{\pi}_{t+1}, \hat{\lambda}_{t+1})$ are then corrected based on these predictions.

We make the following assumptions.
\cref{assump:feasibility} corresponds to Slater’s condition, which assumes the existence of a strictly feasible policy and guarantees strong duality, i.e., the existence of the optimal saddle point.
Slater’s condition is standard in the analysis of constrained optimization and primal--dual methods~\cite{huang2024one,zhang2025alignment,du2025primal}.
\cref{assump:bounded} assumes that all reward models are uniformly bounded, which is a common condition in the RLHF literature~\cite{du2025primal}.
\cref{assump:reference_policy_full_support} requires the reference policy to assign nonzero probability to every feasible action.
For LLM policies parameterized by softmax distributions, token probabilities are strictly positive over the modeled action set. 
When action masking or filtering is applied, we equivalently redefine the action space as the accessible set and require the reference policy to have full support on this restricted space.

\begin{assumption}[Slater's condition]
\label{assump:feasibility}
    There exists a policy $\Bar{\pi}\in\Pi$ and a constant $\xi>0$ such that $\mathbb{E}_{x\sim \mathcal{D},y\sim\Bar{\pi}}\left[R_{j}(x,y)\right]\geq \xi$, $\forall j\in\mathcal{H}$.
\end{assumption}

\begin{assumption}[Bounded rewards]
\label{assump:bounded}
There exists $R_{\max}>0$ such that $ \big|R_k(x,y)\big|\le R_{\max}$ for all $k\in\mathcal{K}$, $x\in\mathcal{X}$, and $y\in\mathcal{Y}$.
\end{assumption}

\begin{assumption}[Full support of the reference policy]
\label{assump:reference_policy_full_support}
    Assume the reference policy has the full support, i.e.,  there exists $p_{\min}>0$ such that $\pi_{\tref}(y|x)\geq p_{\min}$ for any $(x,y)$ pair.
\end{assumption}

\cref{assump:reference_policy_full_support} ensures that policy supports do not collapse along the OPD iterates and that all KL divergence terms remain well-defined throughout optimization.
We initialize $\hat{\pi}_0=\pi_{\tref}$, and all subsequent policy updates are obtained via KL-regularized maximization.
Hence,  the support of $\hat{\pi}_t(\cdot|x)$ remains contained within that of the reference policy for all $t$.
Moreover, any optimal policy $\pi^\star$ satisfying the constraints is covered by the reference support and by the supports of the OPD iterates.
This assumption prevents premature elimination of feasible actions and guarantees that OPD operates over a policy class that contains the optimal solution.

\begin{theorem}
\label{thm:reg_opt_LIC}
Under \cref{assump:feasibility}, \cref{assump:bounded}, and \cref{assump:reference_policy_full_support}, under suitably chosen hyper-parameters $\eta_{\theta}$ and $\eta_{\lambda}$ (e.g., $\eta_{\theta}=\eta_{\lambda}=3\sqrt{|\mathcal{H}|}R_{\max}$), then the optimistic primal--dual iterates of \cref{eq:pi_t,eq:lambda_t,eq:hatpi_t,eq:hatlambda_t} satisfy
\[
\begin{aligned}
&\E_{x\sim\mathcal{D}}\left[\KL(\pi_{\theta^{\star}}(\cdot|x)\|\hat{\pi}_{t}(\cdot|x))\right]
+\|\boldsymbol{\lambda}^{\star}-\hat{\boldsymbol{\lambda}}_{t}\|_2^2\\
&\leq \rho^{t-1}\frac{\Phi_1}{\min\left(\eta_{\theta}+\beta,\frac{7}{4}\eta_{\lambda}-\frac{3}{4}\sqrt{|\mathcal{H}|} R_{\max}\right)},
\end{aligned}
\]
where $0<\rho<1$ is defined in \cref{eq:rho_valued} and $\Phi_1$ is a costant defined as \cref{eq:Phi_1_valued}.
\end{theorem}

\cref{thm:reg_opt_LIC} establishes a linear last-iterate convergence guarantee for OPD in the policy distribution space. The final iterates $(\hat{\pi}_t,\hat{\lambda}_t)$ converge linearly toward the optimal saddle point $(\pi^\star,\lambda^\star)$ at rate $\rho<1$, as measured by the KL divergence in the primal variable and the squared $\ell_2$ error in the dual variable.
In contrast to standard primal--dual methods that typically only ensure ergodic convergence, this result provides direct control over the final policy iterate, which is particularly important in safe RLHF, where constraint satisfaction and alignment quality are evaluated on the deployed policy rather than on an average of iterates.
Moreover, the result holds under a linear dual objective and without strong convexity in the dual variable, highlighting the stabilizing effect of the optimistic primal--dual updates.

\subsection{OPD in Parameter Space}

\subsubsection{OPD Updates in the Parameterized Policy Space}
In the parameterized policy space, where the policy is represented as $\pi_\theta$ with parameters $\theta\in \Theta$, the resulting Lagrangian optimization problem is generally non-convex and the closed-form distributional updates in \cref{eq:pi_t,eq:hatpi_t} are no longer tractable.
We therefore adopt a gradient-based optimistic primal--dual method in the parameter space.

We denote the parameterized counterparts of $\pi_t$ and $\hat{\pi}_t$ by $\pi_{\theta_t}$ and $\pi_{\hat{\theta}_t}$, respectively.
To obtain a tractable update consistent with the distributional formulation shown in \cref{eq:pi_t,eq:hatpi_t}, we approximate the KL divergence by its second-order Taylor expansion around $\hat{\theta}_t$.
Specifically, when $\theta_t$ is sufficiently close to $\hat{\theta}_t$, we have
\[
\begin{aligned}
&\E_{x\sim \mathcal{D}}
\left[\KL\left(\pi_{\theta_t}(\cdot|x)\|\pi_{\hat{\theta}_t}(\cdot|x)\right)\right]\\
&\qquad\approx
\frac{1}{2}
(\theta_t-\hat{\theta}_t)^\top
F(\hat{\theta}_t)
(\theta_t-\hat{\theta}_t), 
\end{aligned}
\]
where $F(\theta)$ denotes the Fisher information matrix,
\[
F(\theta)
=
\E_{x\sim \mathcal{D},y\sim\pi_{\theta}(\cdot|x)}
\left[
\nabla_\theta \log \pi_{\theta}(y|x)
\nabla_\theta \log \pi_{\theta}(y|x)^\top
\right].
\]
To accommodate possible rank deficiency, we use the Moore-Penrose pseudo-inverse $F(\theta)^\dagger$.
Under this local approximation, the distributional OPD updates reduce to natural policy gradient (NPG) steps in the parameter space.
To ensure feasibility in the parameter domain, we project the updated parameters back onto the parameter space $\Theta$ after each primal update.

Similar to \cref{alg:distribution_policy}, the proposed method first performs optimistic primal and dual updates to predict the next-step policy parameters and dual variables, as shown in \cref{eq:npg_thetat,eq:npg_lambdat}.
The actual primal and dual updates are then carried out using these predictions, as specified in \cref{eq:npg_hatthetat+1,eq:npg_hatlambdat+1}.
The complete OPD procedure in the parameterized policy space is summarized in \cref{alg:parameterized_policy}.

The parameterized OPD applies optimism asymmetrically across the primal and dual variables.
In particular, the predicted policy iterate $\pi_{\theta_t}$ is only used to form the dual updates $\lambda_t$ and $\hat{\lambda}_{t+1}$, whereas the actual policy $\pi_{\hat{\theta}_t}$ is used for the primal updates $\theta_t$ and $\hat{\theta}_{t+1}$.
This asymmetric design ensures that policy-gradient computations are performed only for the actual policy updates, while the predicted policy iterate $\pi_{\theta_t}$ is used solely for evaluation in the dual updates and does not require gradient computation.
In contrast, symmetric extragradient methods~\cite{ding2023last} typically require evaluating both primal and dual operators at the predictor iterate, resulting in higher computational cost and variance.

\begin{algorithm}[!t]
\caption{OPD with Primal Parameterized Policies}
\label{alg:parameterized_policy}
\begin{algorithmic}[1]
\REQUIRE Prompt distribution $\mathcal{D}$; Reward models $\{R_k(x,y)\}_{k\in\mathcal{K}}$; soft weights $\{w_j\}_{j\in \mathcal{S}}$;
reference policy $\pi_{\tref}$; KL coefficient $\beta$; 
initial $\hat{\theta}_0$, $\lambda_0\ge 0$; 

\FOR{$t=0,1,2,\dots,T-1$}

\STATE Primal Optimistic Update:
\begin{equation}
\label{eq:npg_thetat}
\theta_{t}
=\text{Proj}_{\Theta}\left(\hat{\theta}_t+\frac{1}{\eta_{\theta}+\beta}F(\hat{\theta}_t)^\dagger\nabla_\theta \mathcal{L}(\pi_{\hat{\theta}_t},\lambda_{t-1})\right).
\end{equation}
\STATE Dual Optimistic Update:
\begin{equation}
\label{eq:npg_lambdat}
\lambda_{t}
=\left[\hat{\lambda}_{t}-\frac{1}{\eta_{\lambda}}\E_{x\sim \mathcal{D},\, y\sim \pi_{\theta_{t-1}}(\cdot|x)}\left[R(x,y)\right]\right]_+.
\end{equation}
\STATE Primal Actual Update:
\begin{equation}
\label{eq:npg_hatthetat+1}
\hat{\theta}_{t+1}
=\text{Proj}_{\Theta}\left(\hat{\theta}_t+\frac{1}{\eta_{\theta}+\beta}F(\hat{\theta}_t)^\dagger\nabla_\theta \mathcal{L}(\pi_{\hat{\theta}_t},\lambda_{t})\right).
\end{equation}
\STATE Dual Actual Update:
\begin{equation}
\label{eq:npg_hatlambdat+1}
\hat{\lambda}_{t+1}
=\left[\hat{\lambda}_{t}-\frac{1}{\eta_{\lambda}}\E_{x\sim \mathcal{D},\, y\sim \pi_{\theta_t}(\cdot|x)}\left[R(x,y)\right]\right]_+.
\end{equation}

\ENDFOR
\STATE \textbf{Return} $\hat{\lambda}_T$ and $\pi_{\hat{\theta}_T}$. 
\end{algorithmic}
\end{algorithm}

\begin{remark}[Equivalence between Distribution-Space OPD and NPG Updates]
\label{prop:OPD-NPG-equivalence}
Under tabular softmax parameterization, the distribution-space OPD updates in
\cref{eq:pi_t,eq:lambda_t,eq:hatpi_t,eq:hatlambda_t}
are equivalent to their parameter-space counterparts in
\cref{eq:npg_thetat,eq:npg_lambdat,eq:npg_hatthetat+1,eq:npg_hatlambdat+1}.
In particular, for all $t$, the induced policies satisfy
$\pi_{\theta_t} = \pi_t$ and $\pi_{\hat{\theta}_{t+1}} = \hat{\pi}_{t+1}$.

The key observation is that, under tabular softmax parameterization, policy parameters $\theta$ are in one-to-one correspondence with policy distributions.
Moreover, KL-regularized optimization in the distribution space is exactly equivalent to mirror descent under the KL geometry, which corresponds to NPG updates in the parameter space.
As a result, the distribution-space OPD updates generate exactly the same sequence of policies as the NPG-based OPD updates.
\end{remark}

\begin{remark} [Relationship to PPO in Practice]
NPG controls policy updates by explicitly constraining the KL divergence between consecutive policies, while proximal policy optimization (PPO) enforces update stability by directly clipping the policy ratio.
Although the two approaches differ in their formulations, both can be interpreted as mechanisms for bounding policy updates and preventing overly aggressive policy changes.
In practice, PPO is often preferred due to its simplicity and empirical robustness, and the proposed OPD framework can be implemented using PPO-style clipped updates.
In this paper, we adopt the NPG formulation for analytical convenience, as it provides a clean connection to KL-regularized optimization and facilitates theoretical analysis.
In our experiments, we implement the proposed OPD framework using PPO-style updates.
\end{remark}

\subsubsection{A Toy RLHF Example Illustrating the Stability of OPD}

\begin{figure}[!t]
    \centering
    \includegraphics[width=\linewidth]{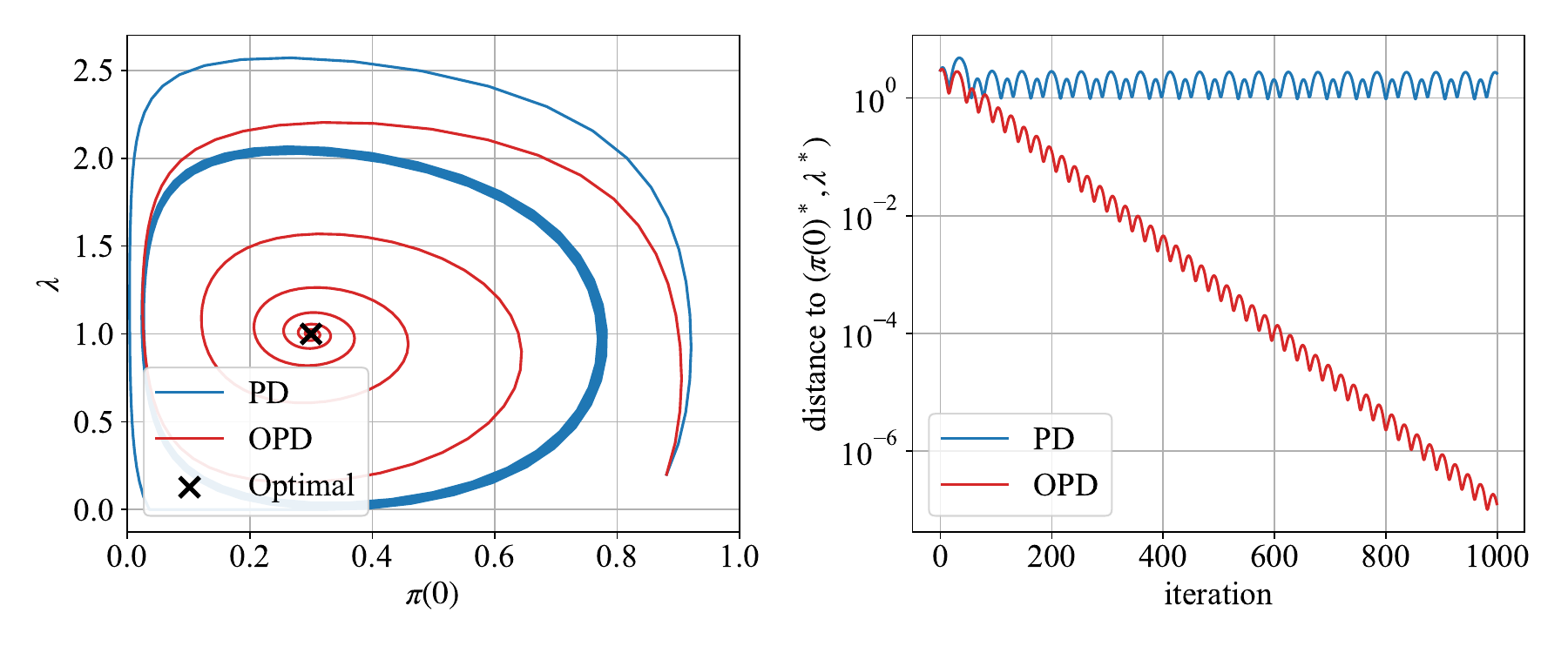}
    \caption{Comparison of OPD and PD under a softmax tabular parameterization in a single-state, two-action RLHF toy problem. OPD (red) converges to the optimal solution in the last iterate, while PD (blue) exhibits persistent oscillations and fails to converge.}
    \label{fig:tabular_toy_example}
\end{figure}

We consider a minimal RLHF-style constrained optimization problem with a single state ($|\mathcal{X}|=1$) and two actions ($|\mathcal{Y}|=2$), denoted by $y_0$ and $y_1$.
Since there is only one state, we omit the dependence on $x$ in the following.

We consider two reward models.
The first reward $R_s$ represents the objective to be maximized, while the second reward $R_h$ corresponds to a safety-related constraint.
We set $R_s(y_0)=1$ and $R_s(y_1)=0$, so that the expected reward under a policy $\pi$ is simply $\pi(y_0)$.
For the constraint reward, we choose $R_h(y_0)=-0.7$ and $R_h(y_1)=0.3$, which induces the constraint
$-0.7\,\pi(y_0) + 0.3\,(1-\pi(y_0)) \ge 0$, or equivalently $\pi(y_0)\le 0.3$.

We select the reference policy as $\pi_{\tref}(y_0)=0.3$.
The resulting optimization problem is to maximize
\[
\pi(y_0) - \beta\,\KL\!\left(\pi(\cdot)\,\middle\|\,\pi_{\tref}(\cdot)\right)
\quad \text{subject to} \quad \pi(y_0)\le 0.3,
\]
with $\beta=0.05$.
It is easy to verify that the optimal policy coincides with the reference policy $\pi_{\tref}$.

To avoid explicit projection onto the policy simplex, we adopt a softmax parameterization $\pi(y_0)=1/(1+\exp(\theta))$.
We set the effective primal stepsize $\alpha = (\eta_\theta+\beta)^{-1}=0.6$ and the dual stepsize $\eta_\lambda^{-1}=0.6$.
Figure~\ref{fig:tabular_toy_example} compares the trajectories of OPD and standard primal--dual updates under this parameterization.

As shown in \cref{fig:tabular_toy_example}, the proposed OPD method converges to the optimal saddle point, with the distance to the optimum decreasing linearly, consistent with the theoretical guarantees in \cref{thm:reg_opt_LIC}.
In contrast, the standard PD updates fail to converge and exhibit divergent behavior in this simple setting.

\subsubsection{Theoretical Results}

Let $\PisubSet$ denote the class of parameterized policies that have full support on the considered action set, i.e., there exists $p_{\min}>0$ such that
$\pi_\theta(y|x)\ge p_{\min}$ for all feasible $(x,y)$.
We further assume that the parameter domain $\Theta\subset\mathbb{R}^d$ is closed and convex, so that the projection operator $\text{Proj}_\Theta(\cdot)$ used in the updates is well-defined.

Since our analysis focuses on optimality within the parameterized policy class, we impose a Slater-type condition in the parameterized policy space.
\begin{assumption}[Slater's condition in the parameterized policy space]
\label{assump:feasibility_parameterized}
There exists a parameter vector $\bar{\theta}\in\Theta$ and a constant $\xi>0$ such that the corresponding policy $\pi_{\bar{\theta}}\in\Pi_{\Theta}$ satisfies
\[
\mathbb{E}_{x\sim \mathcal{D},y\sim\pi_{\bar{\theta}}(\cdot|x)}
\left[R_{j}(x,y)\right]\ge \xi,
\quad \forall j\in\mathcal{H}.
\]
\end{assumption}





\cref{assump:log_bdd} assumes that the log-policy is Lipschitz continuous with respect to the policy parameters.
This condition allows us to translate deviations in the parameter space into controlled changes in the induced policy distributions, and is particularly useful for bounding KL divergence and log-ratio terms that arise in the analysis.
Such an assumption is standard in the analysis of policy gradient and mirror descent methods with parameterized policies.

\begin{assumption}[Log-policy Lipschitz continuity]
\label{assump:log_bdd}
There exists a constant $C>0$ such that for any $\theta_1,\theta_2\in\Theta$,
\[
\E_{x\sim\mathcal{D},\,y\sim\mathcal{Y}(x)}
\left[\left|\log \pi_{\theta_1}(y|x)-\log \pi_{\theta_2}(y|x)\right|\right]
\le C \|\theta_1-\theta_2\|_1 .
\]
\end{assumption}

As the primal updates rely on stochastic gradient estimates and empirical Fisher information computed from finite samples, we make the following assumption to quantify the inexactness arises naturally in practice.
\begin{assumption}[Inexact primal updates]
\label{assump:approx}
Let $\theta_t^*$ and $\hat{\theta}_{t+1}^*$ denote the exact primal updates defined by
\cref{eq:npg_thetat,eq:npg_hatthetat+1} when all expectations are computed exactly.
Due to stochastic estimation and numerical approximation, the implemented updates produce $\theta_t$ and $\hat{\theta}_{t+1}$ such that, for all $t$,
\[
\E\left[\|\theta_t-\theta_t^*\|_1\right] \le \epsilon_{\mathrm{approx}},\,
\E\left[\|\hat{\theta}_{t+1}-\hat{\theta}_{t+1}^*\|_1\right] \le \epsilon_{\mathrm{approx}}.
\]
\end{assumption}
Such per-iteration errors are standard in the analysis of stochastic mirror descent and natural policy gradient methods.
In the tabular setting with exact expectations, this approximation error vanishes, i.e., $\epsilon_{\mathrm{approx}}=0$.
For parameterized policies, $\epsilon_{\mathrm{approx}}$ captures the combined effects of sampling noise and numerical approximation, and can be made arbitrarily small with sufficiently large batch sizes.

\begin{corollary}
    \label{cor:linear_npg_convergence}
    Under \cref{assump:bounded}, \cref{assump:reference_policy_full_support}, \cref{assump:feasibility_parameterized}, \cref{assump:log_bdd} and \cref{assump:approx}, under suitably chosen hyper-parameters $\eta_{\theta}$ and $\eta_{\lambda}$ (e.g., $\eta_{\theta}=\eta_{\lambda}=3\sqrt{|\mathcal{H}|}R_{\max}$), then the optimistic primal--dual iterates of \cref{eq:npg_thetat,eq:npg_lambdat,eq:npg_hatthetat+1,eq:npg_hatlambdat+1} satisfies
\[
\begin{aligned}
&\E_{x\sim\mathcal{D}}\left[\KL(\pi_{\theta^{\star}}(\cdot|x)\|\hat{\pi}_{t}(\cdot|x))\right]
+\|\boldsymbol{\lambda}^{\star}-\hat{\boldsymbol{\lambda}}_{t}\|_2^2\\
&\leq \rho^{t-1}\frac{\Phi_1}{\min\left(\eta_{\theta}+\beta,\frac{7}{4}\eta_{\lambda}-\frac{3}{4}\sqrt{|\mathcal{H}|} R_{\max}\right)}\\
&\quad +\frac{2(1-\rho^{t})}{1-\rho}\gap(\varepsilon_{\mathrm{approx}},p_{\min}),
\end{aligned}
\]
where $0<\rho<1$ is defined in \cref{eq:rho_valued} and $\Phi_1$ is defined as \cref{eq:Phi_1_valued}, and $\gap(\varepsilon_{\mathrm{approx}},p_{\min})
$ is defined in \cref{eq:gap_definition}.
\end{corollary}

The additional error term $\gap(\varepsilon_{\mathrm{approx}},p_{\min})$
characterizes the error gap induced by function approximation and finite-sample estimation in the policy update, and it determines the radius of a bounded neighborhood around the optimal saddle point $(\pi^\star,\lambda^\star)$.
The geometric contraction factor $\rho^{t-1}$ with $0<\rho<1$ ensures last-iterate convergence.
OPD in the parameter space preserves geometric last-iterate convergence, implying that function approximation does not destroy the stabilizing effect of optimism, but only introduces a controlled residual error.
As the approximation error vanishes, the neighborhood shrinks accordingly.
This result establishes OPD as a robust framework for constrained RLHF under practical policy parameterizations.

\section{Computational Experiments}
In this section, we empirically evaluate the effectiveness and robustness of the proposed OPD-based methods for aligning helpfulness and harmlessness.
Specifically, our experiments are designed to answer the following questions:
\begin{itemize}
\item How robust is the training process of the proposed OPD-based method compared to standard PD-based approaches?
\item Does improved training stability translate into superior performance at evaluation time?
\end{itemize}

\paragraph{Datasets and Reward Models}
We adopt the Alpaca-7b-reproduced model as the reference policy throughout our experiments.
For model-based alignment, we directly use the \texttt{beaver-7b-v1.0-reward} and
\texttt{beaver-7b-v1.0-cost} models released with Safe-RLHF~\cite{dai2024safe}
as the reward model for the target objective and the safety model for the constraint, respectively.
Note that the original Safe-RLHF formulation enforces the cost to be smaller than zero;
to match our constraint convention, we negate the cost model outputs.
We conduct our experiments on the PKU-SafeRLHF-30K preference dataset~\cite{dai2024safe},
which contains approximately 27K training prompts and 3K test prompts, each paired with a
preferred and a less-preferred response.
In addition to preference labels, the dataset provides safety annotations, where preferences
are determined jointly based on helpfulness and harmlessness.





\paragraph{OPD implementation}

On the primal side, we follow the standard PPO-style implementation used in practical RLHF systems. Specifically, the policy update is implemented via a clipped policy gradient objective, which can be viewed as a practical approximation of NPG under a trust-region constraint induced by the KL divergence to the reference policy. This design ensures stable policy updates while remaining compatible with large-scale language model fine-tuning.

On the dual side, the dual variable $\lambda$ is updated using gradient ascent in the logarithmic parameterization $\log \lambda$ to enforce non-negativity and improve numerical stability. 
We store the dual gradient from the previous iteration and construct an extrapolated gradient using an extragradient-style correction. The effective update direction is given by
\[
g_t^{\mathrm{OPD}} = 2 g_t - g_{t-1},
\]
where $g_t$ denotes the gradient of the dual objective at iteration $t$. This corrected gradient is applied directly to the log-dual variable $\log \lambda$, yielding an optimistic update that anticipates future primal responses.
Our OPD implementation explicitly introduces temporal coupling across iterations through gradient reuse. 

We set both the actor and critic learning rates to $5\times10^{-5}$, and the stepsize for the dual
variable $\lambda$ to $0.5$.
These relatively aggressive stepsizes intentionally place PD-based methods in an unstable regime, allowing us to test the robustness of the proposed OPD updates.
As shown in \cref{fig:opd_ppo_train}, the OPD method converges to policies that satisfy the safety
constraints while maintaining competitive rewards, whereas PD-based methods exhibit degraded
safety performance at convergence.


\begin{figure}[!t]
    \centering
    \includegraphics[width=\linewidth]{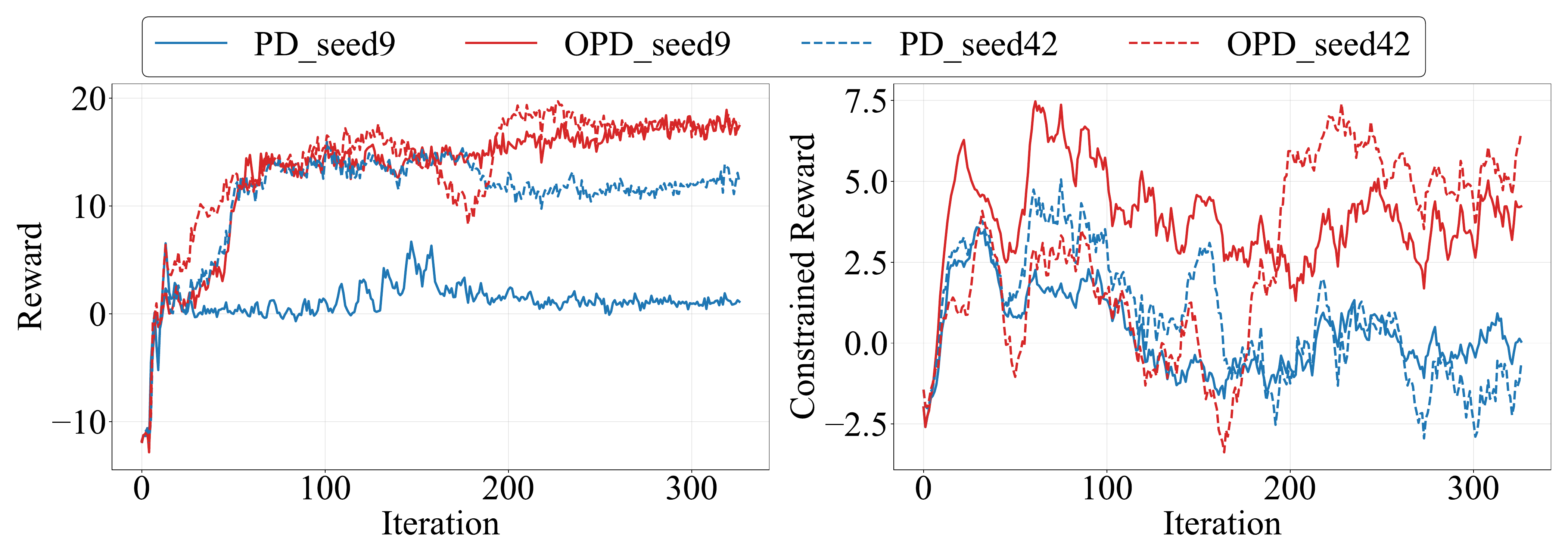}
    \caption{Comparison of PD and OPD on reward and constrained reward during the training phase.}
    \label{fig:opd_ppo_train}
\end{figure}

We conduct model-based evaluations for both helpfulness and safety, as shown in \cref{fig:evaluation}.
Specifically, the generated responses are evaluated by computing the corresponding average
helpfulness and safety scores using the proxy reward and safety models.
The evaluation results show that the OPD-based method achieves higher rewards and constraints than PD-based methods, indicating that improved training stability translates into superior evaluation performance.

\begin{figure}[!t]
    \centering
    \includegraphics[width=\linewidth]{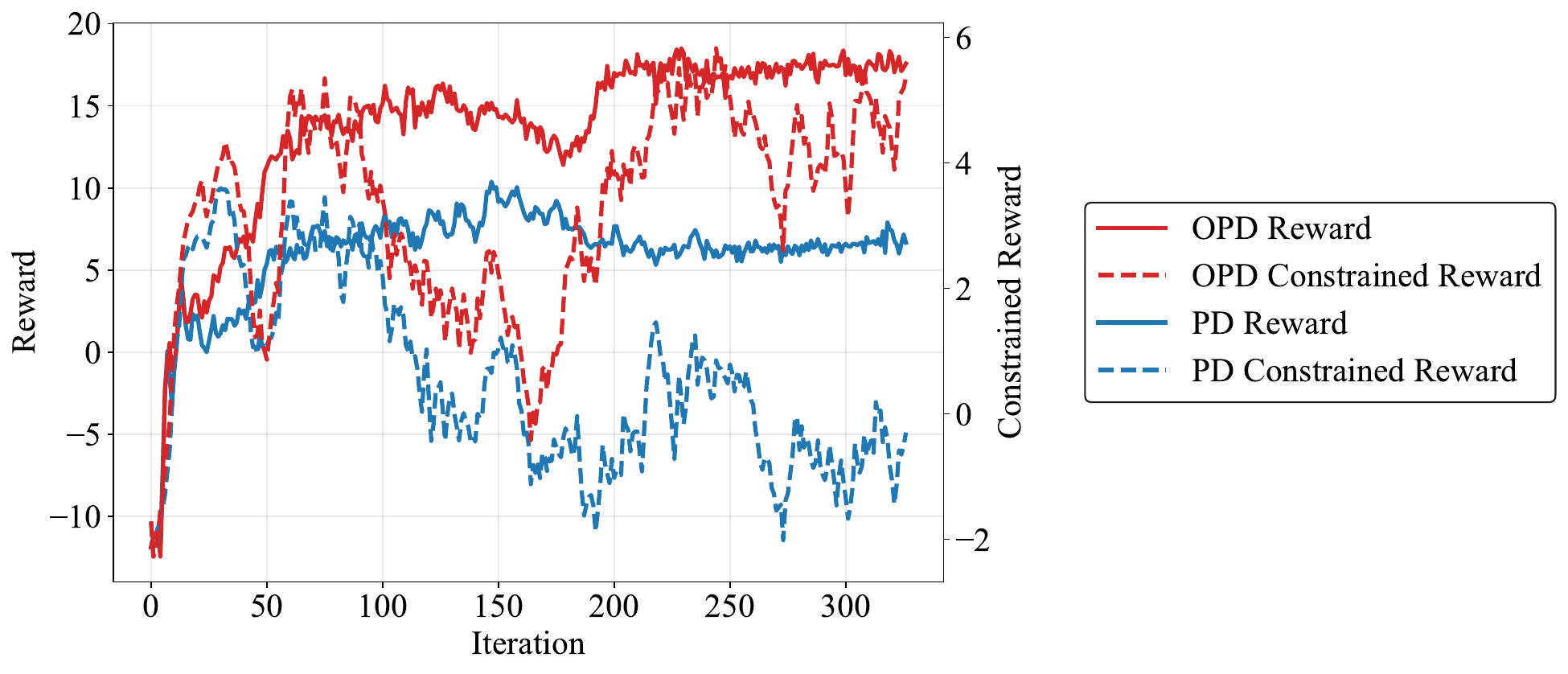}
    \caption{Inference comparison of PD and OPD on reward and cost.}
    \label{fig:evaluation}
\end{figure}

\section{Conclusion}
We develop a universal primal–dual framework that unifies a broad class of Lagrangian approaches to constrained RLHF. Building on this framework, we propose an OPD algorithm that introduces predictive updates for both the primal policy and the dual multipliers to stabilize saddle-point dynamics.
We establish last-iterate convergence in both the distributional policy space and the parameterized policy space. 
In the distributional setting, the iterates converge to the exact saddle point; in the parameterized setting, they converge to a neighborhood of the optimum.
We compare standard primal--dual training with our OPD variant. With more aggressive step sizes, OPD exhibits improved training stability relative to PD, and this stability translates into better performance in evaluations.

\newpage
\bibliography{MARL_Ref.bib}
\bibliographystyle{plainnat}

\newpage
\appendix
\onecolumn

\section{Related Works}
This section summarizes the related work in the LLMs safety alignment. 
\paragraph{Constrained Alignment for LLMs}
There is a growing body of work that formulates safe RLHF as a constrained optimization problem, where helpfulness is maximized subject to safety-related constraints~\cite{dai2024safe,huang2024one,zhang2025alignment}. 
A representative approach is modeling the safety violations via an expected cost constraint and solving the resulting constrained objective using iterative primal-dual updates~\cite{dai2024safe}. 
However, the primal-dual methods can be computationally expensive and may suffer from training instability and hyperparameter sensitivity~\cite{huang2024one}.
To mitigate these issues, \cite{huang2024one} leverages the closed-form structure of the optimal distribution induced by fixed dual variables, and optimizes a smooth dual objective to eliminate repeated primal-dual policy iterations. 
This dualization-based method leads to more stable training in practice.
In contrast, \cite{zhang2025alignment} studies constrained alignment in the parameterized LLM policy space and develops an iterative dual-based alignment method that alternates between maximizing the Lagrangian over the LLM policy parameters and performing dual descent updates. 
\cite{dhillon2024l3ms} proposes an interior point method and uses a relaxed log-barrier function to enforce constraints, thereby avoiding the oscillation between primal and dual variables.
Existing approaches stabilize training via simplifying the dual problem with the closed-form structure of the optimal policy distribution, solving near-optimal primal subproblems for each dual variable, or enforcing the constraints with the interior-point methods. 
However, it is still an open question of how to design iterative primal-dual updates with provable last-iterate guarantees for constrained LLM alignment.

\paragraph{RL-free Based Safety Alignment}
In parallel, a line of RL-free methods has been proposed for preference alignment, which bypasses explicit reward model learning and policy optimization via reinforcement learning, and instead directly optimize the policy using preference data~\cite{rafailov2023direct,azar2024general,ethayarajh2024kto,hong2024orpo,yang2024rewards}. 
Building upon these approaches, several recent works propose constrained preference alignment in an RL-free manner~\cite{liu2024enhancing,wachi2024stepwise,du2025primal,kim2025safedpo}.
Among them, some methods still adopt a primal-dual perspective and iteratively update both the policy and the dual variables, while using Direct Policy Optimization (DPO)-style objectives as the primal optimizer~\cite{liu2024enhancing,du2025primal}. 
To implicitly control the trade-off between reward and safety using only reward and cost preference datasets, existing approaches either reweight or reconstruct preference data according to the current dual variables~\cite{liu2024enhancing}, or perform separate preference optimization on reward and cost datasets under a Lagrangian formulation~\cite{du2025primal}. 
In contrast, \cite{wachi2024stepwise} avoids iterative dual updates and instead performs constrained alignment that evaluates multiple fixed dual values. 
\cite{kim2025safedpo} proposes a heuristic yet lightweight approach that enforces safety by directly reordering preference pairs: responses that violate safety guidelines are automatically relabeled as the worse one, enabling safety-aware alignment without explicit dual optimization.
While these RL-free approaches improve efficiency and empirical stability, they either rely on heuristic trade-off control or lack a principled analysis with respect to last-iterate convergence under safety constraints.

\paragraph{General Multi-objective Preference Optimization}
Beyond single-constraint formulations, several works study alignment from a multi-objective perspective. 
Some approaches aim to approximate Pareto-optimal policies by optimizing scalarization over multiple objectives with respect to a given preference vector~\cite{zhou2024beyond}. 
Other works vary the threshold of constraints to construct a Pareto front, where a primary objective is optimized subject to secondary objectives satisfying varying bounds, enabling flexible trade-offs between helpfulness and safety~\cite{agnihotri2025multi}. 
In addition, recent work explores context-dependent preference modeling, where alignment objectives dynamically vary with user intent or task context, as exemplified by reward-in-context approaches~\cite{yang2024rewards}. 
While these methods emphasize flexible and expressive preference modeling, they typically do not analyze the optimization dynamics of constrained saddle-point formulations, particularly under expectation-based safety constraints.





\section{Useful Lemmas}

\begin{lemma}[Hölder's inequality]
\label{lem:holder's inequality}
$\|fq\|_1\leq \|f\|_{p}\|g\|_q$ for $\frac{1}{p}+\frac{1}{q}=1$.
\end{lemma}

\begin{lemma}[Pinsker's inequality (discrete form)]
\label{lem:pinsker's ineq}
Let $p=[p_1,\cdots,p_d]^\top$ and $q=[q_1,\cdots,q_d]^\top$ be probability vectors on a finite set, and assume $\KL(p\|q)<\infty$. Then
\[\|p-q\|_1 \le \sqrt{2\KL(p\|q)}.\]
\end{lemma}

\begin{lemma}[Young's inequality]
\label{lem:young's ineq}
For $\delta>0$ and $\mathbf{u},\mathbf{v}\in\mathbb{R}^d$,
    \[
    |\left\langle \mathbf{u},\mathbf{v} \right\rangle|\leq \frac{\delta}{2}\|\mathbf{u}\|_2^2+\frac{1}{2\delta}\|\mathbf{v}\|_2^2.
    \]
\end{lemma}

\begin{lemma}
    \label{lem:||a+b+c||_2^2_lower_bd}
    For any $\delta,\theta>0$, we have
    \[
    \|\mathbf{a}+\mathbf{b}+\mathbf{c}\|^2\geq (1-\delta)\| \mathbf{a}\|^2 
+ (1-\frac{1}{\delta})(1-\theta)\|\mathbf{b}\|^2
+(1-\frac{1}{\delta})(1-\frac{1}{\theta})\|\mathbf{c}\|^2.
    \]
\end{lemma}
\begin{proof}
By Young's inequality for any $\delta>0$, $\ip{\mathbf{x}}{\mathbf{y}}\ge -\frac{\delta}{2}\|\mathbf{x}\|^2-\frac{1}{2\delta}\|\mathbf{y}\|^2$,
we have $\|\mathbf{x}+\mathbf{y}\|^2=\|\mathbf{x}\|^2+\|\mathbf{y}\|^2+2\ip{\mathbf{x}}{\mathbf{y}}\geq (1-\delta)\|\mathbf{x}\|^2 + (1-\frac{1}{\delta})\|\mathbf{y}\|^2$.
For any $\delta,\theta >0$
\[
\begin{aligned}
&\|\mathbf{a}+\mathbf{b}+\mathbf{c}\|^2
= \|\mathbf{a}\|^2+\|\mathbf{b}+\mathbf{c}\|^2+2\left\langle \mathbf{a},\mathbf{b}+\mathbf{c}\right\rangle\\
&\ge (1-\delta)\| \mathbf{a}\|^2 + (1-\frac{1}{\delta})\|\mathbf{b}+\mathbf{c}\|^2\\
&\ge (1-\delta)\| \mathbf{a}\|^2 + (1-\frac{1}{\delta})((1-\theta)\|\mathbf{b}\|^2 + (1-\frac{1}{\theta})\|\mathbf{c}\|^2)\\
&= (1-\delta)\| \mathbf{a}\|^2 
+ (1-\frac{1}{\delta})(1-\theta)\|\mathbf{b}\|^2
+(1-\frac{1}{\delta})(1-\frac{1}{\theta})\|\mathbf{c}\|^2.
\end{aligned}
\]
\end{proof}


\begin{lemma}
\label{lem:bdd_lambda}
    (\cite{ding2025convergence} Lemme 3(b))
    Let \cref{assump:feasibility} hold, then $\lambda^*>0$ and
    \[
    \begin{aligned}
        \left\Vert\lambda^*\right\Vert_1
        \leq 
        &\frac{1}{\xi}\mathbb{E}_{x\sim \mathcal{D}}\left[\left(
    \mathbb{E}_{y\sim\pi^*(\cdot|x)}\left[\sum_{j\in\mathcal{S}}w_jR_j(x,y)\right]-\beta \KL(\pi^*(\cdot|x)||\pi_{\tref}(\cdot|x))\right)\right.\\
    &\left.-\left(
    \mathbb{E}_{y\sim\Bar{\pi}(\cdot|x)}\left[\sum_{j\in\mathcal{S}}w_jR_j(x,y)\right]-\beta \KL(\pi^*(\cdot|x)||\pi_{\tref}(\cdot|x))\right)\right].
    \end{aligned}\]
\end{lemma}
\begin{remark}
Assume \cref{assump:bounded} and \cref{assump:reference_policy_full_support} hold. 
By \cref{assump:bounded}, we have $R_j(x,y)\leq R_{\max}$ for any $j\in\mathcal{S}$ and $(x,y)$ pairs. By \cref{assump:reference_policy_full_support}, there exists $p_{\min}>0$ such that $\pi_{\tref}(y|x)\geq p_{\min}$ for any $(x,y)$ pair.
By \cref{lem:bdd_lambda}, we have 
\[
\|\lambda^*\|_1\leq \frac{2}{\xi}\left(R_{\max}+\beta\log \frac{1}{p_{\min}}\right).
\]
Define $\Vert\boldsymbol{\lambda}\Vert_{1,\max}=\frac{2}{\xi}\left(R_{\max}+\beta\log \frac{1}{p_{\min}}\right)$.
If we set $\Lambda = \{\lambda|\|\lambda\|_1\leq \Vert\boldsymbol{\lambda}\Vert_{1,\max}\}$, then the optimality of $\lambda^*$ is not affected by projection.
\end{remark}

\begin{lemma}
\label{lemma:solve_RLHF_opt}
Given an optimization problem
\begin{equation}
\max_{\pi}  \mathbb{E}_{x \sim \mathcal{D}} [  \mathbb{E}_{y \sim \pi(\cdot|x)} [ S(x,y) ] 
- \beta  \mathrm{KL}\left(\pi(\cdot|x) \| \pi_{\mathrm{ref}}(\cdot|x)\right)  ]
\end{equation}
The solution can be written as
\[
\pi^\star(y|x) 
= \frac{1}{Z(x)}  \pi_{\mathrm{ref}}(y|x)  \exp\left( \tfrac{1}{\beta} S(x,y) \right),
\]
where $Z(x)$ is the normalization factor $Z(x) = \sum_{y} \pi_{\mathrm{ref}}(y|x)  \exp\left( \tfrac{1}{\beta} S(x,y) \right)$.
\end{lemma}
\begin{proof}
Since the objective decomposes over $x$, the maximizer can be found pointwise in $x$. 
For a fixed $x$, we can rewrite the optimization problem as
\begin{equation}
\label{eq:inner_x}
\max_{\pi(\cdot|x) \in \Delta}  \sum_{y} \pi(y|x)  S(x,y) - \beta \sum_{y} \pi(y|x)  \log \frac{\pi(y|x)}{\pi_{\mathrm{ref}}(y|x)},
\end{equation}
where $\Delta = \{\mathbf{x}|\|\mathbf{x}\|_1 = 1 \text{ and } \mathbf{x}>0\}$

Introduce a multiplier $\eta(x)$ for the normalization constraint. The Lagrangian for \cref{eq:inner_x} is
\begin{equation}
\mathcal{L}_x(\pi_x,\eta) = \sum_{y} \pi(y|x) S(x,y) - \beta \sum_{y} \pi(y|x) \log \frac{\pi(y|x)}{\pi_{\mathrm{ref}}(y|x)} 
+ \eta(x) ( \sum_{y} \pi(y|x) - 1 ).
\end{equation}
Taking the derivative w.r.t.\ $\pi(y|x)$ and setting to zero gives, for every $y$ in the support,
\begin{align}
0 = \frac{\partial \mathcal{L}_x}{\partial \pi(y|x)} 
&= S(x,y) - \beta ( \log \pi(y|x) - \log \pi_{\mathrm{ref}}(y|x) + 1 ) + \eta(x). \label{eq:stationarity}
\end{align}
Rearranging \cref{eq:stationarity} yields
\begin{align}
\log \pi(y|x) 
&= \log \pi_{\mathrm{ref}}(y|x) + \frac{1}{\beta} ( S(x,y) + \eta(x) - \beta ). \nonumber \\
\Rightarrow \quad \pi^\star(y|x) 
&= \pi_{\mathrm{ref}}(y|x)  \exp\left( \tfrac{1}{\beta} (S(x,y)+\eta(x)-\beta) \right), \label{eq:pi_star}
\end{align}
Taking the derivative w.r.t.\ $\eta(x)$ and setting to zero gives $\sum_y \pi^\star(y|x)=1$. Therefore, we have $\pi^\star(y|x) = \frac{1}{Z(x)}\pi_{\mathrm{ref}}(y|x)  \exp\left( \tfrac{1}{\beta} S(x,y) \right)$, where $Z(x)= \sum_{y} \pi_{\mathrm{ref}}(y|x)  \exp\left( \tfrac{1}{\beta} S(x,y) \right)$.
\end{proof}

\begin{lemma}[Three-point identity for Bregman divergences]
\label{lem:three-point-identity_2}
Let $f:\Omega\to\mathbb{R}$ be a function that is: a) strictly convex, b) continuously differentiable, c) defined on a closed convex set $\Omega$. Then the Bregman divergence is defined as $D_f(u,v)=f(u)-f(v)-\left\langle \nabla f(v),u-v\right\rangle, \forall x,y\in\Omega$. Then for all $x,y,z\in\Omega$
\begin{equation}
\label{eq:three-point-identity}
D_f(x,z)-D_f(x,y)-D_f(y,z)
=
\left\langle \nabla_y D_f(y,z), x-y \right\rangle.
\end{equation}
\end{lemma}

\begin{proof}
By the definition of the Bregman divergence, we have
\[
\begin{aligned}
D_f(x,z)&=f(x)-f(z)-\left\langle \nabla f(z),x-z\right\rangle,\\
D_f(x,y)&=h(x)-h(y)-\left\langle \nabla f(y),x-y\right\rangle,\\
D_f(y,z)&=f(y)-f(z)-\left\langle \nabla f(z),y-z\right\rangle. 
\end{aligned}
\]
Subtracting the latter two from the first gives
\begin{align*}
&D_f(x,z)-D_f(x,y)-D_f(y,z)\\
&=(f(x)-f(z)-\left\langle \nabla f(z),x-z\right\rangle)
-(f(x)-f(y)-\left\langle \nabla f(y),x-y\right\rangle)
-(f(y)-f(z)-\left\langle \nabla f(z),y-z\right\rangle)\\
&=-\left\langle \nabla f(z),x-z\right\rangle
+\left\langle \nabla f(y),x-y\right\rangle
+\left\langle \nabla f(z),y-z\right\rangle\\
&=\left\langle \nabla f(y),x-y\right\rangle-\left\langle \nabla f(z),x-y\right\rangle\\
&=\left\langle \nabla f(y)-\nabla f(z),x-y\right\rangle.
\end{align*}
Also we have 
\[
\begin{aligned}
\left\langle \nabla_y D_f(y,z), x-y \right\rangle
&=\left\langle \nabla_{y}\left(f(y)-f(z)-\left\langle \nabla f(z),y-z\right\rangle\right),x-y\right\rangle\\
&=\left\langle \nabla f(y)-\nabla f(z), x-y\right\rangle\\
\end{aligned}
\]
This concludes the proof.
\end{proof}

\begin{lemma}
\label{lem:three-point-gradient}
Let $h:\Omega\to\mathbb{R}$ be a function that is: a) strictly convex, b) continuously differentiable, c) defined on a closed convex set $\Omega$, and $D_h(u,v), \forall u,v$ be the Bregman divergence defined on $h$.
Let $f(\mathbf{x})=\left\langle g, \mathbf{x}\right\rangle
-\eta D_{h}(\mathbf{x},\mathbf{x}_{\old})$. Given step sizes $\eta>0$, consider the update
\begin{equation}
\mathbf{x}_{\new}
=
\arg\max_{\mathbf{x}\in\Omega} f(\mathbf{x}).
\end{equation}
Then for any $\mathbf{x}'\in \Omega$,
\begin{equation}
  \left\langle g, \mathbf{x}_{\new}-\mathbf{x}'\right\rangle
\geq\eta\left(-D_h(\mathbf{x}',\mathbf{x}_{\old})+D_h(\mathbf{x}',\mathbf{x}_{\new})+D_h(\mathbf{x}_{\new},\mathbf{x}_{\old})\right).
\end{equation}
\end{lemma}
\begin{proof}
Note that 
\[
\begin{aligned}
f(\mathbf{x})
=&\left\langle g, \mathbf{x}\right\rangle
-\eta D_{h}(\mathbf{x},\mathbf{x}_{\old})\\
=&\left\langle g, \mathbf{x}\right\rangle
-\eta (h(\mathbf{x})-h(\mathbf{x}_{\old})-\left\langle\nabla h(\mathbf{x}_{\old}), \mathbf{x}-\mathbf{x}_{\old}\right\rangle) \\
=& \left\langle g-\eta \nabla h(\mathbf{x}_{\old}), \mathbf{x}\right\rangle
-\eta h(\mathbf{x}) +\eta h(\mathbf{x}_{\old})+\eta\left\langle\nabla h(\mathbf{x}_{\old}),\mathbf{x}_{\old}\right\rangle.
\end{aligned}
\]
The first term is linear with $\mathbf{x}$, and the second term is strictly concave with $\mathbf{x}$ since $h(\cdot)$ is strictly convex, and the last two terms are constants with $\mathbf{x}$. Therefore, $f(\mathbf{x})$ is strictly concave with $\mathbf{x}$. As $f(\cdot)$ is differentiable on $\Omega$ and $\mathbf{x}_{\new}=\arg\max_{\mathbf{x}\in\Omega} f(\mathbf{x})$, we have 
\[
\left\langle \nabla f(\mathbf{x}_{\new}), \mathbf{x}_{\new}-\mathbf{x}'\right\rangle\geq 0.
\]
Substituting $f(\mathbf{x})=\left\langle g, \mathbf{x}\right\rangle
-\eta D_{h}(\mathbf{x},\mathbf{x}_{\old})$ into the above inequality, we have
\begin{equation}
\label{eq:kl_gradient_lemma_prf1}
    \left\langle g-\eta\nabla D_{h}(\mathbf{x}_{\new},\mathbf{x}_{\old}), \mathbf{x}_{\new}-\mathbf{x}'\right\rangle\geq 0.
\end{equation}
Let $x=\mathbf{x}'$, $y=\mathbf{x}_{\new}$, and $z=\mathbf{x}_{\old}$ and $D_f(\cdot,\cdot)=D_h(\cdot,\cdot)$ in \cref{lem:three-point-identity_2}, we have 
\[
\left\langle \nabla D_{h}(\mathbf{x}_{\new},\mathbf{x}_{\old}), \mathbf{x}'-\mathbf{x}_{\new}\right\rangle
=
D_{h}(\mathbf{x}',\mathbf{x}_{\old})-D_{h}(\mathbf{x}',\mathbf{x}_{\new})-D_{h}(\mathbf{x}_{\new},\mathbf{x}_{\old}).
\]
Substituting the above equation to the LHS of \cref{eq:kl_gradient_lemma_prf1}, we get
\[
\left\langle g, \mathbf{x}_{\new}-\mathbf{x}'\right\rangle+\eta\left(D_{h}(\mathbf{x}',\mathbf{x}_{\old})-D_{h}(\mathbf{x}',\mathbf{x}_{\new})-D_{h}(\mathbf{x}_{\new},\mathbf{x}_{\old})\right)\geq 0.
\]
This completes the proof.
\end{proof}

\begin{lemma}[Three-point inequality with KL regularization]
\label{lem:three-point-KL}
Let $\PiSet$ be the probability simplex, and $\pi_{\old},\pi_{\tref}\in\PiSet$.
Let $f(\pi)=\left\langle g, \pi\right\rangle
-\eta\KL(\pi\|\pi_{\old})
-\beta\KL(\pi\|\pi_{\tref})$. 
Define $\Pi_{\eff} = \{\pi\in\PiSet|\supp(\pi)\subseteq\supp(\pi_{\old})\cap\supp(\pi_{\tref})\}$. 
Given step sizes $\eta>0$ and $\beta> 0$, consider the update
\begin{equation}
\label{eq:update-21}
\pi_{\new}
=
\arg\max_{\pi\in\PiSet} f(\pi).
\end{equation}
Then for any $\pi'\in\Pi_{\eff}$,
\begin{equation}
\label{eq:target-ineq}
\begin{aligned}
  &\left\langle g, \pi_{\new}-\pi' \right\rangle
-\beta\left(
\KL(\pi_{\new}\|\pi_{\tref})
-
\KL(\pi'\|\pi_{\tref})
\right)\\
&\geq\eta\left(-\KL(\pi'\|\pi_{\old})+\KL(\pi'\|\pi_{\new})+\KL(\pi_{\new}\|\pi_{\old})\right)
+\beta\KL(\pi'\|\pi_{\new}).
\end{aligned}
\end{equation}
\end{lemma}

\begin{proof}
Since $\eta$ and $\beta$ are larger than zero, and $\KL(\pi\Vert \pi_{\old})$ and $\KL(\pi\Vert \pi_{\tref})$ are not defined when $\supp(\pi)\not\subseteq\supp(\pi_{\old})\cap\supp(\pi_{\tref})$, the domain of definition of $f(\pi)$ is $\Pi_{\eff}$.
Moreover, we can rearrange $f(\pi)$ as
\[
f(\pi)=\left\langle g+\eta\log\pi_{\old}+\beta\log\pi_{\tref}, \pi\right\rangle
-(\eta+\beta)h(\pi),
\]
where $h(\pi)$ is the negative entropy of $\pi$. Since the first term is linear with $\pi$, and $h(\pi)$ is strictly convex with $\pi$ and $\eta+\beta>0$, we have $f(\pi)$ is a strictly concave function with $\pi$.

Note that $f(\pi)$ is differentiable on $\Pi_{\eff}$. By the optimality of $\pi_{\new}$ in \cref{eq:update-21}, for every $\pi'\in \Pi_{\eff}$,
\begin{align*}
&\left\langle \nabla f(\pi_{\new}), \pi_{\new}-\pi'\right\rangle \geq 0.
\end{align*}
Calculating $\nabla f(\pi_{\new})$ and rearranging the above equation, we have
\begin{equation}
\label{eq:opt_criteria}
    \left\langle g-\eta\nabla_{\pi_{\new}}\KL(\pi_{\new}\|\pi_{\old})-\beta\nabla_{\pi_{\new}}\KL(\pi_{\new}\|\pi_{\tref}),\pi_{\new}-\pi'\right\rangle \geq 0.
\end{equation}

The Bregman divergence generated by $h(p)$ is $D_h(u,v)=\KL(u\Vert v)$. Applying \cref{lem:three-point-identity_2}, we have
\[
\begin{aligned}
\left\langle \nabla_{\pi_{\new}}\KL(\pi_{\new}\|\pi_{\old}),\pi'-\pi_{\new}\right\rangle
=&\KL(\pi'\|\pi_{\old})-\KL(\pi'\|\pi_{\new})-\KL(\pi_{\new}\|\pi_{\old}),\\
\left\langle \nabla_{\pi_{\new}}\KL(\pi_{\new}\|\pi_{\tref}),\pi'-\pi_{\new}\right\rangle
=&\KL(\pi'\|\pi_{\tref})-\KL(\pi'\|\pi_{\new})-\KL(\pi_{\new}\|\pi_{\tref}),
\end{aligned}
\]
where the first equation is letting $x=\pi'$, $y=\pi_{\new}$, $z=\pi_{\old}$, and $f(p)=h(p)$, and the second equation is letting $x=\pi'$, $y=\pi_{\new}$, $z=\pi_{\tref}$, and $f(p)=h(p)$. 
Substituting the above equations into the LHS of \cref{eq:opt_criteria}, we have
\begin{equation}
\label{eq:ineq-prethreepoint}
    \begin{aligned}
    &\left\langle g-\eta\nabla_{\pi_{\new}}\KL(\pi_{\new}\|\pi_{\old})-\beta\nabla_{\pi_{\new}}\KL(\pi_{\new}\|\pi_{\tref}),\pi_{\new}-\pi'\right\rangle\\
&\quad
=
\left\langle g, \pi_{\new}-\pi' \right\rangle
+\eta\left(\KL(\pi'\|\pi_{\old})-\KL(\pi'\|\pi_{\new})-\KL(\pi_{\new}\|\pi_{\old})\right)\\
&\qquad+\beta\left(\KL(\pi'\|\pi_{\tref})-\KL(\pi'\|\pi_{\new})-\KL(\pi_{\new}\|\pi_{\tref})\right)\\
&\quad\geq 0.
    \end{aligned}
\end{equation}
Rearranging this equation, we have
\[
\begin{aligned}
&\left\langle g, \pi_{\new}-\pi' \right\rangle
-\beta\left(\KL(\pi_{\new}\|\pi_{\tref})-\KL(\pi'\|\pi_{\tref})\right)\\
&\geq\eta\left(-\KL(\pi'\|\pi_{\old})+\KL(\pi'\|\pi_{\new})+\KL(\pi_{\new}\|\pi_{\old})\right)
+\beta\KL(\pi'\|\pi_{\new}). 
\end{aligned}
\]
This concludes the proof.
\end{proof}

\begin{lemma}
\label{lem:npg_bias_update}
Let $\PisubSet$ denote the parameterized policy set whose probability for each action is larger than $p_{\min}$ and $\Theta$ is a convex set. 
Assume $|g(x,y)|\leq g_{\max}$ for any $(x,y)$ pair.
Let $f(\pi)=\left\langle g, \pi\right\rangle
-\eta\KL(\pi\|\pi_{\theta_{\old}})
-\beta\KL(\pi\|\pi_{\tref})$. 
Given step sizes $\eta>0$ and $\beta> 0$, consider the update
\[
\pi_{\new}
=
\arg\max_{\pi\in\PisubSet} f(\pi).
\]
Let $\theta_{\new}=\theta_{\old}+\frac{1}{\eta+\beta}w$, where 
\[
\begin{aligned}
w=&\left(\E_{x\sim \mathcal{D},y\sim \pi_{\theta_{\old}}}\left[\nabla_{\theta_{\old}} \log\pi_{\theta}(y|x)\log\pi_{\theta_{\old}}(y|x)^T\right]\right)^{\dagger}\\
&\qquad\quad\E_{x\sim \mathcal{D},y\sim \pi_{\theta_{\old}}}\left[\left(g(x,y)-\beta\log\frac{\pi_{\theta_{\old}}(y|x)}{\pi_{\tref}(y|x)}\right)\nabla_{\theta_{\old}}\log\pi_{\theta_{\old}}(y|x)\right].
\end{aligned}
\]
Under \cref{assump:log_bdd} and \cref{assump:approx},
for any $\pi'\in\Pi_{\eff}$,
\[
\mathbb{E}_{x\sim \mathcal{D}}\left[\left\langle g-\eta\nabla_{\pi_{\theta_\new}}\KL(\pi_{\theta_\new}\|\pi_{\old})-\beta\nabla_{\pi_{\theta_\new}}\KL(\pi_{\theta_\new}\|\pi_{\tref}),\pi_{\theta_\new}-\pi'\right\rangle\right]
\geq -\gap(\varepsilon_{\mathrm{approx}},p_{\min}).
\]
where $\gap(\varepsilon_{\mathrm{approx}},p_{\min})$ is defined in \cref{eq:gap_definition}.
\end{lemma}

\begin{proof}
We first rewrite $\left\langle g-\eta\nabla_{\pi_{\theta_\new}}\KL(\pi_{\theta_\new}\|\pi_{\old})-\beta\nabla_{\pi_{\theta_\new}}\KL(\pi_{\theta_\new}\|\pi_{\tref}),\pi_{\theta_\new}-\pi'\right\rangle$ as
\[
\begin{aligned}
&\left\langle g-\eta\nabla_{\pi_{\theta_\new}}\KL(\pi_{\theta_\new}\|\pi_{\old})-\beta\nabla_{\pi_{\theta_\new}}\KL(\pi_{\theta_\new}\|\pi_{\tref}),\pi_{\theta_\new}-\pi'\right\rangle\\ 
&=\left\langle g-\eta\nabla_{\pi_{\new}}\KL(\pi_{\new}\|\pi_{\old})-\beta\nabla_{\pi_{\new}}\KL(\pi_{\new}\|\pi_{\tref}),\pi_{\theta_\new}-\pi'\right\rangle\\
&\quad+\left\langle \eta\left(\nabla_{\pi_{\new}}\KL(\pi_{\new}\|\pi_{\old})-\nabla_{\pi_{\theta_\new}}\KL(\pi_{\theta_\new}\|\pi_{\old})\right),\pi_{\theta_\new}-\pi'\right\rangle\\
&\quad+\left\langle \beta\left(\nabla_{\pi_{\new}}\KL(\pi_{\new}\|\pi_{\tref})-\nabla_{\pi_{\theta_\new}}\KL(\pi_{\theta_\new}\|\pi_{\tref})\right),\pi_{\theta_\new}-\pi'\right\rangle.
\end{aligned}
\]
Since $\nabla_{\pi(y|x)}\KL(\pi(\cdot|x)||\pi'(\cdot|x))=\log\frac{\pi(y|x)}{\pi'(y|x)}+1$, the LHS of the above equation can be rewritten as
\[
\begin{aligned}
\text{LHS} 
=&
\left\langle g-\eta\nabla_{\pi_{\new}}\KL(\pi_{\new}\|\pi_{\old})-\beta\nabla_{\pi_{\new}}\KL(\pi_{\new}\|\pi_{\tref}),\pi_{\theta_\new}-\pi'\right\rangle
+\left\langle (\eta+\beta)\log\frac{\pi_{\new}}{\pi_{\theta_{\new}}},\pi_{\theta_{\new}}-\pi' \right\rangle\\
=&\left\langle g-\eta\nabla_{\pi_{\new}}\KL(\pi_{\new}\|\pi_{\old})-\beta\nabla_{\pi_{\new}}\KL(\pi_{\new}\|\pi_{\tref}),\pi_{\theta_\new}-\pi_{\new} \right\rangle\\
&+\left\langle g-\eta\nabla_{\pi_{\new}}\KL(\pi_{\new}\|\pi_{\old})-\beta\nabla_{\pi_{\new}}\KL(\pi_{\new}\|\pi_{\tref}),\pi_{\new}-\pi'\right\rangle\\
&+\left\langle (\eta+\beta)\log\frac{\pi_{\new}}{\pi_{\theta_{\new}}},\pi_{\theta_{\new}}-\pi' \right\rangle.
\end{aligned}
\]
By $\pi_{\new}=\arg\max_{\pi\in\PisubSet} f(\pi)$ and $\Pi_{\Theta}$ is the convex set, for any $\pi'\in\PisubSet$ we have
\[
  \left\langle g-\eta\nabla_{\pi_{\new}}\KL(\pi_{\new}\|\pi_{\old})-\beta\nabla_{\pi_{\new}}\KL(\pi_{\new}\|\pi_{\tref}),\pi_{\new}-\pi'\right\rangle \geq 0.
\]
Substituting the optimality of $\pi_{\new}$ into $\left\langle g-\eta\nabla_{\pi_{\theta_\new}}\KL(\pi_{\theta_\new}\|\pi_{\old})-\beta\nabla_{\pi_{\theta_\new}}\KL(\pi_{\theta_\new}\|\pi_{\tref}),\pi_{\theta_\new}-\pi'\right\rangle$ and taking the expectation over $x\sim \mathcal{D}$, we have
\begin{equation}
\label{eq:npg_lemma1_all}
\begin{aligned}
&\mathbb{E}_{x\sim \mathcal{D}}\left[\left\langle g-\eta\nabla_{\pi_{\theta_\new}}\KL(\pi_{\theta_\new}\|\pi_{\old})-\beta\nabla_{\pi_{\theta_\new}}\KL(\pi_{\theta_\new}\|\pi_{\tref}),\pi_{\theta_\new}-\pi'\right\rangle\right]\\
&\geq \mathbb{E}_{x\sim \mathcal{D}}\left[\left\langle g-\eta\nabla_{\pi_{\new}}\KL(\pi_{\new}\|\pi_{\old})-\beta\nabla_{\pi_{\new}}\KL(\pi_{\new}\|\pi_{\tref}),\pi_{\theta_\new}-\pi_{\new} \right\rangle
+(\eta+\beta)\left\langle \log\frac{\pi_{\new}}{\pi_{\theta_{\new}}},\pi_{\theta_{\new}}-\pi' \right\rangle\right].
\end{aligned}
\end{equation}

Since $\pi_{\tref},\pi_{\new}, \pi_{\old}\in \Pi_{\Theta}$, we have
\begin{equation}
    \label{eq:npg_lemma1_1}
\begin{aligned}
&\mathbb{E}_{x\sim \mathcal{D}}\left[\left|\left\langle g-\eta\nabla_{\pi_{\new}}\KL(\pi_{\new}\|\pi_{\old})-\beta\nabla_{\pi_{\new}}\KL(\pi_{\new}\|\pi_{\tref}),\pi_{\theta_\new}-\pi_{\new} \right\rangle\right|\right]\\
&=\mathbb{E}_{x\sim \mathcal{D}}\left[\left|\left\langle g-\eta\log\frac{\pi_{\new}}{\pi_{\old}}-\beta\log\frac{\pi_{\new}}{\pi_{\tref}},\pi_{\theta_\new}-\pi_{\new} \right\rangle\right|\right]\\
&\leq \mathbb{E}_{x\sim \mathcal{D}}\left[\left\Vert g-\eta\log\frac{\pi_{\new}}{\pi_{\old}}-\beta\log\frac{\pi_{\new}}{\pi_{\tref}}\right\Vert_{\infty}
\left\Vert \pi_{\theta_\new}-\pi_{\new}\right\Vert_1\right]\\
&\leq 
\mathbb{E}_{x\sim \mathcal{D}}\left[\left(g_{\max}+(\eta+\beta)\log\frac{1}{p_{\min}}\right)\left\Vert\pi_{\theta_\new}-\pi_{\new}\right\Vert_1\right]\\
\end{aligned}
\end{equation}

Substituting \cref{eq:npg_lemma1_1} into \cref{eq:npg_lemma1_all} gives
\begin{equation}
    \label{eq:npg_lemma1_all_2}
\begin{aligned}
&\mathbb{E}_{x\sim \mathcal{D}}\left[\left|\left\langle g-\eta\nabla_{\pi_{\theta_\new}}\KL(\pi_{\theta_\new}\|\pi_{\old})-\beta\nabla_{\pi_{\theta_\new}}\KL(\pi_{\theta_\new}\|\pi_{\tref}),\pi_{\theta_\new}-\pi'\right\rangle\right|\right]\\
&\leq \mathbb{E}_{x\sim \mathcal{D}}\left[\left(g_{\max}+(\eta+\beta)\log\frac{1}{p_{\min}}\right)\left\|\pi_{\theta_{\new}}-\pi_{\new}\right\|_1
+(\eta+\beta)\left|\left\langle \log\frac{\pi_{\new}}{\pi_{\theta_{\new}}},\pi_{\theta_{\new}}-\pi' \right\rangle\right|\right]\\
&\leq \mathbb{E}_{x\sim \mathcal{D}}\left[\left(g_{\max}+(\eta+\beta)\log\frac{1}{p_{\min}}\right)\sqrt{2\KL\left(\pi_{\theta_{\new}}\|\pi_{\new}\right)}
+(\eta+\beta)\left(1+\frac{1}{p_{\min}}\right)\KL\left(\pi_{\theta_{\new}}\|\pi_{\new}\right)\right]\\
&\quad\text{(by $\frac{\pi'(y|x)}{\pi_{\theta_{\new}}(y|x)}\leq\frac{1}{p_{\min}}$)}\\
&\leq \left(g_{\max}+(\eta+\beta)\log\frac{1}{p_{\min}}\right)\sqrt{2\mathbb{E}_{x\sim \mathcal{D}}\left[\KL\left(\pi_{\theta_{\new}}\|\pi_{\new}\right)\right]}
+(\eta+\beta)\left(1+\frac{1}{p_{\min}}\right)\mathbb{E}_{x\sim \mathcal{D}}\left[\KL\left(\pi_{\theta_{\new}}\|\pi_{\new}\right)\right].
\end{aligned}
\end{equation}

As $\pi_{\new}(y|x)$ corresponds to parameter $\theta^*$, under \cref{assump:log_bdd} and \cref{assump:approx}, we have
\[
\begin{aligned}
\E_{x\sim\mathcal{D}}\left[\KL\left(\pi_{\theta_{\new}}\|\pi_{\new}\right)\right]
\leq&\E_{x\sim\mathcal{D}}\left[\left\|\frac{\pi_{\theta_{\new}}}{\pi_{\theta_{\old}}}\right\|_{\infty}\E_{y\sim \pi_{\old}}\left[\log\frac{\pi_{\theta_{\new}}}{\pi_{\new}}\right]\right]\\
\leq&\E_{x\sim\mathcal{D}}\left[\frac{1}{p_{\min}}\left\|\log\left(\frac{\pi_{\theta_\new}(y|x)}{\pi_{\new}(y|x)}\right)\right\|_1\right]\\
\leq& \E_{x\sim\mathcal{D}}\left[\frac{C}{p_{\min}}\|\theta_\new-\theta^*\|_1\right]\\
\leq& \frac{C\epsilon_{\text{approx}}}{p_{\min}}
\end{aligned}
\]
We have
\[
\mathbb{E}_{x\sim \mathcal{D}}\left[\left\langle g-\eta\nabla_{\pi_{\theta_\new}}\KL(\pi_{\theta_\new}\|\pi_{\old})-\beta\nabla_{\pi_{\theta_\new}}\KL(\pi_{\theta_\new}\|\pi_{\tref}),\pi_{\theta_\new}-\pi'\right\rangle\right]
\geq -\gap(\varepsilon_{\mathrm{approx}},p_{\min}),
\]
where
\begin{equation}
\label{eq:gap_definition}
    \begin{aligned}
&\gap(\varepsilon_{\mathrm{approx}},p_{\min})
&=\left(g_{\max}+(\eta+\beta)\log\frac{1}{p_{\min}}\right)\sqrt{\frac{2C\epsilon_{\text{approx}}}{p_{\min}}}
+(\eta+\beta)\left(1+\frac{1}{p_{\min}}\right)\frac{C\epsilon_{\text{approx}}}{p_{\min}}. 
\end{aligned}
\end{equation}

\end{proof}

\begin{corollary}
\label{cor:npg_three_point_identity}
Let $\PisubSet$ denote the parameterized policy set whose probability for each action is larger than $p_{\min}$ and $\Theta$ is a convex set. 
Assume $|g(x,y)|\leq g_{\max}$ for any $(x,y)$ pair.
Let $f(\pi)=\left\langle g, \pi\right\rangle
-\eta\KL(\pi\|\pi_{\theta_{\old}})
-\beta\KL(\pi\|\pi_{\tref})$. 
Given step sizes $\eta>0$ and $\beta> 0$, consider the update
\[
\pi_{\new}
=
\arg\max_{\pi\in\PisubSet} f(\pi).
\]
Let $\theta_{\new}=\theta_{\old}+\frac{1}{\eta+\beta}w$, where 
\[
\begin{aligned}
w=&\left(\E_{x\sim \mathcal{D},y\sim \pi_{\theta_{\old}}}\left[\nabla_{\theta_{\old}} \log\pi_{\theta}(y|x)\log\pi_{\theta_{\old}}(y|x)^T\right]\right)^{\dagger}\\
&\qquad\quad\E_{x\sim \mathcal{D},y\sim \pi_{\theta_{\old}}}\left[\left(g(x,y)-\beta\log\frac{\pi_{\theta_{\old}}(y|x)}{\pi_{\tref}(y|x)}\right)\nabla_{\theta_{\old}}\log\pi_{\theta_{\old}}(y|x)\right].
\end{aligned}
\]
Under \cref{assump:log_bdd} and \cref{assump:approx},
for any $\pi'\in\Pi_{\eff}$, we have
\[
    \begin{aligned}
       & 
\mathbb{E}_{x\sim\mathcal{D}}\left[\left\langle g, \pi_{\theta_{\new}}-\pi' \right\rangle
-\beta\left(
\KL(\pi_{\theta_{\new}}\|\pi_{\tref})
-
\KL(\pi'\|\pi_{\tref})
\right)\right]\\
&\geq\mathbb{E}_{x\sim\mathcal{D}}\left[\eta\left(-\KL(\pi'\|\pi_{\old})+\KL(\pi'\|\pi_{\theta_{\new}})+\KL(\pi_{\theta_{\new}}\|\pi_{\old})\right)
+\beta\KL(\pi'\|\pi_{\theta_{\new}})\right]-\gap(\varepsilon_{\mathrm{approx}},p_{\min}),
    \end{aligned}
    \]
where $\gap(\varepsilon_{\mathrm{approx}},p_{\min})$ is defined in \cref{eq:gap_definition}.
\end{corollary}

\begin{proof}
By \cref{lem:npg_bias_update}, we have
\[
\mathbb{E}_{x\sim \mathcal{D}}\left[\left\langle g-\eta\nabla_{\pi_{\theta_\new}}\KL(\pi_{\theta_\new}\|\pi_{\old})-\beta\nabla_{\pi_{\theta_\new}}\KL(\pi_{\theta_\new}\|\pi_{\tref}),\pi_{\theta_\new}-\pi'\right\rangle\right]
\geq -\gap(\varepsilon_{\mathrm{approx}},p_{\min})
\]
Applying \cref{lem:three-point-identity_2} and let $f$ be the negative entropy function, we have
\[
\begin{aligned}
\left\langle \nabla_{\pi_{\theta_\new}}\KL(\pi_{\theta_\new}\|\pi_{\old}),\pi'-\pi_{\theta_\new}\right\rangle
=&\KL(\pi'\|\pi_{\old})-\KL(\pi'\|\pi_{\theta_\new})-\KL(\pi_{\theta_\new}\|\pi_{\old}),\\
\left\langle \nabla_{\pi_{\theta_\new}}\KL(\pi_{\theta_\new}\|\pi_{\tref}),\pi'-\pi_{\theta_\new}\right\rangle
=&\KL(\pi'\|\pi_{\tref})-\KL(\pi'\|\pi_{\theta_\new})-\KL(\pi_{\theta_\new}\|\pi_{\tref}),
\end{aligned}
\]
Substituting the above inequalities into $\mathbb{E}_{x\sim \mathcal{D}}\left[\left\langle g-\eta\nabla_{\pi_{\theta_\new}}\KL(\pi_{\theta_\new}\|\pi_{\old})-\beta\nabla_{\pi_{\theta_\new}}\KL(\pi_{\theta_\new}\|\pi_{\tref}),\pi_{\theta_\new}-\pi'\right\rangle\right]$ concludes the proof.
\end{proof}

\section{Proof of \cref{thm:reg_opt_LIC}}
\label{sec:thm1_proof}

As defined in \cref{eq:Sxy}, the aggregated reward function is the combined the weighted reward objectives and dual-variable weighted constrained reward objectives, shown as
\[
S_{\lambda}(x,y)
= \sum_{k \in \mathcal{S}} w_k R_k(x,y)
   + \sum_{j \in \mathcal{H}} \lambda_j R_j(x,y),
\]
Define the corresponding value function as
\begin{align}
\label{eq:V}
V^{\pi}_{S_{\lambda}}(x)
&:= \mathbb{E}_{y\sim\pi(\cdot\mid x)}\big[S_{\lambda}(x,y)\big].
\end{align}
The Lagrangian associated with the constrained MO-RLHF problem can then be written as
\[
\begin{aligned}
\mathcal{L}(\pi, \lambda)
&= J(\pi;\mathbf{w})
  + \sum_{j \in \mathcal{H}} \lambda_j
    \mathbb{E}_{x\sim\mathcal{D},y\sim\pi(\cdot\mid x)}
    \big[R_j(x,y)\big] \\
&= \mathbb{E}_{x\sim \mathcal{D}}
   \Big[
     V^{\pi}_{S_{\lambda}}(x)
     - \beta\mathrm{KL}\big(\pi(\cdot\mid x)\|\pi_{\tref}(\cdot\mid x)\big)
   \Big].
\end{aligned}
\]

The Lagrangian problem is 
\[
\min_{\lambda \ge 0} \max_{\pi\in \PiSet} \mathcal{L}(\pi,\lambda).
\]
Note that $\Pi$ is a finite policy set, hence the primal maximization attains an optimum.
Moreover, \cref{lem:bdd_lambda}, there exists optimal dual variable $\lambda^\star$ and $\lambda_{\max}>0$ such that $0\leq \lambda^{\star}\leq \lambda_{\max}$
Under \cref{assump:feasibility}, the strong duality holds and optimal saddle-point $(\pi^\star,\lambda^\star)$ exists.

Since $\pi^{\star}=\argmax_{\pi} L(\pi,\boldsymbol{\lambda}^{\star})$, we have $L(\pi^{\star},\boldsymbol{\lambda}^{\star})\geq L(\pi,\boldsymbol{\lambda}^{\star})$ for any $\pi\in \PiSet$. Similarly, since $\boldsymbol{\lambda}^{\star}=\argmin_{\lambda} L(\pi^{\star},\lambda)$, we have $L(\pi^{\star},\lambda)\geq L(\pi^{\star},\boldsymbol{\lambda}^{\star})$ for any $\lambda\geq 0$. Combining these two inequalities together, for any $\pi\in\PiSet$ and $\lambda\geq 0$, we have
\begin{equation}
\label{eq:L(pi*)-L(lambda*)>0}
L(\pi^{\star},\lambda) - L(\pi,\boldsymbol{\lambda}^{\star}) 
= \underbrace{L(\pi^{\star},\lambda) - L(\pi^{\star},\boldsymbol{\lambda}^{\star})}_{\geq 0} + \underbrace{L(\pi^{\star},\boldsymbol{\lambda}^{\star}) - L(\pi,\boldsymbol{\lambda}^{\star})}_{\geq 0}
\geq 0    
\end{equation}
Let $\pi = \pi_t$ and $\lambda =\boldsymbol{\lambda}_t$ and substituting the definition of $L(\pi,\lambda)$ into the LHS of the above inequality, we have
\begin{equation}
  \label{eq:AB-split}
\begin{aligned}
L(\pi^{\star},\boldsymbol{\lambda}_t) - L(\pi_t,\boldsymbol{\lambda}^{\star})
=&\E_{x\sim \mathcal{D}}\left[V^{\pi^{\star}}_{S_{\boldsymbol{\lambda}_t}}(x)
- \beta \KL(\pi^{\star}(\cdot|x) \| \pi_{\tref}(\cdot|x))\right]
-\E_{x\sim \mathcal{D}}\left[V^{\pi}_{S_{\boldsymbol{\lambda}^{\star}}}(x)
- \beta \KL(\pi(\cdot|x) \| \pi_{\tref}(\cdot|x))\right]\\
=&\underbrace{\E_{x\sim \mathcal{D}}\left[V^{\pi^{\star}}_{S_{\boldsymbol{\lambda}_t}}(x)
- \beta \KL(\pi^{\star}(\cdot|x) \| \pi_{\tref}(\cdot|x))\right]
-\E_{x\sim \mathcal{D}}\left[V^{\pi_t}_{S_{\boldsymbol{\lambda}_t}}(x)
- \beta \KL(\pi_t(\cdot|x) \| \pi_{\tref}(\cdot|x))\right]}_{\text{A}}\\
&+
\underbrace{\E_{x\sim \mathcal{D}}\left[V^{\pi_t}_{S_{\boldsymbol{\lambda}_t}}(x)
- \beta \KL(\pi_t(\cdot|x) \| \pi_{\tref}(\cdot|x))\right]
-\E_{x\sim \mathcal{D}}\left[V^{\pi_t}_{S_{\boldsymbol{\lambda}^{\star}}}(x)
- \beta \KL(\pi_t(\cdot|x) \| \pi_{\tref}(\cdot|x))\right])}_{\text{B}}
\end{aligned}  
\end{equation}

\subsection{Upper bound of term $\mathrm{A}$}
\label{sec:A_ub}
We can rewrite term $\mathrm{A}$ as:
\begin{equation}
\label{eq:A-basic}
\begin{aligned}
\mathrm{A}
=& \E_{x\sim \mathcal{D}}\left[V^{\pi^{\star}}_{S_{\boldsymbol{\lambda}_t}}(x)
- \beta \KL(\pi^{\star}(\cdot|x) \| \pi_{\tref}(\cdot|x))\right]
-\E_{x\sim \mathcal{D}}\left[V^{\pi_t}_{S_{\boldsymbol{\lambda}_t}}(x)
- \beta \KL(\pi_t(\cdot|x) \| \pi_{\tref}(\cdot|x))\right]\\
=& \E_{x\sim \mathcal{D}}\left[\left(V^{\pi^{\star}}_{S_{\boldsymbol{\lambda}_t}}(x)-V^{\pi_t}_{S_{\boldsymbol{\lambda}_t}}(x)\right)
- \beta \left(\KL(\pi^{\star}(\cdot|x) \| \pi_{\tref}(\cdot|x))
-\KL(\pi_t(\cdot|x) \| \pi_{\tref}(\cdot|x))\right)\right]\\
\overset{(a)}{=}& \E_{x\sim \mathcal{D}}\left[\langle \pi^{\star}(\cdot|x)-\pi_t(\cdot|x), S_{\boldsymbol{\lambda}_t}(x,\cdot)\rangle
-\beta\left(\KL(\pi^{\star}(\cdot|x))\Vert\pi_{\tref}(\cdot|x)))-\KL(\pi_t(\cdot|x))\Vert\pi_{\tref}(\cdot|x)))\right)\right]\\
\overset{(b)}{=}&\E_{x\sim \mathcal{D}}\left[\langle \pi^{\star}(\cdot|x)-\hat{\pi}_{t+1}(\cdot|x), S_{\boldsymbol{\lambda}_t}(x,\cdot)\rangle -\beta\left(\KL(\pi^{\star}(\cdot|x)\Vert\pi_{\tref}(\cdot|x))-\KL(\hat{\pi}_{t+1}(\cdot|x)\Vert\pi_{\tref}(\cdot|x))\right)\right.\\
&\qquad\quad+ \langle \hat{\pi}_{t+1}(\cdot|x)-\pi_t(\cdot|x), S_{\boldsymbol{\lambda}_{t-1}}(x,\cdot)\rangle -\beta\left(\KL(\hat{\pi}_{t+1}(\cdot|x)\Vert\pi_{\tref}(\cdot|x))-\KL(\pi_t(\cdot|x)\Vert\pi_{\tref}(\cdot|x))\right)\\
&\qquad\quad\left.+ \langle \hat{\pi}_{t+1}(\cdot|x)-\pi_t(\cdot|x), S_{\boldsymbol{\lambda}_t}(x,\cdot)-S_{\boldsymbol{\lambda}_{t-1}}(x,\cdot)\rangle\right]\\
\end{aligned}
\end{equation}
where $(a)$ is because the action space is discrete and $V^{\pi}_{S_{\boldsymbol{\lambda}_t}}(x)=\sum_{y}\pi(y|x){S_{\boldsymbol{\lambda}_t}}(x,y)=\langle \pi(\cdot|x), S_{\boldsymbol{\lambda}_t}(x,\cdot)\rangle$ for any $\pi\in\PiSet$, and $(b)$ is because adding and subtracting the same term keeps the equality.
Recall the $\hat{\pi}_{t+1}$ update shown in \cref{eq:hatpi_t} of the optimistic policy gradient primal-dual method,
\[
\hat{\pi}_{t+1}
= \arg\max_{\pi}
\mathbb{E}_{x\sim \mathcal{D}}\left[\mathbb{E}_{y\sim\pi(\cdot|x)}\left[S_{\boldsymbol{\lambda}_t}(x,y)\right]
-\beta\mathrm{KL}\left(\pi(\cdot|x)\Vert\pi_{\mathrm{ref}}(\cdot|x)\right)
-\eta_{\theta}\mathrm{KL}\left(\pi(\cdot|x)\Vert\hat{\pi}_t(\cdot|x)\right)\right], 
\]
Since the optimality is independent for any $x$, we can write the
for fixed $x$ as:
\[
\begin{aligned}
\hat{\pi}_{t+1}(\cdot|x)
=& \argmax_{\pi}
\mathbb{E}_{y\sim\pi(\cdot|x)}\left[S_{\boldsymbol{\lambda}_t}(x,y)\right]
-\beta\mathrm{KL}\left(\pi(\cdot|x)\Vert\pi_{\mathrm{ref}}(\cdot|x)\right)
-\eta_{\theta}\mathrm{KL}\left(\pi(\cdot|x)\Vert\hat{\pi}_t(\cdot|x)\right)\\
=& \argmax_{\pi}
\langle \pi(\cdot|x), S_{\boldsymbol{\lambda}_t}(x,\cdot)\rangle
-\beta\mathrm{KL}\left(\pi(\cdot|x)\Vert\pi_{\mathrm{ref}}(\cdot|x)\right)
-\eta_{\theta}\mathrm{KL}\left(\pi(\cdot|x)\Vert\hat{\pi}_t(\cdot|x)\right)
\end{aligned}
\]

Recall that we start with $\pi_0$ which has the same support as $\pi_{\tref}$, which has the full support by \cref{assump:reference_policy_full_support}. 
We also set $\hat{\pi}_0=\pi_{\tref}$
Given the optimistic policy gradient shown in \cref{eq:pi_t,eq:hatpi_t}, $\pi_t(y|x)\propto\hat{\pi}_t(y|x)^{\frac{\eta}{\eta+\beta}}{\pi}_{\tref}(y|x)^{\frac{\beta}{\eta+\beta}}\exp(S_{\lambda_{t-1}}(x,y))$ for any $t$, $\pi_{t}$ and $\hat{\pi}_t\propto\hat{\pi}_t(y|x)^{\frac{\eta}{\eta+\beta}}{\pi}_{\tref}(y|x)^{\frac{\beta}{\eta+\beta}}\exp(S_{\lambda_{t}}(x,y))$ for any $t$. By iteration, we have $\hat{\pi}_{t}$ and $\pi_t$ have the same support as $\pi_{\tref}$ for any $t$.  
Therefore, $\supp(\pi_{t})\cap\supp(\pi_{\tref})=\supp(\hat{\pi}_{t})\cap\supp(\pi_{\tref})=\supp(\pi_{\tref})$.
Since $\pi_{\tref}$ spans the action space by \cref{assump:reference_policy_full_support}, $\pi^{\star}$ is covered by $\pi_{\tref}$.
Therefore $\supp\left(\pi^{\star}\right)\subseteq\supp\left(\pi_t\right)\cap\supp\left(\pi_{\tref}\right)$.
Using \cref{lem:three-point-KL} and letting $\eta=\eta_{\theta}$, $g=S_{\boldsymbol{\lambda}_t}(x,\cdot)$, $\pi_{\old} = \hat{\pi}_t(\cdot|x)$, $\pi_{\new}=\hat{\pi}_{t+1}(\cdot|x)$, and $\pi'=\pi^{\star}(\cdot|x)$, we have
\[
\begin{aligned}
&\langle S_{\boldsymbol{\lambda}_t}(x,\cdot), \hat{\pi}_{t+1}(\cdot|x)-\pi^{\star}(\cdot|x) \rangle
-\beta\left(\KL(\hat{\pi}_{t+1}(\cdot|x)\|\pi_{\tref}(\cdot|x))-\KL(\pi^{\star}(\cdot|x)\|\pi_{\tref}(\cdot|x))\right)\\
&\geq\eta_{\theta}\left(-\KL(\pi^{\star}(\cdot|x)\|\hat{\pi}_t(\cdot|x))+\KL(\pi^{\star}(\cdot|x)\|\hat{\pi}_{t+1}(\cdot|x))+\KL(\hat{\pi}_{t+1}(\cdot|x)\|\hat{\pi}_t(\cdot|x))\right)
+\beta\KL(\pi^{\star}(\cdot|x)\|\hat{\pi}_{t+1}(\cdot|x))
. 
\end{aligned}
\]
Putting a negative sign on both sides, we have
\begin{equation}
\label{eq:A_1}
    \begin{aligned}
&\langle \pi^{\star}(\cdot|x) -\hat{\pi}_{t+1}(\cdot|x), S_{\boldsymbol{\lambda}_t}(x,\cdot) \rangle
-\beta\left(\KL(\pi^{\star}(\cdot|x)\|\pi_{\tref}(\cdot|x))-\KL(\hat{\pi}_{t+1}(\cdot|x)\|\pi_{\tref}(\cdot|x))\right)\\
&\leq\eta_{\theta}\left(\KL(\pi^{\star}(\cdot|x)\|\hat{\pi}_t(\cdot|x))
-\KL(\pi^{\star}(\cdot|x)\|\hat{\pi}_{t+1}(\cdot|x))
-\KL(\hat{\pi}_{t+1}(\cdot|x)\|\hat{\pi}_t(\cdot|x))\right)
-\beta\KL(\pi^{\star}(\cdot|x)\|\hat{\pi}_{t+1}(\cdot|x))
\\
&=\eta_{\theta}\KL(\pi^{\star}(\cdot|x)\|\hat{\pi}_t(\cdot|x))
-(\eta_{\theta}+\beta)\KL(\pi^{\star}(\cdot|x)\|\hat{\pi}_{t+1}(\cdot|x))
-\eta_{\theta}\KL(\hat{\pi}_{t+1}(\cdot|x)\|\hat{\pi}_t(\cdot|x))
\end{aligned}
\end{equation}
Similarly, the update of $\pi_t$ in the optimistic update is
\[
\pi_t
= \arg\max_{\pi}
\mathbb{E}_{x\sim \mathcal{D}}\left[\mathbb{E}_{y\sim\pi(\cdot|x)}\left[S_{\boldsymbol{\lambda}_{t-1}}(x,y)\right]
-\beta\mathrm{KL}\left(\pi(\cdot|x)\Vert\pi_{\mathrm{ref}}(\cdot|x)\right)
-\eta_{\theta}\mathrm{KL}\left(\pi(\cdot|x)\Vert\hat{\pi}_t(\cdot|x)\right)\right]
\]
let $\eta=\eta_{\theta}$, $g=S_{\boldsymbol{\lambda}_{t-1}}(x,\cdot)$, $\pi_{\old} = \hat{\pi}_t(\cdot|x)$, $\pi_{\new}=\pi_{t}(\cdot|x)$, and $\pi'=\hat{\pi}_{t+1}(\cdot|x)$ in \cref{lem:three-point-KL}, we have
\begin{equation}
\label{eq:A_2}
\begin{aligned}
&\langle \hat{\pi}_{t+1}(\cdot|x) -\pi_{t}(\cdot|x), S_{\boldsymbol{\lambda}_{t-1}}(x,\cdot) \rangle
-\beta\left(\KL(\hat{\pi}_{t+1}(\cdot|x)\|\pi_{\tref}(\cdot|x))-\KL(\pi_{t}(\cdot|x)\|\pi_{\tref}(\cdot|x))\right)\\
&\leq\eta_{\theta}\left(\KL(\hat{\pi}_{t+1}(\cdot|x)\|\hat{\pi}_t(\cdot|x))
-\KL(\hat{\pi}_{t+1}(\cdot|x)\|\pi_{t}(\cdot|x))
-\KL(\pi_{t}(\cdot|x)\|\hat{\pi}_t(\cdot|x))\right)
-\beta\KL(\hat{\pi}_{t+1}(\cdot|x)\|\pi_{t}(\cdot|x))
\\
&=\eta_{\theta}\KL(\hat{\pi}_{t+1}(\cdot|x)\|\hat{\pi}_t(\cdot|x))
-(\eta_{\theta}+\beta)\KL(\hat{\pi}_{t+1}(\cdot|x)\|\pi_{t}(\cdot|x))
-\eta_{\theta}\KL(\pi_{t}(\cdot|x)\|\hat{\pi}_t(\cdot|x))
.
\end{aligned}
\end{equation}
Let $C>0$ be a constant. For the last term in the RHS of \cref{eq:A-basic}, we derive the upper bound as 
\begin{equation}
\label{eq:A_3}
\begin{aligned}
&\langle\hat{\pi}_{t+1}(\cdot|x)-\pi_t(\cdot|x), S_{\boldsymbol{\lambda}_t}(x,\cdot)-S_{\boldsymbol{\lambda}_{t-1}}(x,\cdot)\rangle \\
 &=\left\langle\hat{\pi}_{t+1}(\cdot|x)-\pi_t(\cdot|x), \sum_{j\in\mathcal{H}}(\lambda_{t,j}-\lambda_{t-1,j})R_j(x,\cdot)\right\rangle \\  
 &\overset{(a)}{\leq} \|\hat{\pi}_{t+1}(\cdot|x)-\pi_t(\cdot|x)\|_1 \left\|\sum_{j\in\mathcal{H}}(\lambda_{t,j}-\lambda_{t-1,j})R_j(x,\cdot)\right\|_{\infty}\\
 &\overset{(b)}{\leq}  \|\hat{\pi}_{t+1}(\cdot|x)-\pi_t(\cdot|x)\|_1 \|\boldsymbol{\lambda}_{t}-\boldsymbol{\lambda}_{t-1}\|_{1}R_{\max} \\
 &\leq \sqrt{2\KL(\hat{\pi}_{t+1}(\cdot|x)\| \pi_t(\cdot|x))}\|\boldsymbol{\lambda}_{t}-\boldsymbol{\lambda}_{t-1}\|_{1}R_{\max}  \quad\text{(By Pinsker's inequality in \cref{lem:pinsker's ineq})}\\
&\leq C\KL(\hat{\pi}_{t+1}(\cdot|x)\| \pi_t(\cdot|x))+\frac{R_{\max}^2}{2C}\|\boldsymbol{\lambda}_{t}-\boldsymbol{\lambda}_{t-1}\|_{1}^2
\quad\text{(By AM-GM inequality $\frac{x^2}{2C}+\frac{y^2C}{2}\geq xy$ with $C>0$)}\\
 &\leq C\KL(\hat{\pi}_{t+1}(\cdot|x)\| \pi_t(\cdot|x))+\frac{|\mathcal{H}| R_{\max}^2}{2C}\|\boldsymbol{\lambda}_{t}-\boldsymbol{\lambda}_{t-1}\|_{2}^2 \quad\text{(By $\|\mathbf{x}\|_1^2\leq d\|\mathbf{x}\|_2^2$, $\forall \mathbf{x}\in\mathbb{R}^d$)}\\
 &\overset{(c)}{\leq} C\KL(\hat{\pi}_{t+1}(\cdot|x)\| \pi_t(\cdot|x))+\frac{|\mathcal{H}| R_{\max}^2}{C}\left(\|\boldsymbol{\lambda}_{t}-\hat{\boldsymbol{\lambda}}_{t}\|_{2}^2+\|\hat{\boldsymbol{\lambda}}_{t}-\boldsymbol{\lambda}_{t-1}\|_{2}^2\right),
\end{aligned}
\end{equation}
where $(a)$ is by Hölder's inequality shown in \cref{lem:holder's inequality} and letting $f(\cdot)=\hat{\pi}_{t+1}(\cdot|x)-\pi_t(\cdot|x)$, $g(\cdot)=\sum_{j\in\mathcal{H}}(\lambda_{t,j}-\lambda_{t-1,j})R_j(x,\cdot)$, $p=1$, and $q=\infty$.
$(b)$ is by
\[
\begin{aligned}
 \left\|\sum_{j\in\mathcal{H}}(\lambda_{t,j}-\lambda_{t-1,j})R_j(x,\cdot)\right\|_{\infty}
\leq& \sum_{j\in\mathcal{H}}\left\|(\lambda_{t,j}-\lambda_{t-1,j})R_j(x,\cdot)\right\|_{\infty} \quad\text{(by triangle inequality)}\\
=&  \sum_{j\in\mathcal{H}}|\lambda_{t,j}-\lambda_{t-1,j}|\left\|R_j(x,\cdot)\right\|_{\infty} \\
\leq& \left(\sum_{j\in\mathcal{H}}|\lambda_{t,j}-\lambda_{t-1,j}|\right)\max_j\left\|R_j(x,\cdot)\right\|_{\infty} \\
=&\|\boldsymbol{\lambda}_{t}-\boldsymbol{\lambda}_{t-1}\|_1 \max_j\left\|R_j(x,\cdot)\right\|_{\infty} \quad\text{(by definition of 1-norm)}\\
\leq& \|\boldsymbol{\lambda}_{t}-\boldsymbol{\lambda}_{t-1}\|_1 R_{\max} \\
&\text{(by \cref{assump:bounded} that $R_j(x,y)\leq R_{\max}$ for any $j\in \mathcal{S}\cup\mathcal{H}$ and $(x,y)$ pair).}
\end{aligned}
\]
$(c)$ is because
\[
\begin{aligned}
\|\boldsymbol{\lambda}_{t}-\boldsymbol{\lambda}_{t-1}\|_{2}^2
=&\|\boldsymbol{\lambda}_{t}-\hat{\boldsymbol{\lambda}}_{t}+\hat{\boldsymbol{\lambda}}_{t}-\boldsymbol{\lambda}_{t-1}\|_{2}^2\\
=&\|\boldsymbol{\lambda}_{t}-\hat{\boldsymbol{\lambda}}_{t}\|_{2}^2 + \|\hat{\boldsymbol{\lambda}}_{t}-\boldsymbol{\lambda}_{t-1}\|_{2}^2 + 2\langle\boldsymbol{\lambda}_{t}-\hat{\boldsymbol{\lambda}}_{t}, \hat{\boldsymbol{\lambda}}_{t}-\boldsymbol{\lambda}_{t-1}\rangle\\
\leq& 2\|\boldsymbol{\lambda}_{t}-\hat{\boldsymbol{\lambda}}_{t}\|_{2}^2 + 2\|\hat{\boldsymbol{\lambda}}_{t}-\boldsymbol{\lambda}_{t-1}\|_{2}^2 \\
&\text{(by Young's inequality with $p=q=2$, i.e., $\langle \mathbf{x},\mathbf{y}\rangle\leq \frac{1}{2}(\|\mathbf{x}\|_2^2+\|\mathbf{y}\|_2^2)$)}
\end{aligned}
\]
Substituting \cref{eq:A_1}, \cref{eq:A_2}, and \cref{eq:A_3} into the RHS of \cref{eq:A-basic}, we have
\begin{equation}
\label{eq:A_basic2}
\begin{aligned}
\mathrm{A}
\leq& \E_{x\sim\mathcal{D}}\left[\eta_{\theta}\KL(\pi^{\star}(\cdot|x)\|\hat{\pi}_t(\cdot|x))
-(\eta_{\theta}+\beta)\KL(\pi^{\star}(\cdot|x)\|\hat{\pi}_{t+1}(\cdot|x))
-\eta_{\theta}\KL(\hat{\pi}_{t+1}(\cdot|x)\|\hat{\pi}_t(\cdot|x))\right.\\
&\qquad\quad+ \eta_{\theta}\KL(\hat{\pi}_{t+1}(\cdot|x)\|\hat{\pi}_t(\cdot|x))
-(\eta_{\theta}+\beta)\KL(\hat{\pi}_{t+1}(\cdot|x)\|\pi_{t}(\cdot|x))
-\eta_{\theta}\KL(\pi_{t}(\cdot|x)\|\hat{\pi}_t(\cdot|x))\\
&\qquad\quad\left.+ C\KL(\hat{\pi}_{t+1}(\cdot|x)\| \pi_t(\cdot|x))+\frac{|\mathcal{H}| R_{\max}^2}{C}\left(\|\boldsymbol{\lambda}_{t}-\hat{\boldsymbol{\lambda}}_{t}\|_{2}^2+\|\hat{\boldsymbol{\lambda}}_{t}-\boldsymbol{\lambda}_{t-1}\|_{2}^2\right)\right]\\
=& \E_{x\sim\mathcal{D}}\left[\eta_{\theta}\KL(\pi^{\star}(\cdot|x)\|\hat{\pi}_t(\cdot|x))
-(\eta_{\theta}+\beta)\KL(\pi^{\star}(\cdot|x)\|\hat{\pi}_{t+1}(\cdot|x))
\right.\\
&\qquad\quad-(\eta_{\theta}+\beta)\KL(\hat{\pi}_{t+1}(\cdot|x)\|\pi_{t}(\cdot|x))
-\eta_{\theta}\KL(\pi_{t}(\cdot|x)\|\hat{\pi}_t(\cdot|x))\\
&\qquad\quad\left.+ C\KL(\hat{\pi}_{t+1}(\cdot|x)\| \pi_t(\cdot|x))\right]+\frac{|\mathcal{H}| R_{\max}^2}{C}\left(\|\boldsymbol{\lambda}_{t}-\hat{\boldsymbol{\lambda}}_{t}\|_{2}^2+\|\hat{\boldsymbol{\lambda}}_{t}-\boldsymbol{\lambda}_{t-1}\|_{2}^2\right)
\end{aligned}
\end{equation}

\subsection{Upper bound of term $\mathrm{B}$}
\label{sec:B_ub}
Similarly, we rewrite the term $\mathrm{B}$ as
\begin{equation}
    \begin{aligned}
\mathrm{B}
=&\E_{x\sim \mathcal{D}}\left[V^{\pi_t}_{S_{\boldsymbol{\lambda}_t}}(x)
- \beta \KL(\pi_t(\cdot|x) \| \pi_{\tref}(\cdot|x))\right]
-\E_{x\sim \mathcal{D}}\left[V^{\pi_t}_{S_{\boldsymbol{\lambda}^{\star}}}(x)
- \beta \KL(\pi_t(\cdot|x) \| \pi_{\tref}(\cdot|x))\right])\\
=& \E_{x\sim\mathcal{D}}\left[ V_{S_{\boldsymbol{\lambda}_t}}^{\pi_t}(x)
- V_{S_{\boldsymbol{\lambda}^{\star}}}^{\pi_{t}}(x)\right]\\
=&\E_{x\sim\mathcal{D}}\left[
\left(V_{S_{\boldsymbol{\lambda}_t}}^{\pi_t}(x) - V_{S_{\hat{\boldsymbol{\lambda}}_{t+1}}}^{\pi_{t}}(x)\right)
+\left(V_{S_{\hat{\boldsymbol{\lambda}}_{t+1}}}^{\pi_t}(x) - V_{S_{\boldsymbol{\lambda}^{\star}}}^{\pi_{t}}(x)\right)\right.\\
&\qquad\quad\left.-\left(V_{S_{\boldsymbol{\lambda}_t}}^{\pi_{t-1}}(x) - V_{S_{\hat{\boldsymbol{\lambda}}_{t+1}}}^{\pi_{t-1}}(x)\right)
+\left(V_{S_{\boldsymbol{\lambda}_t}}^{\pi_{t-1}}(x) - V_{S_{\hat{\boldsymbol{\lambda}}_{t+1}}}^{\pi_{t-1}}(x)\right)
\right]
    \end{aligned}
\end{equation}
Define $R_{\mathcal{H}}=\left[R_{j_1},R_{j_2},\cdots,R_{j_{|\mathcal{H}|}}\right]\in \mathbb{R}^{|\mathcal{H}|}$, where $j_1<j_2<\cdots<j_{|\mathcal{H}|}$, and $j_k\in \mathcal{H}$ for any integer $1\leq k \leq |\mathcal{H}|$. Define $V_{R_{\mathcal{H}}}^{\pi}(x)=\E_{y\sim \pi(\cdot|x)}R_{\mathcal{H}}(x,y)\in\mathbb{R}^{|\mathcal{H}|}$. By the definition of $V_{S_{\lambda}}^{\pi}(x)$, we have
\[
\begin{aligned}
    V_{S_{\lambda}}^{\pi}(x) 
    =&\E_{y\sim \pi(\cdot|x)}S_{\lambda}(x,y) \quad\text{(by definition of $V_{S_{\lambda}}^{\pi}(x) $)}\\
    =&\E_{y\sim \pi(\cdot|x)}\left[\sum_{j\in\mathcal{S}}w_jR_j(x,y)+\sum_{j\in\mathcal{H}}\lambda_jR_j(x,y)\right] \quad\text{(by definition of $S_{\lambda}$)}\\
    =&\E_{y\sim \pi(\cdot|x)}\left[\sum_{j\in\mathcal{S}}w_jR_j(x,y))\right]+\sum_{j\in\mathcal{H}}\lambda_j\E_{y\sim \pi(\cdot|x)}\left[R_j(x,y)\right]\\
    =&\E_{y\sim \pi(\cdot|x)}\left[\sum_{j\in\mathcal{S}}w_jR_j(x,y))\right]+\lambda^TV_{R_{\mathcal{H}}}^{\pi}(x) \quad\text{(by definition of $V_{R_{\mathcal{H}}}^{\pi}$)}
\end{aligned}
\]
Plugging the above expression of $V_{S_{\lambda}}^{\pi}(x)$ into term $\mathrm{B}$, we have
\begin{equation}
\label{eq:B_basic_rewrite1}
\begin{aligned}
 \mathrm{B}
 =&\E_{x\sim\mathcal{D}}\left[
 \left(\boldsymbol{\lambda}_t-\hat{\boldsymbol{\lambda}}_{t+1}\right)^TV_{R_{\mathcal{H}}}^{\pi_{t}}(x)
 +\left(\hat{\boldsymbol{\lambda}}_{t+1}-\boldsymbol{\lambda}^{\star}\right)^TV_{R_{\mathcal{H}}}^{\pi_{t}}(x)
 -\left(\boldsymbol{\lambda}_t-\hat{\boldsymbol{\lambda}}_{t+1}\right)^TV_{R_{\mathcal{H}}}^{\pi_{t-1}}(x)
 +\left(\boldsymbol{\lambda}_t-\hat{\boldsymbol{\lambda}}_{t+1}\right)^TV_{R_{\mathcal{H}}}^{\pi_{t-1}}(x)
\right]  \\
=&\E_{x\sim\mathcal{D}}\left[
\left(\hat{\boldsymbol{\lambda}}_{t+1}-\boldsymbol{\lambda}^{\star}\right)^TV_{R_{\mathcal{H}}}^{\pi_{t}}(x)
+\left(\boldsymbol{\lambda}_t-\hat{\boldsymbol{\lambda}}_{t+1}\right)^TV_{R_{\mathcal{H}}}^{\pi_{t-1}}(x)
+\left(\boldsymbol{\lambda}_t-\hat{\boldsymbol{\lambda}}_{t+1}\right)^T\left(V_{R_{\mathcal{H}}}^{\pi_{t}}(x)-V_{R_{\mathcal{H}}}^{\pi_{t-1}}(x)\right)\right] \\
=&
\left(\hat{\boldsymbol{\lambda}}_{t+1}-\boldsymbol{\lambda}^{\star}\right)^T\E_{x\sim\mathcal{D}}\left[V_{R_{\mathcal{H}}}^{\pi_{t}}(x)\right]
+\left(\boldsymbol{\lambda}_t-\hat{\boldsymbol{\lambda}}_{t+1}\right)^T\E_{x\sim\mathcal{D}}\left[V_{R_{\mathcal{H}}}^{\pi_{t-1}}(x)\right]\\
&+\left(\boldsymbol{\lambda}_t-\hat{\boldsymbol{\lambda}}_{t+1}\right)^T\left(\E_{x\sim\mathcal{D}}\left[V_{R_{\mathcal{H}}}^{\pi_{t}}(x)\right]-\E_{x\sim\mathcal{D}}\left[V_{R_{\mathcal{H}}}^{\pi_{t-1}}(x)\right]\right) 
\end{aligned}
\end{equation}

By \cref{lem:bdd_lambda}, $\lambda^*\geq 0$.
Without loss of generality, for a vector $\mathbf{a}$, we write $\mathbf{a} \ge 0$ to indicate that all entries of $\mathbf{a}$ are nonnegative. 
Recall \cref{eq:hatlambda_t} gives the $\hat{\lambda}_{t+1,j}$ in the optimistic gradient descent, and we rewrite the update as follows 
\[\hat{\lambda}_{t+1,j}
= \arg\min_{\lambda\geq 0}
\lambda\mathbb{E}_{x\sim \mathcal{D}, y\sim \pi_t(\cdot|x)}\big[R_j(x,y)\big]
+\eta_{\lambda}\big(\lambda-\hat{\lambda}_{t,j}\big)^2. \]
Rewrite the above update in the vectorized form as follows
\[
\begin{aligned}
 \hat{\boldsymbol{\lambda}}_{t+1}
=& \arg\min_{\boldsymbol{\lambda}\geq 0}
\boldsymbol{\lambda}^T\mathbb{E}_{x\sim \mathcal{D}}\big[V_{R_\mathcal{H}}^{\pi_t}(x)\big]
+\eta_{\lambda}\|\boldsymbol{\lambda}-\hat{\boldsymbol{\lambda}}_{t}\|_2^2\\   
=& \arg\max_{\boldsymbol{\lambda}\geq 0}
-\boldsymbol{\lambda}^T\mathbb{E}_{x\sim \mathcal{D}}\big[V_{R_\mathcal{H}}^{\pi_t}(x)\big]
-\eta_{\lambda}\|\boldsymbol{\lambda}-\hat{\boldsymbol{\lambda}}_{t}\|_2^2. 
\end{aligned}
\]
Let $\Omega=\mathbb{R}_+^{|\mathcal{H}|}$, $g=-\mathbb{E}_{x\sim \mathcal{D}}\big[V_{R_\mathcal{H}}^{\pi_t}(x)\big]$, $\eta=\eta_{\lambda}$, $h(\mathbf{x})=\|\mathbf{x}\|_2^2$, $D_h(\mathbf{x},\mathbf{y})=\|\mathbf{x}-\mathbf{y}\|_2^2$, $\mathbf{x}'=\boldsymbol{\lambda}^{\star}$, $\mathbf{x}_{\new}=\hat{\boldsymbol{\lambda}}_{t+1}$, $\mathbf{x}_{\old}=\hat{\boldsymbol{\lambda}}_t$ in \cref{lem:three-point-gradient}, then we have
\[
  \langle -\mathbb{E}_{x\sim \mathcal{D}}\big[V_{R_\mathcal{H}}^{\pi_t}(x)\big], \hat{\boldsymbol{\lambda}}_{t+1}-\boldsymbol{\lambda}^{\star}\rangle
\geq\eta_{\lambda}\left(-\|\boldsymbol{\lambda}^{\star}-\hat{\boldsymbol{\lambda}}_t\|_2^2
+\|\boldsymbol{\lambda}^{\star}-\hat{\boldsymbol{\lambda}}_{t+1}\|_2^2
+\|\hat{\boldsymbol{\lambda}}_{t+1}-\hat{\boldsymbol{\lambda}}_t\|_2^2\right).
\]
Putting a negative sign on both sides, we have
\begin{equation}
    \label{eq:B_1}
     \langle \hat{\boldsymbol{\lambda}}_{t+1}-\boldsymbol{\lambda}^{\star}, \mathbb{E}_{x\sim \mathcal{D}}\big[V_{R_\mathcal{H}}^{\pi_t}(x)\big] \rangle
\leq\eta_{\lambda}\left(\|\boldsymbol{\lambda}^{\star}-\hat{\boldsymbol{\lambda}}_t\|_2^2
-\|\boldsymbol{\lambda}^{\star}-\hat{\boldsymbol{\lambda}}_{t+1}\|_2^2
-\|\hat{\boldsymbol{\lambda}}_{t+1}-\hat{\boldsymbol{\lambda}}_t\|_2^2\right).
\end{equation}
Similarly, since \cref{eq:lambda_t} gives the optimistic update of $\lambda_{t,j}$ as follows,
\[
\lambda_{t,j}
= \arg\min_{\lambda\geq 0}
\lambda\mathbb{E}_{x\sim \mathcal{D}, y\sim \pi_{t-1}(\cdot|x)}\big[R_j(x,y)\big]
+\eta_{\lambda}\big(\lambda-\hat{\lambda}_{t,j}\big)^2,
\]
Applying \cref{lem:three-point-gradient} by setting $g=-\mathbb{E}_{x\sim \mathcal{D}}\big[V_{R_\mathcal{H}}^{\pi_{t-1}}(x)\big]$, $\eta=\eta_{\lambda}$, $h(\mathbf{x})=\|\mathbf{x}\|_2^2$, $D_h(\mathbf{x},\mathbf{y})=\|\mathbf{x}-\mathbf{y}\|_2^2$, $\mathbf{x}'=\hat{\boldsymbol{\lambda}}_{t+1}$, $\mathbf{x}_{\new}=\boldsymbol{\lambda}_{t}$, $\mathbf{x}_{\old}=\hat{\boldsymbol{\lambda}}_t$, we have
\begin{equation}
\label{eq:B_2}
     \langle \boldsymbol{\lambda}_{t}-\hat{\boldsymbol{\lambda}}_{t+1}, \mathbb{E}_{x\sim \mathcal{D}}\big[V_{R_\mathcal{H}}^{\pi_{t-1}}(x)\big] \rangle
\leq\eta_{\lambda}\left(\|\hat{\boldsymbol{\lambda}}_{t+1}-\hat{\boldsymbol{\lambda}}_t\|_2^2
-\|\hat{\boldsymbol{\lambda}}_{t+1}-\boldsymbol{\lambda}_{t}\|_2^2
-\|\boldsymbol{\lambda}_{t}-\hat{\boldsymbol{\lambda}}_t\|_2^2\right).
\end{equation}
We upper bound the last term of $\mathrm{B}$ as
\[
\begin{aligned}
&\left(\boldsymbol{\lambda}_t-\hat{\boldsymbol{\lambda}}_{t+1}\right)^T\left(\E_{x\sim\mathcal{D}}\left[V_{R_{\mathcal{H}}}^{\pi_{t}}(x)\right]-\E_{x\sim\mathcal{D}}\left[V_{R_{\mathcal{H}}}^{\pi_{t-1}}(x)\right]\right) \\
&=\E_{x\sim\mathcal{D}}\left[\left(\boldsymbol{\lambda}_t-\hat{\boldsymbol{\lambda}}_{t+1}\right)^T\left(V_{R_{\mathcal{H}}}^{\pi_{t}}(x)-V_{R_{\mathcal{H}}}^{\pi_{t-1}}(x)\right)\right] \\
&=\E_{x\sim\mathcal{D}}\left[\left(\boldsymbol{\lambda}_t-\hat{\boldsymbol{\lambda}}_{t+1}\right)^T\left(\langle \pi_t(\cdot|x), R_{\mathcal{H}}(x,\cdot)\rangle-\langle \pi_{t-1}(\cdot|x), R_{\mathcal{H}}(x,\cdot)\rangle\right)\right] \\
&\qquad\qquad\qquad\qquad\qquad\qquad\qquad\qquad\qquad\qquad
\text{(by $V_{R_{\mathcal{H}}}^{\pi}(x)=\E_{y\sim \pi(\cdot|x)}\left[R_{\mathcal{H}}(x,y)\right]=\langle \pi(\cdot|x), R_{\mathcal{H}}(x,\cdot)\rangle$)}\\
&= \E_{x\sim\mathcal{D}}\left[\langle\pi_{t}(\cdot|x)-\pi_{t-1}(\cdot|x), \left(\boldsymbol{\lambda}_t-\hat{\boldsymbol{\lambda}}_{t+1}\right)^TR_{\mathcal{H}}(x,\cdot)\rangle\right] \\  
&= \E_{x\sim\mathcal{D}}\left[\left\langle\pi_{t}(\cdot|x)-\pi_{t-1}(\cdot|x), \sum_{j\in\mathcal{H}}(\lambda_{t,j}-\hat{\lambda}_{t+1,j})R_j(x,\cdot)\right\rangle\right]
\end{aligned}
\]
Define constants $C_1>0$ and $C_2>0$. Fixing $x$, we have
\[
\begin{aligned}
&\left\langle\pi_{t}(\cdot|x)-\pi_{t-1}(\cdot|x), \sum_{j\in\mathcal{H}}(\lambda_{t,j}-\hat{\lambda}_{t+1,j})R_j(x,\cdot)\right\rangle\\
&=\left\langle\pi_{t}(\cdot|x)-\hat{\pi}_{t}(\cdot|x), \sum_{j\in\mathcal{H}}(\lambda_{t,j}-\hat{\lambda}_{t+1,j})R_j(x,\cdot)\right\rangle
+\left\langle\hat{\pi}_{t}(\cdot|x)-\pi_{t-1}(\cdot|x), \sum_{j\in\mathcal{H}}(\lambda_{t,j}-\hat{\lambda}_{t+1,j})R_j(x,\cdot)\right\rangle\\
&\leq
C_1\KL(\pi_{t}(\cdot|x)\| \hat{\pi}_{t}(\cdot|x))
+\frac{|\mathcal{H}| R_{\max}^2}{2C_1}\|\boldsymbol{\lambda}_{t}-\hat{\boldsymbol{\lambda}}_{t+1}\|_{2}^2
+
C_2\KL(\hat{\pi}_{t}(\cdot|x)\| \pi_{t-1}(\cdot|x))
+\frac{|\mathcal{H}| R_{\max}^2}{2C_2}\|\boldsymbol{\lambda}_{t}-\hat{\boldsymbol{\lambda}}_{t+1}\|_{2}^2\\
&\qquad\qquad\qquad\qquad\qquad\qquad\qquad\qquad\qquad\qquad\qquad\qquad\qquad\qquad\qquad\quad
\text{(By derivations of \cref{eq:A_3})}\\
&=C_1\KL(\pi_{t}(\cdot|x)\| \hat{\pi}_{t}(\cdot|x))
+C_2\KL(\hat{\pi}_{t}(\cdot|x)\| \pi_{t-1}(\cdot|x))
+|\mathcal{H}| R_{\max}^2\left(\frac{1}{2C_1}+\frac{1}{2C_2}\right)\|\boldsymbol{\lambda}_{t}-\hat{\boldsymbol{\lambda}}_{t+1}\|_{2}^2
\end{aligned}
\]
Substituting the above inequality into $\left(\boldsymbol{\lambda}_t-\hat{\boldsymbol{\lambda}}_{t+1}\right)^T\left(\E_{x\sim\mathcal{D}}\left[V_{R_{\mathcal{H}}}^{\pi_{t}}(x)\right]-\E_{x\sim\mathcal{D}}\left[V_{R_{\mathcal{H}}}^{\pi_{t-1}}(x)\right]\right)$, we have
\begin{equation}
\label{eq:B_3}
\begin{aligned}
&\left(\boldsymbol{\lambda}_t-\hat{\boldsymbol{\lambda}}_{t+1}\right)^T\left(\E_{x\sim\mathcal{D}}\left[V_{R_{\mathcal{H}}}^{\pi_{t}}(x)\right]-\E_{x\sim\mathcal{D}}\left[V_{R_{\mathcal{H}}}^{\pi_{t-1}}(x)\right]\right) \\
&=\E_{x\sim\mathcal{D}}\left[C_1\KL(\pi_{t}(\cdot|x)\| \hat{\pi}_{t}(\cdot|x))
+C_2\KL(\hat{\pi}_{t}(\cdot|x)\| \pi_{t-1}(\cdot|x))\right]
+|\mathcal{H}| R_{\max}^2\left(\frac{1}{2C_1}+\frac{1}{2C_2}\right)\|\boldsymbol{\lambda}_{t}-\hat{\boldsymbol{\lambda}}_{t+1}\|_{2}^2
\end{aligned}
\end{equation}
Combining \cref{eq:B_1}, \cref{eq:B_2}, and \cref{eq:B_3}, we can upper bound $\mathrm{B}$ as 
\begin{equation}
\label{eq:B_basic2}
\begin{aligned}
    \mathrm{B}
    \leq& \eta_{\lambda}\left(\|\boldsymbol{\lambda}^{\star}-\hat{\boldsymbol{\lambda}}_t\|_2^2
-\|\boldsymbol{\lambda}^{\star}-\hat{\boldsymbol{\lambda}}_{t+1}\|_2^2
-\|\hat{\boldsymbol{\lambda}}_{t+1}-\hat{\boldsymbol{\lambda}}_t\|_2^2\right)
+\eta_{\lambda}\left(\|\hat{\boldsymbol{\lambda}}_{t+1}-\hat{\boldsymbol{\lambda}}_t\|_2^2
-\|\hat{\boldsymbol{\lambda}}_{t+1}-\boldsymbol{\lambda}_{t}\|_2^2
-\|\boldsymbol{\lambda}_{t}-\hat{\boldsymbol{\lambda}}_t\|_2^2\right)\\
&+\E_{x\sim\mathcal{D}}\left[C_1\KL(\pi_{t}(\cdot|x)\| \hat{\pi}_{t}(\cdot|x))
+C_2\KL(\hat{\pi}_{t}(\cdot|x)\| \pi_{t-1}(\cdot|x))\right]
+|\mathcal{H}| R_{\max}^2\left(\frac{1}{2C_1}+\frac{1}{2C_2}\right)\|\boldsymbol{\lambda}_{t}-\hat{\boldsymbol{\lambda}}_{t+1}\|_{2}^2\\
=&\eta_{\lambda}\left(\|\boldsymbol{\lambda}^{\star}-\hat{\boldsymbol{\lambda}}_t\|_2^2
-\|\boldsymbol{\lambda}^{\star}-\hat{\boldsymbol{\lambda}}_{t+1}\|_2^2
-\|\boldsymbol{\lambda}_{t}-\hat{\boldsymbol{\lambda}}_{t+1}\|_2^2
-\|\boldsymbol{\lambda}_{t}-\hat{\boldsymbol{\lambda}}_t\|_2^2\right)\\
&+\E_{x\sim\mathcal{D}}\left[C_1\KL(\pi_{t}(\cdot|x)\| \hat{\pi}_{t}(\cdot|x))
+C_2\KL(\hat{\pi}_{t}(\cdot|x)\| \pi_{t-1}(\cdot|x))\right]
+|\mathcal{H}| R_{\max}^2\left(\frac{1}{2C_1}+\frac{1}{2C_2}\right)\|\boldsymbol{\lambda}_{t}-\hat{\boldsymbol{\lambda}}_{t+1}\|_{2}^2
\end{aligned}
\end{equation}

\subsection{Combining $\mathrm{A}$ and $\mathrm{B}$}
\label{sec:combine_ab}
Substituting \cref{eq:A_basic2} and \cref{eq:B_basic2} into the RHS of \cref{eq:AB-split}, we get 
\[
\begin{aligned}
&L(\pi^{\star},\boldsymbol{\lambda}_t) - L(\pi_t,\boldsymbol{\lambda}^{\star})\\
&\leq\E_{x\sim\mathcal{D}}\left[\eta_{\theta}\KL(\pi^{\star}(\cdot|x)\|\hat{\pi}_t(\cdot|x))
-(\eta_{\theta}+\beta)\KL(\pi^{\star}(\cdot|x)\|\hat{\pi}_{t+1}(\cdot|x))
{\color{red}-(\eta_{\theta}+\beta)\KL(\hat{\pi}_{t+1}(\cdot|x)\|\pi_{t}(\cdot|x))}
\right.\\
&\qquad\qquad
\left.
{\color{blue}-\eta_{\theta}\KL(\pi_{t}(\cdot|x)\|\hat{\pi}_t(\cdot|x))}
{\color{red}+ C\KL(\hat{\pi}_{t+1}(\cdot|x)\|\pi_t(\cdot|x))}\right]\\
&\quad+\frac{|\mathcal{H}| R_{\max}^2}{C}\left({\color{magenta}\|\boldsymbol{\lambda}_{t}-\hat{\boldsymbol{\lambda}}_{t}\|_{2}^2}+\|\hat{\boldsymbol{\lambda}}_{t}-\boldsymbol{\lambda}_{t-1}\|_{2}^2\right)
+\eta_{\lambda}\left(\|\boldsymbol{\lambda}^{\star}-\hat{\boldsymbol{\lambda}}_t\|_2^2
-\|\boldsymbol{\lambda}^{\star}-\hat{\boldsymbol{\lambda}}_{t+1}\|_2^2
{\color{cyan}-\|\boldsymbol{\lambda}_{t}-\hat{\boldsymbol{\lambda}}_{t+1}\|_2^2}
{-\color{magenta}\|\boldsymbol{\lambda}_{t}-\hat{\boldsymbol{\lambda}}_t\|_2^2}\right)\\
&\quad+\E_{x\sim\mathcal{D}}\left[{\color{blue}C_1\KL(\pi_{t}(\cdot|x)\| \hat{\pi}_{t}(\cdot|x))}
+C_2\KL(\hat{\pi}_{t}(\cdot|x)\| \pi_{t-1}(\cdot|x))\right]
+{\color{cyan}|\mathcal{H}| R_{\max}^2\left(\frac{1}{2C_1}+\frac{1}{2C_2}\right)\|\boldsymbol{\lambda}_{t}-\hat{\boldsymbol{\lambda}}_{t+1}\|_{2}^2}\\
&\leq\E_{x\sim\mathcal{D}}\left[\eta_{\theta}\KL(\pi^{\star}(\cdot|x)\|\hat{\pi}_t(\cdot|x))
-(\eta_{\theta}+\beta)\KL(\pi^{\star}(\cdot|x)\|\hat{\pi}_{t+1}(\cdot|x))
-(\eta_{\theta}+\beta-C)\KL(\hat{\pi}_{t+1}(\cdot|x)\|\pi_{t}(\cdot|x))
\right.\\
&\qquad\qquad
\left.
-\left(\eta_{\theta}-C_1\right)\KL(\pi_{t}(\cdot|x)\|\hat{\pi}_t(\cdot|x))
+C_2\KL(\hat{\pi}_{t}(\cdot|x)\| \pi_{t-1}(\cdot|x))\right]\\
&\quad
-\left(\eta_{\lambda}-\frac{|\mathcal{H}| R_{\max}^2}{C}\right)\|\boldsymbol{\lambda}_{t}-\hat{\boldsymbol{\lambda}}_{t}\|_{2}^2+\frac{|\mathcal{H}| R_{\max}^2}{C}\|\hat{\boldsymbol{\lambda}}_{t}-\boldsymbol{\lambda}_{t-1}\|_{2}^2
+\eta_{\lambda}\|\boldsymbol{\lambda}^{\star}-\hat{\boldsymbol{\lambda}}_t\|_2^2
-\eta_{\lambda}\|\boldsymbol{\lambda}^{\star}-\hat{\boldsymbol{\lambda}}_{t+1}\|_2^2\\
&\quad
-\left(\eta_{\lambda}-|\mathcal{H}| R_{\max}^2\left(\frac{1}{2C_1}+\frac{1}{2C_2}\right)\right)\|\boldsymbol{\lambda}_{t}-\hat{\boldsymbol{\lambda}}_{t+1}\|_2^2,
\end{aligned}
\]
where terms sharing the same color can be combined.
Recall \cref{eq:L(pi*)-L(lambda*)>0} that $L(\pi^{\star},\boldsymbol{\lambda}_t) - L(\pi_t,\boldsymbol{\lambda}^{\star})\geq 0$. Substituting this into the above equation and rearranging the equation, we have
\begin{equation}
\label{eq:Phi_t_contraction_0}
\begin{aligned}
&
(\eta_{\theta}+\beta)\E_{x\sim\mathcal{D}}\left[\KL(\pi^{\star}(\cdot|x)\|\hat{\pi}_{t+1}(\cdot|x))\right]
+(\eta_{\theta}+\beta-C)\E_{x\sim\mathcal{D}}\left[\KL(\hat{\pi}_{t+1}(\cdot|x)\|\pi_{t}(\cdot|x))\right]\\
&
+\eta_{\lambda}\|\boldsymbol{\lambda}^{\star}-\hat{\boldsymbol{\lambda}}_{t+1}\|_2^2
+\left(\eta_{\lambda}-|\mathcal{H}| R_{\max}^2\left(\frac{1}{2C_1}+\frac{1}{2C_2}\right)\right)\|\boldsymbol{\lambda}_{t}-\hat{\boldsymbol{\lambda}}_{t+1}\|_2^2\\
&\leq
\eta_{\theta}\E_{x\sim\mathcal{D}}\left[\KL(\pi^{\star}(\cdot|x)\|\hat{\pi}_t(\cdot|x))\right]
+C_2\E_{x\sim\mathcal{D}}\left[\KL(\hat{\pi}_{t}(\cdot|x)\| \pi_{t-1}(\cdot|x))\right]
+\eta_{\lambda}\|\boldsymbol{\lambda}^{\star}-\hat{\boldsymbol{\lambda}}_t\|_2^2
+\frac{|\mathcal{H}| R_{\max}^2}{C}\|\hat{\boldsymbol{\lambda}}_{t}-\boldsymbol{\lambda}_{t-1}\|_{2}^2
\\
&\quad
-\left(\eta_{\theta}-C_1\right)\E_{x\sim\mathcal{D}}\left[\KL(\pi_{t}(\cdot|x)\|\hat{\pi}_t(\cdot|x))\right]
-\left(\eta_{\lambda}-\frac{|\mathcal{H}| R_{\max}^2}{C}\right)\|\boldsymbol{\lambda}_{t}-\hat{\boldsymbol{\lambda}}_{t}\|_{2}^2
\end{aligned}
\end{equation}
Note that for any $\delta,\theta>0$, we have
\begin{equation}
\label{eq:split_lambdat-hatlambdat}
    \begin{aligned}
\|\boldsymbol{\lambda}_t-\hat{\boldsymbol{\lambda}}_t\|^2
&= \left\|\left(\boldsymbol{\lambda}_t-\hat{\boldsymbol{\lambda}}_{t+1}\right)+\left(\hat{\boldsymbol{\lambda}}_{t+1}-\boldsymbol{\lambda}^\star\right)+\left(\boldsymbol{\lambda}^\star-\hat{\boldsymbol{\lambda}}_t\right)\right\|^2\\
&\ge 
(1-\delta)\|\hat{\boldsymbol{\lambda}}_{t+1}-\boldsymbol{\lambda}^{\star}\|^2
+(1-\frac{1}{\delta})(1-\theta)\|\boldsymbol{\lambda}_t-\hat{\boldsymbol{\lambda}}_{t+1}\|^2
+(1-\frac{1}{\delta})(1-\frac{1}{\theta})\|\boldsymbol{\lambda}^{\star}-\hat{\boldsymbol{\lambda}}_t\|^2 \quad\text{(by \cref{lem:||a+b+c||_2^2_lower_bd})}.
\end{aligned}
\end{equation}
Substituting \cref{eq:split_lambdat-hatlambdat} into the last term of the RHS of \cref{eq:Phi_t_contraction_0} to get
\[
\begin{aligned}
&
(\eta_{\theta}+\beta)\E_{x\sim\mathcal{D}}\left[\KL(\pi^{\star}(\cdot|x)\|\hat{\pi}_{t+1}(\cdot|x))\right]
+(\eta_{\theta}+\beta-C)\E_{x\sim\mathcal{D}}\left[\KL(\hat{\pi}_{t+1}(\cdot|x)\|\pi_{t}(\cdot|x))\right]\\
&
{\color{red}+\eta_{\lambda}\|\boldsymbol{\lambda}^{\star}-\hat{\boldsymbol{\lambda}}_{t+1}\|_2^2}
+{\color{blue}\left(\eta_{\lambda}-|\mathcal{H}| R_{\max}^2\left(\frac{1}{2C_1}+\frac{1}{2C_2}\right)\right)\|\boldsymbol{\lambda}_{t}-\hat{\boldsymbol{\lambda}}_{t+1}\|_2^2}\\
&\leq
\eta_{\theta}\E_{x\sim\mathcal{D}}\left[\KL(\pi^{\star}(\cdot|x)\|\hat{\pi}_t(\cdot|x))\right]
+C_2\E_{x\sim\mathcal{D}}\left[\KL(\hat{\pi}_{t}(\cdot|x)\| \pi_{t-1}(\cdot|x))\right]
+{\color{cyan}\eta_{\lambda}\|\boldsymbol{\lambda}^{\star}-\hat{\boldsymbol{\lambda}}_t\|_2^2}
+\frac{|\mathcal{H}| R_{\max}^2}{C}\|\hat{\boldsymbol{\lambda}}_{t}-\boldsymbol{\lambda}_{t-1}\|_{2}^2
\\
&\quad
-\left(\eta_{\theta}-C_1\right)\E_{x\sim\mathcal{D}}\left[\KL(\pi_{t}(\cdot|x)\|\hat{\pi}_t(\cdot|x))\right]
{\color{red}-\left(\eta_{\lambda}-\frac{|\mathcal{H}| R_{\max}^2}{C}\right)(1-\delta)\|\hat{\boldsymbol{\lambda}}_{t+1}-\boldsymbol{\lambda}^{\star}\|^2}\\
&\quad
{\color{blue}-\left(\eta_{\lambda}-\frac{|\mathcal{H}| R_{\max}^2}{C}\right)
(1-\frac{1}{\delta})(1-\theta)\|\boldsymbol{\lambda}_t-\hat{\boldsymbol{\lambda}}_{t+1}\|^2}
{\color{cyan}-\left(\eta_{\lambda}-\frac{|\mathcal{H}| R_{\max}^2}{C}\right)(1-\frac{1}{\delta})(1-\frac{1}{\theta})\|\boldsymbol{\lambda}^{\star}-\hat{\boldsymbol{\lambda}}_t\|^2},
\end{aligned}
\]
where terms of the same color can be combined.
Rearranging the above equation, we have
\begin{equation}
\label{eq:Phi_t_contraction_1}
\begin{aligned}
&
(\eta_{\theta}+\beta)\E_{x\sim\mathcal{D}}\left[\KL(\pi^{\star}(\cdot|x)\|\hat{\pi}_{t+1}(\cdot|x))\right]
+(\eta_{\theta}+\beta-C)\E_{x\sim\mathcal{D}}\left[\KL(\hat{\pi}_{t+1}(\cdot|x)\|\pi_{t}(\cdot|x))\right]\\
&
+\left(\eta_{\lambda}+\left(\eta_{\lambda}-\frac{|\mathcal{H}| R_{\max}^2}{C}\right)(1-\delta)\right)\|\boldsymbol{\lambda}^{\star}-\hat{\boldsymbol{\lambda}}_{t+1}\|_2^2\\
&
+\left(\eta_{\lambda}-|\mathcal{H}| R_{\max}^2\left(\frac{1}{2C_1}+\frac{1}{2C_2}\right)+\left(\eta_{\lambda}-\frac{|\mathcal{H}| R_{\max}^2}{C}\right)
(1-\frac{1}{\delta})(1-\theta)\right)\|\hat{\boldsymbol{\lambda}}_{t+1}-\boldsymbol{\lambda}_{t}\|_2^2\\
&\leq
\eta_{\theta}\E_{x\sim\mathcal{D}}\left[\KL(\pi^{\star}(\cdot|x)\|\hat{\pi}_t(\cdot|x))\right]
+C_2\E_{x\sim\mathcal{D}}\left[\KL(\hat{\pi}_{t}(\cdot|x)\| \pi_{t-1}(\cdot|x))\right]\\
&\quad
+\left(\eta_{\lambda}-\left(\eta_{\lambda}-\frac{|\mathcal{H}| R_{\max}^2}{C}\right)(1-\frac{1}{\delta})(1-\frac{1}{\theta})\right)\|\boldsymbol{\lambda}^{\star}-\hat{\boldsymbol{\lambda}}_t\|_2^2\\
&\quad
+\frac{|\mathcal{H}| R_{\max}^2}{C}\|\hat{\boldsymbol{\lambda}}_{t}-\boldsymbol{\lambda}_{t-1}\|_{2}^2
-\left(\eta_{\theta}-C_1\right)\E_{x\sim\mathcal{D}}\left[\KL(\pi_{t}(\cdot|x)\|\hat{\pi}_t(\cdot|x))\right]
\end{aligned}
\end{equation}

Define $\Phi_t$ as the LHS of \cref{eq:Phi_t_contraction_1}, i.e.,
\[
\begin{aligned}
 \Phi_{t+1}:=&
(\eta_{\theta}+\beta)\E_{x\sim\mathcal{D}}\left[\KL(\pi^{\star}(\cdot|x)\|\hat{\pi}_{t+1}(\cdot|x))\right]
+(\eta_{\theta}+\beta-C)\E_{x\sim\mathcal{D}}\left[\KL(\hat{\pi}_{t+1}(\cdot|x)\|\pi_{t}(\cdot|x))\right]\\
&
+\left(\eta_{\lambda}+\left(\eta_{\lambda}-\frac{|\mathcal{H}| R_{\max}^2}{C}\right)(1-\delta)\right)\|\boldsymbol{\lambda}^{\star}-\hat{\boldsymbol{\lambda}}_{t+1}\|_2^2\\
&
+\left(\eta_{\lambda}-|\mathcal{H}| R_{\max}^2\left(\frac{1}{2C_1}+\frac{1}{2C_2}\right)+\left(\eta_{\lambda}-\frac{|\mathcal{H}| R_{\max}^2}{C}\right)
(1-\frac{1}{\delta})(1-\theta)\right)\|\hat{\boldsymbol{\lambda}}_{t+1}-\boldsymbol{\lambda}_{t}\|_2^2
\end{aligned}
\]
If the following requirements are satisfied:
\begin{enumerate}
    \item Multipliers of all terms of LHS of \cref{eq:Phi_t_contraction_1} are positive:
    \[
    \begin{aligned}
        &\eta_{\theta}+\beta>0, \\
        &\eta_{\theta}+\beta-C>0, \\
        &\eta_{\lambda}+\left(\eta_{\lambda}-\frac{|\mathcal{H}| R_{\max}^2}{C}\right)(1-\delta)>0, \\
        &\eta_{\lambda}-|\mathcal{H}| R_{\max}^2\left(\frac{1}{2C_1}+\frac{1}{2C_2}\right)+\left(\eta_{\lambda}-\frac{|\mathcal{H}| R_{\max}^2}{C}\right)(1-\frac{1}{\delta})(1-\theta)>0.\\
    \end{aligned}
    \]
    \item Multipliers of all terms of RHS of \cref{eq:Phi_t_contraction_1} are positive:
    \[
    \begin{aligned}
        &\eta_{\theta}>0,\\
        &C_2>0,\\
        &\eta_{\lambda}-\left(\eta_{\lambda}-\frac{|\mathcal{H}| R_{\max}^2}{C}\right)(1-\frac{1}{\delta})(1-\frac{1}{\theta})>0,\\ &\frac{|\mathcal{H}| R_{\max}^2}{C}>0,\\
        &\eta_{\theta}-C_1>0.
    \end{aligned}
    \]
    \item Define 
    \[
    \begin{aligned}
      \rho: = &\max\left(
    \frac{\eta_{\theta}}{\eta_{\theta}+\beta},
    \frac{C_2}{ \eta_{\theta}+\beta-C},
    \frac{\eta_{\lambda}-\left(\eta_{\lambda}-\frac{|\mathcal{H}| R_{\max}^2}{C}\right)(1-\frac{1}{\delta})(1-\frac{1}{\theta})}{\eta_{\lambda}+\left(\eta_{\lambda}-\frac{|\mathcal{H}| R_{\max}^2}{C}\right)(1-\delta)},\right.\\
    &\qquad\qquad
    \left.\frac{\frac{|\mathcal{H}| R_{\max}^2}{C}}{\eta_{\lambda}-|\mathcal{H}| R_{\max}^2\left(\frac{1}{2C_1}+\frac{1}{2C_2}\right)+\left(\eta_{\lambda}-\frac{|\mathcal{H}| R_{\max}^2}{C}\right)
(1-\frac{1}{\delta})(1-\theta)}
    \right),      
    \end{aligned}
    \]
    then $\rho<1$.
\end{enumerate}
Then \cref{eq:Phi_t_contraction_1} can be written as
\[
\begin{aligned}
 \Phi_{t+1} \leq& \rho\Phi_t -\left(\eta_{\theta}-C_1\right)\E_{x\sim\mathcal{D}}\left[\KL(\pi_{t}(\cdot|x)\|\hat{\pi}_t(\cdot|x))\right]
 \quad\text{(by the definition of $\rho$)}\\
 \leq& \rho\Phi_t 
 \quad\text{(by $\eta-C_1>0$ and $\E_{x\sim\mathcal{D}}\left[\KL(\pi_{t}(\cdot|x)\|\hat{\pi}_t(\cdot|x))\right]>0$)}
\end{aligned}
\]
Iteratively apply the recursion, we have $\Phi_t\leq\rho^{t}\Phi_1$, where 
\[
\begin{aligned}
\Phi_1=&(\eta_{\theta}+\beta)\E_{x\sim\mathcal{D}}\left[\KL(\pi^{\star}(\cdot|x)\|\hat{\pi}_{1}(\cdot|x))\right]
+(\eta_{\theta}+\beta-C)\E_{x\sim\mathcal{D}}\left[\KL(\hat{\pi}_{1}(\cdot|x)\|\pi_{0}(\cdot|x))\right]\\
&
+\left(\eta_{\lambda}+\left(\eta_{\lambda}-\frac{|\mathcal{H}| R_{\max}^2}{C}\right)(1-\delta)\right)\|\boldsymbol{\lambda}^{\star}-\hat{\boldsymbol{\lambda}}_{1}\|_2^2\\
&
+\left(\eta_{\lambda}-|\mathcal{H}| R_{\max}^2\left(\frac{1}{2C_1}+\frac{1}{2C_2}\right)+\left(\eta_{\lambda}-\frac{|\mathcal{H}| R_{\max}^2}{C}\right)
(1-\frac{1}{\delta})(1-\theta)\right)\|\hat{\boldsymbol{\lambda}}_{1}-\boldsymbol{\lambda}_{0}\|_2^2.    
\end{aligned}
\]

\cref{assump:reference_policy_full_support} guarantees support of the policy does not shrink along the OPD iterates, and also ensures that the KL terms $\KL({\pi(y|x)}\mid{\pi_{\tref}(y|x)})$ and $\KL({\pi(y|x)}\mid{\hat{\pi}_{t}(y|x)})$ are well-defined throughout iterations.
We initialize $\pi_0$ to have the same support as $\pi_{\tref}$ and set $\hat{\pi}_0=\pi_{\tref}$. From the closed-form solution of the KL-regularized maximization in \cref{eq:pi_t,eq:hatpi_t}, the policy updates have the form
\[
\pi_t(y|x)\propto\hat{\pi}_t(y|x)^{\frac{\eta_\theta}{\eta_\theta+\beta}}
\pi_{\tref}(y|x)^{\frac{\beta}{\eta_\theta+\beta}}
\exp\left(S_{\lambda_{t-1}}(x,y)\right),
\]
and similarly,
\[
\hat{\pi}_{t+1}(y|x)\propto\hat{\pi}_t(y|x)^{\frac{\eta_\theta}{\eta_\theta+\beta}}
\pi_{\tref}(y|x)^{\frac{\beta}{\eta_\theta+\beta}}
\exp\left(S_{\lambda_{t}}(x,y)\right).
\]
Since all factors on the right-hand side are strictly positive whenever $\pi_{\tref}(y|x)>0$, it follows by induction that for all $x\in\mathcal{X}$ and iteration $i$, we have
\[
\supp(\pi_t(\cdot|x))=\supp(\hat{\pi}_t(\cdot|x))=\supp(\pi_{\tref}(\cdot|x)).
\]
Therefore, $\Phi_1$ is bounded.

Furthermore, we have
\[
\begin{aligned}
\E_{x\sim\mathcal{D}}\left[\KL(\pi^{\star}(\cdot|x)\|\hat{\pi}_{t}(\cdot|x))\right]
+\|\boldsymbol{\lambda}^{\star}-\hat{\boldsymbol{\lambda}}_{t}\|_2^2
\leq \rho^{t}\frac{\Phi_1}{\rho\max\left(\eta_{\theta}+\beta,\eta_{\lambda}+\left(\eta_{\lambda}-\frac{|\mathcal{H}| R_{\max}^2}{C}\right)(1-\delta)\right)}
\end{aligned}
\]
and this shows the desired result.

\paragraph{Hyperparameters and Constants Selection}
Our next step is to choose hyperparameters $\eta_{\theta}$ and $\eta_{\lambda}$ as well as constants $C_1$, $C_2$, and $C$ to satisfy the requirements. 
For simplicity, with a little abuse of notations, we denote $h=|\mathcal{H}|$ and $R=R_{\max}$ in this parameter and constants selection section.
Let
\[
\eta_{\theta}=\eta_{\lambda}=\eta=3\sqrt{h}R,\quad
C_1=C_2=C=\sqrt{h}R,\quad
\frac{1}{2}<\delta<1,\quad
\frac{1}{2}<\theta<1.
\]
We will verify that this set of parameters satisfies the requirements.

\begin{enumerate}
    \item \textbf{Verifications that multipliers of all terms of LHS of \cref{eq:Phi_t_contraction_1} are positive}.
    (1) Since $\eta_{\theta}>0$ and $\beta>0$, we have $\eta_{\theta}+\beta>0$. 
    (2) $\eta_{\theta}+\beta-C=\beta+2\sqrt{h}R>0$. 
    (3) Since $\eta_{\lambda}-hR^2/C=3\sqrt{h}R-hR^2/(\sqrt{h}R)=2\sqrt{h}R>0$ and $\delta<1$, we have $\left(\eta_{\lambda}-hR^2/C\right)(1-\delta)>0$. Hence $\eta_{\lambda}+\left(\eta_{\lambda}-hR^2/C\right)(1-\delta)>0$.
    (4) $\eta_{\lambda}-hR^2\left(1/(2C_1)+1/(2C_2)\right)+\left(\eta_{\lambda}-hR^2/C\right)(1-1/\delta)(1-\theta)
    =3\sqrt{h}R-hR^2/(\sqrt{h}R)+(2\sqrt{h}R-hR^2/(\sqrt{h}R))(1-1/\delta)(1-\theta)
    =\sqrt{h}R\left(2+\left(1-1/\delta\right)\left(1-\theta\right)\right)$.
    Since $1/2<\delta,\theta<1$, we have $-1<1-1/\delta<0$ and $0<1-\theta<1/2$, hence $-1/2<\left(1-1/\delta\right)\left(1-\theta\right)<0$. 
    Therefore, $2+\left(1-1/\delta\right)\left(1-\theta\right)>0$.

    \item \textbf{Verifications that the multipliers of all terms of the RHS of \cref{eq:Phi_t_contraction_1} are positive}. 
    (1) $\eta_{\theta}>0$ by the definition of $\eta_{\theta}$. 
    (2) $C_2>0$ by the definition of $C_2$. 
    (3) $\eta_{\lambda}-\left(\eta_{\lambda}-hR^2/C\right)(1-1/\delta)(1-1/\theta)
    = 3\sqrt{h}R-(3\sqrt{h}R-hR^2/(\sqrt{h}R)(1-1/
    \delta)(1-1/\theta)
    =\sqrt{h}R(3-2(1-1/\delta)(1-1/\theta))$. 
    Since $1/2<\delta,\theta<1$, we have $-1<1-1/\theta<0$ and $-1<1-1/\delta<0$, hence $0<(1-1/\theta)(1-1/\delta)<1$.
    Therefore, we get $3-2(1-1/\delta)(1-1/\theta)>0$.
    (4) As $hR^2>0$ and $C=\sqrt{hR}>0$, we have $hR^2/C>0$.
    (5) $\eta_{\theta}-C_1=3\sqrt{hR}-\sqrt{h}R=2\sqrt{h}R>0$.

    \item 
    (1) Since $\eta_{\theta}>0$ and $\beta>0$, we have $\frac{\eta_{\theta}}{\eta_{\theta}+\beta}<1$.
    (2) We have $C_2/(\eta_{\theta}+\beta-C)=\sqrt{h}R/(3\sqrt{h}R-\beta-\sqrt{hR})<\sqrt{h}R/(3\sqrt{h}R-\sqrt{hR})=1/2$, where the inequality is because $\beta>0$.
    (3) Since $\eta_{\lambda}-hR^2/C=3\sqrt{hR}-hR^2/(\sqrt{h}R)=2\sqrt{h}R>0$, $(1-1/\delta)(1-1/\theta)>0$, and $1-\delta>0$, we have $\left(\eta_{\lambda}-hR^2/C\right)(1-1/\delta)(1-1/\theta)>0$ and $\left(\eta_{\lambda}-hR^2/C\right)(1-\delta)>0$, hence $\eta_{\lambda}-\left(\eta_{\lambda}-hR^2/C\right)(1-1/\delta)(1-1/\theta)<\eta_{\lambda}+\left(\eta_{\lambda}-hR^2/C\right)(1-\delta)$, and $(\eta_{\lambda}-\left(\eta_{\lambda}-hR^2/C\right)(1-1/\delta)(1-1/\theta))/(\eta_{\lambda}+\left(\eta_{\lambda}-hR^2/C\right)(1-\delta))<1$.
    (4)Plugging the parameters values into the last requirement, we have
    \[
\begin{aligned}
&\frac{{hR^2}/{C}}{\eta_{\lambda}-hR^2\left({1}/{2C_1}+{1}/{2C_2}\right)+(\eta_{\lambda}-{hR^2}/{C})(1-1/\delta)(1-\theta)}\\
&=\frac{\sqrt{h}R}{2\sqrt{h}R +2\sqrt{h}R(1-1/\delta)(1-\theta)}\\
&=\frac{1}{2+2(1-1/\delta)(1-\theta)}.
\end{aligned}
\]
As $1/2<\theta,\delta<1$, $-1/2<(1-1/\delta)(1-\theta)<0$. Therefore $1<2+2(1-1/\delta)(1-\theta)<2$ and ${1}/{2+2(1-1/\delta)(1-\theta)}<1$.
\end{enumerate}

If we further set $\delta=\theta=\frac{3}{4}$, then we can write $\rho$ and $\Phi_1$ as
 \begin{equation}
    \label{eq:rho_valued} 
      \rho = \max\left(
    \frac{3\sqrt{\mathcal{H}}R_{\max}}{3\sqrt{\mathcal{H}}R_{\max}+\beta},
    \frac{\sqrt{\mathcal{H}}R_{\max}}{2\sqrt{\mathcal{H}}R_{\max}+\beta},
    \frac{50}{63}\right),      
 \end{equation}
 \begin{equation}
    \label{eq:Phi_1_valued} 
\begin{aligned}
\Phi_1=&(3\sqrt{\mathcal{H}}R_{\max}+\beta)\E_{x\sim\mathcal{D}}\left[\KL(\pi^{\star}(\cdot|x)\|\hat{\pi}_{1}(\cdot|x))\right]
+(2\sqrt{\mathcal{H}}R_{\max}+\beta)\E_{x\sim\mathcal{D}}\left[\KL(\hat{\pi}_{1}(\cdot|x)\|\pi_{0}(\cdot|x))\right]\\
&
+\frac{7}{2}\sqrt{\mathcal{H}}R_{\max}\|\boldsymbol{\lambda}^{\star}-\hat{\boldsymbol{\lambda}}_{1}\|_2^2
+\frac{11}{6}\sqrt{\mathcal{H}}R_{\max}\|\hat{\boldsymbol{\lambda}}_{1}-\boldsymbol{\lambda}_{0}\|_2^2.    
\end{aligned}
 \end{equation}

\section{Proof of \cref{prop:OPD-NPG-equivalence}}

In this section, we show under the tabular softmax parameterization, the updated policy $\pi_{\theta_{+}}$ is equivalent to $\pi_{+}$ with $\theta_{+}$ and $\pi_{+}$ shown as follows.
\begin{equation}
\label{eq:theta+}
\theta_{+}=\theta+\frac{1}{\eta_{\theta}+\beta}F(\theta)^\dagger\nabla_\theta \mathcal{L}(\pi_{\theta},\lambda),
\end{equation}
\begin{equation}
\label{eq:pi+}
    \pi_+
= \arg\max_{\pi}
\mathbb{E}_{x\sim \mathcal{D}}\left[\mathbb{E}_{y\sim\pi(\cdot|x)}\left[S_{\lambda}(x,y)\right]
-\beta\mathrm{KL}\left(\pi(\cdot|x)\Vert\pi_{\mathrm{ref}}(\cdot|x)\right)
-\eta_{\theta}\mathrm{KL}\left(\pi(\cdot|x)\Vert\pi_{\theta}(\cdot|x)\right)\right],
\end{equation}
where $\mathcal{L}(\pi_\theta,\lambda)=\mathbb{E}_{x\sim \mathcal{D}}\left[\E_{y\sim\pi_{\theta}(\cdot|x)}S_{\lambda}(x,y)
- \beta \mathrm{KL}(\pi(\cdot|x) \| \pi_{\mathrm{ref}}(\cdot|x))\right]$ and $S_{\lambda}(x,y)
= \sum_{k \in \mathcal{S}} w_k R_k(x,y)
   + \sum_{j \in \mathcal{H}} \lambda_j R_j(x,y)$.
   
If we let $\theta=\hat{\theta}_t$ and $\lambda=\lambda_{t-1}$, then the above equivalence proves that the $\pi_{\theta_t}$ with NPG update shown in \cref{eq:npg_thetat} and $\pi_t$ with OPG update shown in \cref{eq:pi_t} are the same under the tabular softmax parameterized distribution. Similarly, if we let $\theta=\hat{\theta}_t$ and $\lambda=\lambda_{t}$, then we have $\pi_{\hat{\theta}_{t+1}}$ of NPG update shown in \cref{eq:npg_hatthetat+1} and $\hat{\pi}_{t+1}$ of OPG update shown in \cref{eq:hatpi_t} are the same.

Define
\begin{align}
  V_\lambda^\pi(x)
  &:= 
  \mathbb{E}_{y\sim\pi(\cdot| x)}
  \Big[
    S_\lambda(x,y)
    - \beta\log\frac{\pi(y| x)}{\pi_{\mathrm{ref}}(y| x)}
  \Big],
  \label{eq:V-def}
  \\
  A_\lambda^\pi(x,y)
  &:= S_\lambda(x,y)
      - \beta\log\frac{\pi(y| x)}{\pi_{\mathrm{ref}}(y| x)}
      - V_\lambda^\pi(x).
  \label{eq:A-def}
\end{align}
We can rewrite $\nabla_\theta \mathcal{L}(\pi_\theta,\lambda)$ as
\begin{equation}
\begin{aligned}
    \nabla_\theta \mathcal{L}(\pi_\theta,\lambda)
  =&
  \mathbb{E}_{x\sim \mathcal{D},y\sim\pi_\theta(\cdot| x)}
  \left[
    \left(
      S_\lambda(x,y)
      - \beta\log\tfrac{\pi_\theta(y| x)}{\pi_{\mathrm{ref}}(y| x)}
    \right)
    \nabla_\theta\log\pi_\theta(y| x)
  \right]\\
  =&\mathbb{E}_{x\sim \mathcal{D},y\sim\pi_\theta(\cdot| x)}
  \left[
    A_\lambda^{\pi_\theta}(x,y)
    \nabla_\theta\log\pi_\theta(y| x)
  \right],  
\end{aligned}
\end{equation}
where the first equation is because $\mathcal{L}(\pi_\theta,\lambda)=\mathbb{E}_{x\sim \mathcal{D}}\left[\E_{y\sim\pi_{\theta}(\cdot|x)}S_{\lambda}(x,y)
- \beta \mathrm{KL}(\pi(\cdot|x) \| \pi_{\mathrm{ref}}(\cdot|x))\right]$ and $S_{\lambda}(x,y)
= \sum_{k \in \mathcal{S}} w_k R_k(x,y)
   + \sum_{j \in \mathcal{H}} \lambda_j R_j(x,y)$, and the second equation is by the definition of $A_\lambda^{\pi_\theta}(x,y)$ and $\mathbb{E}_{y\sim\pi_\theta(\cdot| x)}[V_\lambda^{\pi_\theta}(x)\nabla_\theta\log\pi_\theta(y| x)]=0$.
Then the partial derivation w.r.t. $\theta_{x,y}$ is
\[
\begin{aligned}
      \frac{\partial \mathcal{L}(\pi_\theta,\lambda)}{\partial\theta_{x,y}}
  &= 
  \mathbb{E}_{x'\sim \mathcal{D},y'\sim\pi_\theta(\cdot| x')}
  \left[
    A_\lambda^{\pi_\theta}(x',y')
    \frac{\partial}{\partial\theta_{x,y}}
    \log\pi_\theta(y'| x')
  \right] \\
  &\overset{(a)}{=}
  \mathbb{E}_{y'\sim\pi_\theta(\cdot| x)}
  \left[
    A_\lambda^{\pi_\theta}(x,y')
    \left(\mathbb{I}\{y'=y\} - \pi_\theta(y| x)\right)
  \right] \mathcal{D}(x) \\
  &=
  \mathcal{D}(x)\pi_\theta(y| x)
  \left(
    A_\lambda^{\pi_\theta}(x,y)
    -
    \mathbb{E}_{y'\sim\pi_\theta(\cdot| x)}[A_\lambda^{\pi_\theta}(x,y')]
  \right)\\
  &\overset{(b)}{=}
  \mathcal{D}(x)\pi_\theta(y| x)A_\lambda^{\pi_\theta}(x,y),
\end{aligned}
\]
where $(a)$ is because 
\[\frac{\partial}{\partial\theta_{x,y}}\log\pi_\theta(y'| x')
  =
  \mathbb{I}\{x'=x\}
  \left(
    \mathbb{I}\{y'=y\} - \pi_\theta(y| x)
  \right),\]
and $(b)$ is because $ \mathbb{E}_{y'\sim\pi_\theta(\cdot| x)}[A_\lambda^{\pi_\theta}(x,y')]=0$.

Let $\mathbf{e}_{x,y}\in \mathbb{R}^{|\mathcal{X}||\mathcal{Y}|}$ with only the position $\theta_{x,y}$ has element $1$ and all other elements are $0$, and $\pi_{\theta,x}\in \mathbb{R}^{|\mathcal{X}||\mathcal{Y}|}$ with only the positions $\theta_{x,y}$ has value $\pi_{\theta}(y|x)$ and all other elements are $0$. Then we rewrite $F(\theta)$ as
\[
\begin{aligned}
F(\theta)
=&\mathbb{E}_{x\sim \mathcal{D},y\sim\pi_\theta(\cdot| x)}
  \left[
    \nabla_\theta\log\pi_\theta(y| x)
    \nabla_\theta\log\pi_\theta(y| x)^\top
  \right]\\
=& \mathbb{E}_{x\sim \mathcal{D},y\sim\pi_\theta(\cdot| x)}
  \left[
    \left(\mathbf{e}_{x,y}-\pi_{\theta,x}\right)\left(\mathbf{e}_{x,y}-\pi_{\theta,x}\right)^\top
  \right]\\
=&\mathbb{E}_{x\sim \mathcal{D},y\sim\pi_\theta(\cdot| x)}
  \left[
   \mathbf{e}_{x,y}\mathbf{e}_{x,y}^\top-\pi_{\theta,x}\mathbf{e}_{x,y}^\top-\mathbf{e}_{x,y}\pi_{\theta,x}^\top+\pi_{\theta,x}\pi_{\theta,x}^\top
  \right]\\
=&\mathbb{E}_{x\sim \mathcal{D}}
  \left[
    \mathrm{diag}(\pi_{\theta,x})-\pi_{\theta,x}\pi_{\theta,x}^\top-\pi_{\theta,x}\pi_{\theta,x}^\top+\pi_{\theta,x}\pi_{\theta,x}^\top
  \right]\\ 
=&\mathbb{E}_{x\sim \mathcal{D}}
  \left[
    \mathrm{diag}(\pi_{\theta,x})-\pi_{\theta,x}\pi_{\theta,x}^\top
  \right],
\end{aligned}
\]
where the second equality uses the partial derivative of $\theta$ of $\log\pi_{\theta}(y|x)$.

We now characterize the natural-gradient direction
$\mathbf{w} = F(\theta)^\dagger\nabla_\theta \mathcal{L}(\pi_\theta,\lambda)$, where
$F(\theta)^\dagger$ is the Moore--Penrose pseudoinverse of $F(\theta)$. In other words, $F(\theta)\mathbf{w}=\nabla_\theta \mathcal{L}(\pi_\theta,\lambda)$. Let $\bar w_x := \pi_{\theta,x}^\top \mathbf{w} = \mathbb{E}_{y\sim\pi_x}[\mathbf{w}_{x,y}]$. Consider the $(x,y)$-th coordinate of the LHS, we have 
\[
\begin{aligned}
\left[F(\theta)\mathbf{w}\right]_{x,y}
=&\mathcal{D}(x)\left[\left(\mathrm{diag}(\pi_{\theta,x})-\pi_x\mathbf{\pi}_{\theta,x}^\top\right)\mathbf{w}\right]_y\\
=&\mathcal{D}(x)\left[\mathrm{diag}(\pi_{\theta,x})\mathbf{w}\right]_y
  - \left[\pi_{\theta,x}\pi_{\theta,x}^\top \mathbf{w}\right]_y\\
=& \mathcal{D}(x)\pi_{\theta,x}(y|x)\left(\mathbf{w}_{x,y} - \bar w_x\right),
\end{aligned}
\]
Comparing with $\frac{\partial \mathcal{L}(\pi_\theta,\lambda)}{\partial\theta_{x,y}}=\mathcal{D}(x)\pi_\theta(y| x)A_\lambda^{\pi_\theta}(x,y)
$, we have $\mathbf{w}_{x,y}=A_\lambda^{\pi_\theta}(x,y)+c(x)$.
Plugging $\mathbf{w}_{x,y}$ into \cref{eq:theta+}, $\theta_{+}=\theta+\frac{1}{\eta_{\theta}+\beta}\left(A_\lambda^{\pi_\theta}(x,y)+c(x)\right)$. The corresponding policy can be written as
\[
\pi_{\theta^+}(y| x)
  = \frac{\exp(\theta^+_{x,y})}
          {\sum_{y'\in\mathcal{Y}}\exp(\theta^+_{x,y'})}
   = \frac{\exp(\theta_{x,y})
            \exp(\frac{1}{\eta_{\theta}+\beta} A_\lambda^{\pi_{\theta}}(x,y))}
          {\sum_{y'\in\mathcal{Y}}
            \exp(\theta_{x,y'})
            \exp(\frac{1}{\eta_{\theta}+\beta} A_\lambda^{\pi_{\theta}}(x,y'))}.
        = \frac{\pi_{\theta}(y|x)
            \exp(\frac{1}{\eta_{\theta}+\beta} A_\lambda^{\pi_{\theta}}(x,y))}
          {\sum_{y'\in\mathcal{Y}}
            \pi_{\theta}(y'|x)
            \exp(\frac{1}{\eta_{\theta}+\beta} A_\lambda^{\pi_{\theta}}(x,y'))}.
\]
That is,
\[
\pi_{\theta^+}(y| x)\propto \pi_{\theta}(y|x) \exp\left(\frac{1}{\eta_{\theta}+\beta} A_\lambda^{\pi_{\theta}}(x,y)\right).
\]

As shown in \cref{eq:pi+}
\[
\pi_+
= \arg\max_{\pi}
\mathbb{E}_{x\sim \mathcal{D}}\left[\mathbb{E}_{y\sim\pi(\cdot|x)}\left[S_{\lambda}(x,y)\right]
-\beta\mathrm{KL}\left(\pi(\cdot|x)\Vert\pi_{\mathrm{ref}}(\cdot|x)\right)
-\eta_{\theta}\mathrm{KL}\left(\pi(\cdot|x)\Vert\pi_{\theta}(\cdot|x)\right)\right].
\]
Solve the maximization problem over the simplex $\Delta(\mathcal{Y})$ yields the
softmax solution
\[
\begin{aligned}
\pi^+(y| x)
\propto&
\exp\left(
    \tfrac{1}{\eta_\theta+\beta}\left(S_\lambda(x,y)
  + \beta\log\pi_{\mathrm{ref}}(y| x)
  + \eta_\theta\log\pi_{\theta}(y| x)\right)
  \right)\\
 \propto&\pi_{\theta}(y| x)\exp\left(
    \tfrac{1}{\eta_\theta+\beta}\left(S_\lambda(x,y)
  - \beta\log\frac{\pi_{\theta}(y|x)}{\pi_{\mathrm{ref}}(y| x)}\right)
  \right)\\
  \propto&\pi_{\theta}(y| x)\exp\left(\frac{1}{\eta_{\theta}+\beta} A_\lambda^{\pi_{\theta}}(x,y)\right),
\end{aligned}
\]
where the last equality is by the definition of $A_{\lambda}^{\pi_{\theta}}(x,y)$.

Comparing $\pi_{\theta_{+}}$ and $\pi_{+}$ concludes the proof.

\section{Proof of \cref{cor:linear_npg_convergence}}
\label{sec:npg_linear_npg_convergence}
Let $\PisubSet$ denote the class of parameterized policies that have full support on the considered action set, i.e., there exists $p_{\min}>0$ such that
$\pi_\theta(y|x)\ge p_{\min}$ for all feasible $(x,y)$.
Also, the parameter domain $\Theta\subset\mathbb{R}^d$ is closed and convex.
The Lagrangian problem is 
\[
\min_{\lambda \ge 0} \max_{\theta\in \Theta} \mathcal{L}(\pi_{\theta},\lambda),
\]
where $\mathcal{L}(\pi_{\theta},\lambda)= \E_{x\sim \mathcal{D}}\left[V^{\pi_{\theta}}_{S_{\lambda}}(x)
- \beta \KL(\pi_{\theta}(\cdot|x) \| \pi_{\tref}(\cdot|x))\right]$.
Under Slater's condition in the parameterized policy space, as shown in  \cref{assump:feasibility_parameterized}, strong duality holds and hence an optimal saddle point $(\pi_{\theta^\star},\lambda^\star)$ exists in the parameterized policy space.

Without loss of generality, we denote $\pi_{\theta_t}$ by $\pi_t$ and $\pi_{\hat{\theta}_t}$ by $\hat{\pi}_t$. Throughout this section, we further denote the optimal policy $\pi{\theta^\star}$ by $\pi^\star$.

Since $\pi^{\star}=\argmax_{\pi} L(\pi,\boldsymbol{\lambda}^{\star})$, we have $L(\pi^{\star},\boldsymbol{\lambda}^{\star})\geq L(\pi,\boldsymbol{\lambda}^{\star})$ for any $\pi\in \PiSet$. Similarly, since $\boldsymbol{\lambda}^{\star}=\argmin_{\lambda} L(\pi^{\star},\lambda)$, we have $L(\pi^{\star},\lambda)\geq L(\pi^{\star},\boldsymbol{\lambda}^{\star})$ for any $\lambda\geq 0$. Combining these two inequalities together, for any $\pi\in\PiSet$ and $\lambda\geq 0$, we have
\begin{equation}
\label{eq:npg_L(pi*)-L(lambda*)>0}
L(\pi^{\star},\lambda) - L(\pi,\boldsymbol{\lambda}^{\star}) 
= \underbrace{L(\pi^{\star},\lambda) - L(\pi^{\star},\boldsymbol{\lambda}^{\star})}_{\geq 0} + \underbrace{L(\pi^{\star},\boldsymbol{\lambda}^{\star}) - L(\pi,\boldsymbol{\lambda}^{\star})}_{\geq 0}
\geq 0    
\end{equation}

Let $\pi = \pi_t$ and $\lambda =\boldsymbol{\lambda}_t$ and substituting the definition of $L(\pi,\lambda)$ into the LHS of the above inequality, we have
\begin{equation}
  \label{eq:npg_AB-split}
\begin{aligned}
L(\pi^{\star},\boldsymbol{\lambda}_t) - L(\pi_t,\boldsymbol{\lambda}^{\star})
=&\E_{x\sim \mathcal{D}}\left[V^{\pi^{\star}}_{S_{\boldsymbol{\lambda}_t}}(x)
- \beta \KL(\pi^{\star}(\cdot|x) \| \pi_{\tref}(\cdot|x))\right]
-\E_{x\sim \mathcal{D}}\left[V^{\pi}_{S_{\boldsymbol{\lambda}^{\star}}}(x)
- \beta \KL(\pi(\cdot|x) \| \pi_{\tref}(\cdot|x))\right]\\
=&\underbrace{\E_{x\sim \mathcal{D}}\left[V^{\pi^{\star}}_{S_{\boldsymbol{\lambda}_t}}(x)
- \beta \KL(\pi^{\star}(\cdot|x) \| \pi_{\tref}(\cdot|x))\right]
-\E_{x\sim \mathcal{D}}\left[V^{\pi_t}_{S_{\boldsymbol{\lambda}_t}}(x)
- \beta \KL(\pi_t(\cdot|x) \| \pi_{\tref}(\cdot|x))\right]}_{\text{A}}\\
&+
\underbrace{\E_{x\sim \mathcal{D}}\left[V^{\pi_t}_{S_{\boldsymbol{\lambda}_t}}(x)
- \beta \KL(\pi_t(\cdot|x) \| \pi_{\tref}(\cdot|x))\right]
-\E_{x\sim \mathcal{D}}\left[V^{\pi_t}_{S_{\boldsymbol{\lambda}^{\star}}}(x)
- \beta \KL(\pi_t(\cdot|x) \| \pi_{\tref}(\cdot|x))\right])}_{\text{B}}
\end{aligned}  
\end{equation}

\subsection{Upper bound of term $\mathrm{A}$}
\label{sec:npg_A_ub}
We can rewrite term $\mathrm{A}$ as:
\begin{equation}
\label{eq:npg_A-basic}
\begin{aligned}
\mathrm{A}
=& \E_{x\sim \mathcal{D}}\left[V^{\pi^{\star}}_{S_{\boldsymbol{\lambda}_t}}(x)
- \beta \KL(\pi^{\star}(\cdot|x) \| \pi_{\tref}(\cdot|x))\right]
-\E_{x\sim \mathcal{D}}\left[V^{\pi_t}_{S_{\boldsymbol{\lambda}_t}}(x)
- \beta \KL(\pi_t(\cdot|x) \| \pi_{\tref}(\cdot|x))\right]\\
=& \E_{x\sim \mathcal{D}}\left[\left(V^{\pi^{\star}}_{S_{\boldsymbol{\lambda}_t}}(x)-V^{\pi_t}_{S_{\boldsymbol{\lambda}_t}}(x)\right)
- \beta \left(\KL(\pi^{\star}(\cdot|x) \| \pi_{\tref}(\cdot|x))
-\KL(\pi_t(\cdot|x) \| \pi_{\tref}(\cdot|x))\right)\right]\\
\overset{(a)}{=}& \E_{x\sim \mathcal{D}}\left[\langle \pi^{\star}(\cdot|x)-\pi_t(\cdot|x), S_{\boldsymbol{\lambda}_t}(x,\cdot)\rangle
-\beta\left(\KL(\pi^{\star}(\cdot|x))\Vert\pi_{\tref}(\cdot|x)))-\KL(\pi_t(\cdot|x))\Vert\pi_{\tref}(\cdot|x)))\right)\right]\\
\overset{(b)}{=}&\E_{x\sim \mathcal{D}}\left[\langle \pi^{\star}(\cdot|x)-\hat{\pi}_{t+1}(\cdot|x), S_{\boldsymbol{\lambda}_t}(x,\cdot)\rangle -\beta\left(\KL(\pi^{\star}(\cdot|x)\Vert\pi_{\tref}(\cdot|x))-\KL(\hat{\pi}_{t+1}(\cdot|x)\Vert\pi_{\tref}(\cdot|x))\right)\right.\\
&\qquad\quad+ \langle \hat{\pi}_{t+1}(\cdot|x)-\pi_t(\cdot|x), S_{\boldsymbol{\lambda}_{t-1}}(x,\cdot)\rangle -\beta\left(\KL(\hat{\pi}_{t+1}(\cdot|x)\Vert\pi_{\tref}(\cdot|x))-\KL(\pi_t(\cdot|x)\Vert\pi_{\tref}(\cdot|x))\right)\\
&\qquad\quad\left.+ \langle \hat{\pi}_{t+1}(\cdot|x)-\pi_t(\cdot|x), S_{\boldsymbol{\lambda}_t}(x,\cdot)-S_{\boldsymbol{\lambda}_{t-1}}(x,\cdot)\rangle\right]\\
\end{aligned}
\end{equation}
where $(a)$ is because the action space is discrete and $V^{\pi}_{S_{\boldsymbol{\lambda}_t}}(x)=\sum_{y}\pi(y|x){S_{\boldsymbol{\lambda}_t}}(x,y)=\langle \pi(\cdot|x), S_{\boldsymbol{\lambda}_t}(x,\cdot)\rangle$ for any $\pi\in\PiSet$, and $(b)$ is because adding and subtracting the same term keeps the equality.

As we consider the NPG update in the linear parameterized space, where $\pi_t$ and $\hat{\pi}_{t+1}$ updates follow \cref{eq:npg_thetat} and \cref{eq:npg_hatthetat+1}. Using \cref{cor:npg_three_point_identity} and letting $\eta=\eta_{\theta}$, $g=S_{\boldsymbol{\lambda}_t}(x,\cdot)$, $\pi_{\old} = \hat{\pi}_t(\cdot|x)$, $\pi_{\new}=\hat{\pi}_{t+1}(\cdot|x)$, and $\pi'=\pi^{\star}(\cdot|x)$, we have
\[
\begin{aligned}
&\E_{x\sim \mathcal{D}}\left[\langle S_{\boldsymbol{\lambda}_t}(x,\cdot), \hat{\pi}_{t+1}(\cdot|x)-\pi^{\star}(\cdot|x) \rangle
-\beta\left(\KL(\hat{\pi}_{t+1}(\cdot|x)\|\pi_{\tref}(\cdot|x))-\KL(\pi^{\star}(\cdot|x)\|\pi_{\tref}(\cdot|x))\right)\right]\\
&\geq\E_{x\sim \mathcal{D}}\left[\eta_{\theta}\left(-\KL(\pi^{\star}(\cdot|x)\|\hat{\pi}_t(\cdot|x))+\KL(\pi^{\star}(\cdot|x)\|\hat{\pi}_{t+1}(\cdot|x))+\KL(\hat{\pi}_{t+1}(\cdot|x)\|\hat{\pi}_t(\cdot|x))\right)
+\beta\KL(\pi^{\star}(\cdot|x)\|\hat{\pi}_{t+1}(\cdot|x))\right]\\
&\quad{-\gap(\varepsilon_{\mathrm{approx}},p_{\min})}. 
\end{aligned}
\]
Putting a negative sign on both sides, we have
\begin{equation}
\label{eq:npg_A_1}
    \begin{aligned}
&\E_{x\sim \mathcal{D}}\left[\langle \pi^{\star}(\cdot|x) -\hat{\pi}_{t+1}(\cdot|x), S_{\boldsymbol{\lambda}_t}(x,\cdot) \rangle
-\beta\left(\KL(\pi^{\star}(\cdot|x)\|\pi_{\tref}(\cdot|x))-\KL(\hat{\pi}_{t+1}(\cdot|x)\|\pi_{\tref}(\cdot|x))\right)\right]\\
&\leq\E_{x\sim \mathcal{D}}\left[\eta_{\theta}\left(\KL(\pi^{\star}(\cdot|x)\|\hat{\pi}_t(\cdot|x))
-\KL(\pi^{\star}(\cdot|x)\|\hat{\pi}_{t+1}(\cdot|x))
-\KL(\hat{\pi}_{t+1}(\cdot|x)\|\hat{\pi}_t(\cdot|x))\right)
-\beta\KL(\pi^{\star}(\cdot|x)\|\hat{\pi}_{t+1}(\cdot|x))\right]\\
&\quad {+\gap(\varepsilon_{\mathrm{approx}},p_{\min})}\\
&=\E_{x\sim \mathcal{D}}\left[\eta_{\theta}\KL(\pi^{\star}(\cdot|x)\|\hat{\pi}_t(\cdot|x))
-(\eta_{\theta}+\beta)\KL(\pi^{\star}(\cdot|x)\|\hat{\pi}_{t+1}(\cdot|x))
-\eta_{\theta}\KL(\hat{\pi}_{t+1}(\cdot|x)\|\hat{\pi}_t(\cdot|x))\right]{+\gap(\varepsilon_{\mathrm{approx}},p_{\min})}
\end{aligned}
\end{equation}

Let $\eta=\eta_{\theta}$, $g=S_{\boldsymbol{\lambda}_{t-1}}(x,\cdot)$, $\pi_{\old} = \hat{\pi}_t(\cdot|x)$, $\pi_{\new}=\pi_{t}(\cdot|x)$, and $\pi'=\hat{\pi}_{t+1}(\cdot|x)$ in \cref{cor:npg_three_point_identity}, we have
\begin{equation}
\label{eq:npg_A_2}
\begin{aligned}
&\E_{x\sim \mathcal{D}}\left[\langle \hat{\pi}_{t+1}(\cdot|x) -\pi_{t}(\cdot|x), S_{\boldsymbol{\lambda}_{t-1}}(x,\cdot) \rangle
-\beta\left(\KL(\hat{\pi}_{t+1}(\cdot|x)\|\pi_{\tref}(\cdot|x))-\KL(\pi_{t}(\cdot|x)\|\pi_{\tref}(\cdot|x))\right)\right]\\
&\leq\E_{x\sim \mathcal{D}}\left[\eta_{\theta}\left(\KL(\hat{\pi}_{t+1}(\cdot|x)\|\hat{\pi}_t(\cdot|x))
-\KL(\hat{\pi}_{t+1}(\cdot|x)\|\pi_{t}(\cdot|x))
-\KL(\pi_{t}(\cdot|x)\|\hat{\pi}_t(\cdot|x))\right)
-\beta\KL(\hat{\pi}_{t+1}(\cdot|x)\|\pi_{t}(\cdot|x))\right]\\
&\quad{+\gap(\varepsilon_{\mathrm{approx}},p_{\min})
}\\
&=\E_{x\sim \mathcal{D}}\left[\eta_{\theta}\KL(\hat{\pi}_{t+1}(\cdot|x)\|\hat{\pi}_t(\cdot|x))
-(\eta_{\theta}+\beta)\KL(\hat{\pi}_{t+1}(\cdot|x)\|\pi_{t}(\cdot|x))
-\eta_{\theta}\KL(\pi_{t}(\cdot|x)\|\hat{\pi}_t(\cdot|x))\right]\\
&\quad{+\gap(\varepsilon_{\mathrm{approx}},p_{\min})}.
\end{aligned}
\end{equation}
Let $C>0$ be a constant. For the last term in the RHS of \cref{eq:npg_A-basic}, we derive the upper bound as 
\begin{equation}
\label{eq:npg_A_3}
\begin{aligned}
&\langle\hat{\pi}_{t+1}(\cdot|x)-\pi_t(\cdot|x), S_{\boldsymbol{\lambda}_t}(x,\cdot)-S_{\boldsymbol{\lambda}_{t-1}}(x,\cdot)\rangle \\
 &=\left\langle\hat{\pi}_{t+1}(\cdot|x)-\pi_t(\cdot|x), \sum_{j\in\mathcal{H}}(\lambda_{t,j}-\lambda_{t-1,j})R_j(x,\cdot)\right\rangle \\  
 &\overset{(a)}{\leq} \|\hat{\pi}_{t+1}(\cdot|x)-\pi_t(\cdot|x)\|_1 \left\|\sum_{j\in\mathcal{H}}(\lambda_{t,j}-\lambda_{t-1,j})R_j(x,\cdot)\right\|_{\infty}\\
 &\overset{(b)}{\leq}  \|\hat{\pi}_{t+1}(\cdot|x)-\pi_t(\cdot|x)\|_1 \|\boldsymbol{\lambda}_{t}-\boldsymbol{\lambda}_{t-1}\|_{1}R_{\max} \\
 &\leq \sqrt{2\KL(\hat{\pi}_{t+1}(\cdot|x)\| \pi_t(\cdot|x))}\|\boldsymbol{\lambda}_{t}-\boldsymbol{\lambda}_{t-1}\|_{1}R_{\max}  \quad\text{(By Pinsker's inequality in \cref{lem:pinsker's ineq})}\\
&\leq C\KL(\hat{\pi}_{t+1}(\cdot|x)\| \pi_t(\cdot|x))+\frac{R_{\max}^2}{2C}\|\boldsymbol{\lambda}_{t}-\boldsymbol{\lambda}_{t-1}\|_{1}^2
\quad\text{(By AM-GM inequality $\frac{x^2}{2C}+\frac{y^2C}{2}\geq xy$ with $C>0$)}\\
 &\leq C\KL(\hat{\pi}_{t+1}(\cdot|x)\| \pi_t(\cdot|x))+\frac{|\mathcal{H}| R_{\max}^2}{2C}\|\boldsymbol{\lambda}_{t}-\boldsymbol{\lambda}_{t-1}\|_{2}^2 \quad\text{(By $\|\mathbf{x}\|_1^2\leq d\|\mathbf{x}\|_2^2$, $\forall \mathbf{x}\in\mathbb{R}^d$)}\\
 &\overset{(c)}{\leq} C\KL(\hat{\pi}_{t+1}(\cdot|x)\| \pi_t(\cdot|x))+\frac{|\mathcal{H}| R_{\max}^2}{C}\left(\|\boldsymbol{\lambda}_{t}-\hat{\boldsymbol{\lambda}}_{t}\|_{2}^2+\|\hat{\boldsymbol{\lambda}}_{t}-\boldsymbol{\lambda}_{t-1}\|_{2}^2\right),
\end{aligned}
\end{equation}
where $(a)$ is by Hölder's inequality shown in \cref{lem:holder's inequality} and letting $f(\cdot)=\hat{\pi}_{t+1}(\cdot|x)-\pi_t(\cdot|x)$, $g(\cdot)=\sum_{j\in\mathcal{H}}(\lambda_{t,j}-\lambda_{t-1,j})R_j(x,\cdot)$, $p=1$, and $q=\infty$.
$(b)$ is by
\[
\begin{aligned}
 \left\|\sum_{j\in\mathcal{H}}(\lambda_{t,j}-\lambda_{t-1,j})R_j(x,\cdot)\right\|_{\infty}
\leq& \sum_{j\in\mathcal{H}}\left\|(\lambda_{t,j}-\lambda_{t-1,j})R_j(x,\cdot)\right\|_{\infty} \quad\text{(by triangle inequality)}\\
=&  \sum_{j\in\mathcal{H}}|\lambda_{t,j}-\lambda_{t-1,j}|\left\|R_j(x,\cdot)\right\|_{\infty} \\
\leq& \left(\sum_{j\in\mathcal{H}}|\lambda_{t,j}-\lambda_{t-1,j}|\right)\max_j\left\|R_j(x,\cdot)\right\|_{\infty} \\
=&\|\boldsymbol{\lambda}_{t}-\boldsymbol{\lambda}_{t-1}\|_1 \max_j\left\|R_j(x,\cdot)\right\|_{\infty} \quad\text{(by definition of 1-norm)}\\
\leq& \|\boldsymbol{\lambda}_{t}-\boldsymbol{\lambda}_{t-1}\|_1 R_{\max} \\
&\text{(by \cref{assump:bounded} that $R_j(x,y)\leq R_{\max}$ for any $j\in \mathcal{S}\cup\mathcal{H}$ and $(x,y)$ pair).}
\end{aligned}
\]
$(c)$ is because
\[
\begin{aligned}
\|\boldsymbol{\lambda}_{t}-\boldsymbol{\lambda}_{t-1}\|_{2}^2
=&\|\boldsymbol{\lambda}_{t}-\hat{\boldsymbol{\lambda}}_{t}+\hat{\boldsymbol{\lambda}}_{t}-\boldsymbol{\lambda}_{t-1}\|_{2}^2\\
=&\|\boldsymbol{\lambda}_{t}-\hat{\boldsymbol{\lambda}}_{t}\|_{2}^2 + \|\hat{\boldsymbol{\lambda}}_{t}-\boldsymbol{\lambda}_{t-1}\|_{2}^2 + 2\langle\boldsymbol{\lambda}_{t}-\hat{\boldsymbol{\lambda}}_{t}, \hat{\boldsymbol{\lambda}}_{t}-\boldsymbol{\lambda}_{t-1}\rangle\\
\leq& 2\|\boldsymbol{\lambda}_{t}-\hat{\boldsymbol{\lambda}}_{t}\|_{2}^2 + 2\|\hat{\boldsymbol{\lambda}}_{t}-\boldsymbol{\lambda}_{t-1}\|_{2}^2 \\
&\text{(by Young's inequality with $p=q=2$, i.e., $\langle \mathbf{x},\mathbf{y}\rangle\leq \frac{1}{2}(\|\mathbf{x}\|_2^2+\|\mathbf{y}\|_2^2)$)}
\end{aligned}
\]
Substituting \cref{eq:npg_A_1}, \cref{eq:npg_A_2}, and \cref{eq:npg_A_3} into the RHS of \cref{eq:npg_A-basic}, we have
\begin{equation}
\label{eq:npg_A_basic2}
\begin{aligned}
\mathrm{A}
\leq& \E_{x\sim\mathcal{D}}\left[\eta_{\theta}\KL(\pi^{\star}(\cdot|x)\|\hat{\pi}_t(\cdot|x))
-(\eta_{\theta}+\beta)\KL(\pi^{\star}(\cdot|x)\|\hat{\pi}_{t+1}(\cdot|x))
-\eta_{\theta}\KL(\hat{\pi}_{t+1}(\cdot|x)\|\hat{\pi}_t(\cdot|x))\right.\\
&\qquad\quad+ \eta_{\theta}\KL(\hat{\pi}_{t+1}(\cdot|x)\|\hat{\pi}_t(\cdot|x))
-(\eta_{\theta}+\beta)\KL(\hat{\pi}_{t+1}(\cdot|x)\|\pi_{t}(\cdot|x))
-\eta_{\theta}\KL(\pi_{t}(\cdot|x)\|\hat{\pi}_t(\cdot|x))\\
&\qquad\quad\left.+ C\KL(\hat{\pi}_{t+1}(\cdot|x)\| \pi_t(\cdot|x))+\frac{|\mathcal{H}| R_{\max}^2}{C}\left(\|\boldsymbol{\lambda}_{t}-\hat{\boldsymbol{\lambda}}_{t}\|_{2}^2+\|\hat{\boldsymbol{\lambda}}_{t}-\boldsymbol{\lambda}_{t-1}\|_{2}^2\right)\right]
{+2\gap(\varepsilon_{\mathrm{approx}},p_{\min})
}\\
=& \E_{x\sim\mathcal{D}}\left[\eta_{\theta}\KL(\pi^{\star}(\cdot|x)\|\hat{\pi}_t(\cdot|x))
-(\eta_{\theta}+\beta)\KL(\pi^{\star}(\cdot|x)\|\hat{\pi}_{t+1}(\cdot|x))
\right.\\
&\qquad\quad-(\eta_{\theta}+\beta)\KL(\hat{\pi}_{t+1}(\cdot|x)\|\pi_{t}(\cdot|x))
-\eta_{\theta}\KL(\pi_{t}(\cdot|x)\|\hat{\pi}_t(\cdot|x))\\
&\qquad\quad\left.+ C\KL(\hat{\pi}_{t+1}(\cdot|x)\| \pi_t(\cdot|x))\right]+\frac{|\mathcal{H}| R_{\max}^2}{C}\left(\|\boldsymbol{\lambda}_{t}-\hat{\boldsymbol{\lambda}}_{t}\|_{2}^2+\|\hat{\boldsymbol{\lambda}}_{t}-\boldsymbol{\lambda}_{t-1}\|_{2}^2\right)
{+2\gap(\varepsilon_{\mathrm{approx}},p_{\min})
}
\end{aligned}
\end{equation}

\subsection{Upper bound of term $\mathrm{B}$}
\label{sec:npg_B_ub}
Similarly, we rewrite the term $\mathrm{B}$ as
\begin{equation}
    \begin{aligned}
\mathrm{B}
=&\E_{x\sim \mathcal{D}}\left[V^{\pi_t}_{S_{\boldsymbol{\lambda}_t}}(x)
- \beta \KL(\pi_t(\cdot|x) \| \pi_{\tref}(\cdot|x))\right]
-\E_{x\sim \mathcal{D}}\left[V^{\pi_t}_{S_{\boldsymbol{\lambda}^{\star}}}(x)
- \beta \KL(\pi_t(\cdot|x) \| \pi_{\tref}(\cdot|x))\right])\\
=& \E_{x\sim\mathcal{D}}\left[ V_{S_{\boldsymbol{\lambda}_t}}^{\pi_t}(x)
- V_{S_{\boldsymbol{\lambda}^{\star}}}^{\pi_{t}}(x)\right]\\
=&\E_{x\sim\mathcal{D}}\left[
\left(V_{S_{\boldsymbol{\lambda}_t}}^{\pi_t}(x) - V_{S_{\hat{\boldsymbol{\lambda}}_{t+1}}}^{\pi_{t}}(x)\right)
+\left(V_{S_{\hat{\boldsymbol{\lambda}}_{t+1}}}^{\pi_t}(x) - V_{S_{\boldsymbol{\lambda}^{\star}}}^{\pi_{t}}(x)\right)\right.\\
&\qquad\quad\left.-\left(V_{S_{\boldsymbol{\lambda}_t}}^{\pi_{t-1}}(x) - V_{S_{\hat{\boldsymbol{\lambda}}_{t+1}}}^{\pi_{t-1}}(x)\right)
+\left(V_{S_{\boldsymbol{\lambda}_t}}^{\pi_{t-1}}(x) - V_{S_{\hat{\boldsymbol{\lambda}}_{t+1}}}^{\pi_{t-1}}(x)\right)
\right]
    \end{aligned}
\end{equation}
Define $R_{\mathcal{H}}=\left[R_{j_1},R_{j_2},\cdots,R_{j_{|\mathcal{H}|}}\right]\in \mathbb{R}^{|\mathcal{H}|}$, where $j_1<j_2<\cdots<j_{|\mathcal{H}|}$, and $j_k\in \mathcal{H}$ for any integer $1\leq k \leq |\mathcal{H}|$. Define $V_{R_{\mathcal{H}}}^{\pi}(x)=\E_{y\sim \pi(\cdot|x)}R_{\mathcal{H}}(x,y)\in\mathbb{R}^{|\mathcal{H}|}$. By the definition of $V_{S_{\lambda}}^{\pi}(x)$, we have
\[
\begin{aligned}
    V_{S_{\lambda}}^{\pi}(x) 
    =&\E_{y\sim \pi(\cdot|x)}S_{\lambda}(x,y) \quad\text{(by definition of $V_{S_{\lambda}}^{\pi}(x) $)}\\
    =&\E_{y\sim \pi(\cdot|x)}\left[\sum_{j\in\mathcal{S}}w_jR_j(x,y)+\sum_{j\in\mathcal{H}}\lambda_jR_j(x,y)\right] \quad\text{(by definition of $S_{\lambda}$)}\\
    =&\E_{y\sim \pi(\cdot|x)}\left[\sum_{j\in\mathcal{S}}w_jR_j(x,y))\right]+\sum_{j\in\mathcal{H}}\lambda_j\E_{y\sim \pi(\cdot|x)}\left[R_j(x,y)\right]\\
    =&\E_{y\sim \pi(\cdot|x)}\left[\sum_{j\in\mathcal{S}}w_jR_j(x,y))\right]+\lambda^TV_{R_{\mathcal{H}}}^{\pi}(x) \quad\text{(by definition of $V_{R_{\mathcal{H}}}^{\pi}$)}
\end{aligned}
\]
Plugging the above expression of $V_{S_{\lambda}}^{\pi}(x)$ into term $\mathrm{B}$, we have
\begin{equation}
\label{eq:npg_B_basic_rewrite1}
\begin{aligned}
 \mathrm{B}
 =&\E_{x\sim\mathcal{D}}\left[
 \left(\boldsymbol{\lambda}_t-\hat{\boldsymbol{\lambda}}_{t+1}\right)^TV_{R_{\mathcal{H}}}^{\pi_{t}}(x)
 +\left(\hat{\boldsymbol{\lambda}}_{t+1}-\boldsymbol{\lambda}^{\star}\right)^TV_{R_{\mathcal{H}}}^{\pi_{t}}(x)
 -\left(\boldsymbol{\lambda}_t-\hat{\boldsymbol{\lambda}}_{t+1}\right)^TV_{R_{\mathcal{H}}}^{\pi_{t-1}}(x)
 +\left(\boldsymbol{\lambda}_t-\hat{\boldsymbol{\lambda}}_{t+1}\right)^TV_{R_{\mathcal{H}}}^{\pi_{t-1}}(x)
\right]  \\
=&\E_{x\sim\mathcal{D}}\left[
\left(\hat{\boldsymbol{\lambda}}_{t+1}-\boldsymbol{\lambda}^{\star}\right)^TV_{R_{\mathcal{H}}}^{\pi_{t}}(x)
+\left(\boldsymbol{\lambda}_t-\hat{\boldsymbol{\lambda}}_{t+1}\right)^TV_{R_{\mathcal{H}}}^{\pi_{t-1}}(x)
+\left(\boldsymbol{\lambda}_t-\hat{\boldsymbol{\lambda}}_{t+1}\right)^T\left(V_{R_{\mathcal{H}}}^{\pi_{t}}(x)-V_{R_{\mathcal{H}}}^{\pi_{t-1}}(x)\right)\right] \\
=&
\left(\hat{\boldsymbol{\lambda}}_{t+1}-\boldsymbol{\lambda}^{\star}\right)^T\E_{x\sim\mathcal{D}}\left[V_{R_{\mathcal{H}}}^{\pi_{t}}(x)\right]
+\left(\boldsymbol{\lambda}_t-\hat{\boldsymbol{\lambda}}_{t+1}\right)^T\E_{x\sim\mathcal{D}}\left[V_{R_{\mathcal{H}}}^{\pi_{t-1}}(x)\right]\\
&+\left(\boldsymbol{\lambda}_t-\hat{\boldsymbol{\lambda}}_{t+1}\right)^T\left(\E_{x\sim\mathcal{D}}\left[V_{R_{\mathcal{H}}}^{\pi_{t}}(x)\right]-\E_{x\sim\mathcal{D}}\left[V_{R_{\mathcal{H}}}^{\pi_{t-1}}(x)\right]\right) 
\end{aligned}
\end{equation}
Recall \cref{eq:npg_hatlambdat+1} gives the $\hat{\lambda}_{t+1,j}$ in the optimistic gradient descent in the parameterized space, and we rewrite the update as follows 
\[\hat{\lambda}_{t+1,j}
= \arg\min_{\lambda\geq 0}
\lambda\mathbb{E}_{x\sim \mathcal{D}, y\sim \pi_t(\cdot|x)}\big[R_j(x,y)\big]
+\eta_{\lambda}\big(\lambda-\hat{\lambda}_{t,j}\big)^2. \]
Without loss of generality, for a vector $\mathbf{a}$, we write $\mathbf{a} \ge 0$ to indicate that all entries of $\mathbf{a}$ are nonnegative. 
Rewrite the above update in the vectorized form as follows
\[
\begin{aligned}
 \hat{\boldsymbol{\lambda}}_{t+1}
=& \arg\min_{\boldsymbol{\lambda}\geq 0}
\boldsymbol{\lambda}^T\mathbb{E}_{x\sim \mathcal{D}}\big[V_{R_\mathcal{H}}^{\pi_t}(x)\big]
+\eta_{\lambda}\|\boldsymbol{\lambda}-\hat{\boldsymbol{\lambda}}_{t}\|_2^2\\   
=& \arg\max_{\boldsymbol{\lambda}\geq 0}
-\boldsymbol{\lambda}^T\mathbb{E}_{x\sim \mathcal{D}}\big[V_{R_\mathcal{H}}^{\pi_t}(x)\big]
-\eta_{\lambda}\|\boldsymbol{\lambda}-\hat{\boldsymbol{\lambda}}_{t}\|_2^2. 
\end{aligned}
\]
 Let $g=-\mathbb{E}_{x\sim \mathcal{D}}\big[V_{R_\mathcal{H}}^{\pi_t}(x)\big]$, $\eta=\eta_{\lambda}$, $h(\mathbf{x})=\|\mathbf{x}\|_2^2$, $D_h(\mathbf{x},\mathbf{y})=\|\mathbf{x}-\mathbf{y}\|_2^2$, $\mathbf{x}'=\boldsymbol{\lambda}^{\star}$, $\mathbf{x}_{\new}=\hat{\boldsymbol{\lambda}}_{t+1}$, $\mathbf{x}_{\old}=\hat{\boldsymbol{\lambda}}_t$, and $\Omega=\mathbb{R}_{+}^{|\mathcal{H}|}$ in \cref{lem:three-point-gradient}, we have
\[
  \langle -\mathbb{E}_{x\sim \mathcal{D}}\big[V_{R_\mathcal{H}}^{\pi_t}(x)\big], \hat{\boldsymbol{\lambda}}_{t+1}-\boldsymbol{\lambda}^{\star}\rangle
\geq\eta_{\lambda}\left(-\|\boldsymbol{\lambda}^{\star}-\hat{\boldsymbol{\lambda}}_t\|_2^2
+\|\boldsymbol{\lambda}^{\star}-\hat{\boldsymbol{\lambda}}_{t+1}\|_2^2
+\|\hat{\boldsymbol{\lambda}}_{t+1}-\hat{\boldsymbol{\lambda}}_t\|_2^2\right).
\]
Putting negative sign on both sides, we have
\begin{equation}
    \label{eq:npg_B_1}
     \langle \hat{\boldsymbol{\lambda}}_{t+1}-\boldsymbol{\lambda}^{\star}, \mathbb{E}_{x\sim \mathcal{D}}\big[V_{R_\mathcal{H}}^{\pi_t}(x)\big] \rangle
\leq\eta_{\lambda}\left(\|\boldsymbol{\lambda}^{\star}-\hat{\boldsymbol{\lambda}}_t\|_2^2
-\|\boldsymbol{\lambda}^{\star}-\hat{\boldsymbol{\lambda}}_{t+1}\|_2^2
-\|\hat{\boldsymbol{\lambda}}_{t+1}-\hat{\boldsymbol{\lambda}}_t\|_2^2\right).
\end{equation}
Similarly, since \cref{eq:npg_lambdat} gives the optimistic update of $\lambda_{t,j}$ as follows,
\[
\lambda_{t,j}
= \arg\min_{\lambda\geq 0}
\lambda\mathbb{E}_{x\sim \mathcal{D}, y\sim \pi_{t-1}(\cdot|x)}\big[R_j(x,y)\big]
+\eta_{\lambda}\big(\lambda-\hat{\lambda}_{t,j}\big)^2,
\]
Applying \cref{lem:three-point-gradient} by setting $g=-\mathbb{E}_{x\sim \mathcal{D}}\big[V_{R_\mathcal{H}}^{\pi_{t-1}}(x)\big]$, $\eta=\eta_{\lambda}$, $h(\mathbf{x})=\|\mathbf{x}\|_2^2$, $D_h(\mathbf{x},\mathbf{y})=\|\mathbf{x}-\mathbf{y}\|_2^2$, $\mathbf{x}'=\hat{\boldsymbol{\lambda}}_{t+1}$, $\mathbf{x}_{\new}=\boldsymbol{\lambda}_{t}$, $\mathbf{x}_{\old}=\hat{\boldsymbol{\lambda}}_t$, and $\Omega=\mathbb{R}_{+}^{|\mathcal{H}|}$, we have
\begin{equation}
\label{eq:npg_B_2}
     \langle \boldsymbol{\lambda}_{t}-\hat{\boldsymbol{\lambda}}_{t+1}, \mathbb{E}_{x\sim \mathcal{D}}\big[V_{R_\mathcal{H}}^{\pi_{t-1}}(x)\big] \rangle
\leq\eta_{\lambda}\left(\|\hat{\boldsymbol{\lambda}}_{t+1}-\hat{\boldsymbol{\lambda}}_t\|_2^2
-\|\hat{\boldsymbol{\lambda}}_{t+1}-\boldsymbol{\lambda}_{t}\|_2^2
-\|\boldsymbol{\lambda}_{t}-\hat{\boldsymbol{\lambda}}_t\|_2^2\right).
\end{equation}
We upper bound the last term of $\mathrm{B}$ as
\[
\begin{aligned}
&\left(\boldsymbol{\lambda}_t-\hat{\boldsymbol{\lambda}}_{t+1}\right)^T\left(\E_{x\sim\mathcal{D}}\left[V_{R_{\mathcal{H}}}^{\pi_{t}}(x)\right]-\E_{x\sim\mathcal{D}}\left[V_{R_{\mathcal{H}}}^{\pi_{t-1}}(x)\right]\right) \\
&=\E_{x\sim\mathcal{D}}\left[\left(\boldsymbol{\lambda}_t-\hat{\boldsymbol{\lambda}}_{t+1}\right)^T\left(V_{R_{\mathcal{H}}}^{\pi_{t}}(x)-V_{R_{\mathcal{H}}}^{\pi_{t-1}}(x)\right)\right] \\
&=\E_{x\sim\mathcal{D}}\left[\left(\boldsymbol{\lambda}_t-\hat{\boldsymbol{\lambda}}_{t+1}\right)^T\left(\langle \pi_t(\cdot|x), R_{\mathcal{H}}(x,\cdot)\rangle-\langle \pi_{t-1}(\cdot|x), R_{\mathcal{H}}(x,\cdot)\rangle\right)\right] \\
&\qquad\qquad\qquad\qquad\qquad\qquad\qquad\qquad\qquad\qquad
\text{(by $V_{R_{\mathcal{H}}}^{\pi}(x)=\E_{y\sim \pi(\cdot|x)}\left[R_{\mathcal{H}}(x,y)\right]=\langle \pi(\cdot|x), R_{\mathcal{H}}(x,\cdot)\rangle$)}\\
&= \E_{x\sim\mathcal{D}}\left[\langle\pi_{t}(\cdot|x)-\pi_{t-1}(\cdot|x), \left(\boldsymbol{\lambda}_t-\hat{\boldsymbol{\lambda}}_{t+1}\right)^TR_{\mathcal{H}}(x,\cdot)\rangle\right] \\  
&= \E_{x\sim\mathcal{D}}\left[\left\langle\pi_{t}(\cdot|x)-\pi_{t-1}(\cdot|x), \sum_{j\in\mathcal{H}}(\lambda_{t,j}-\hat{\lambda}_{t+1,j})R_j(x,\cdot)\right\rangle\right]
\end{aligned}
\]
Define constants $C_1>0$ and $C_2>0$. Fixing $x$, we have
\[
\begin{aligned}
&\left\langle\pi_{t}(\cdot|x)-\pi_{t-1}(\cdot|x), \sum_{j\in\mathcal{H}}(\lambda_{t,j}-\hat{\lambda}_{t+1,j})R_j(x,\cdot)\right\rangle\\
&=\left\langle\pi_{t}(\cdot|x)-\hat{\pi}_{t}(\cdot|x), \sum_{j\in\mathcal{H}}(\lambda_{t,j}-\hat{\lambda}_{t+1,j})R_j(x,\cdot)\right\rangle
+\left\langle\hat{\pi}_{t}(\cdot|x)-\pi_{t-1}(\cdot|x), \sum_{j\in\mathcal{H}}(\lambda_{t,j}-\hat{\lambda}_{t+1,j})R_j(x,\cdot)\right\rangle\\
&\leq
C_1\KL(\pi_{t}(\cdot|x)\| \hat{\pi}_{t}(\cdot|x))
+\frac{|\mathcal{H}| R_{\max}^2}{2C_1}\|\boldsymbol{\lambda}_{t}-\hat{\boldsymbol{\lambda}}_{t+1}\|_{2}^2
+
C_2\KL(\hat{\pi}_{t}(\cdot|x)\| \pi_{t-1}(\cdot|x))
+\frac{|\mathcal{H}| R_{\max}^2}{2C_2}\|\boldsymbol{\lambda}_{t}-\hat{\boldsymbol{\lambda}}_{t+1}\|_{2}^2\\
&\qquad\qquad\qquad\qquad\qquad\qquad\qquad\qquad\qquad\qquad\qquad\qquad\qquad\qquad\qquad\quad
\text{(By derivations of \cref{eq:npg_A_3})}\\
&=C_1\KL(\pi_{t}(\cdot|x)\| \hat{\pi}_{t}(\cdot|x))
+C_2\KL(\hat{\pi}_{t}(\cdot|x)\| \pi_{t-1}(\cdot|x))
+|\mathcal{H}| R_{\max}^2\left(\frac{1}{2C_1}+\frac{1}{2C_2}\right)\|\boldsymbol{\lambda}_{t}-\hat{\boldsymbol{\lambda}}_{t+1}\|_{2}^2
\end{aligned}
\]
Substituting the above inequality into $\left(\boldsymbol{\lambda}_t-\hat{\boldsymbol{\lambda}}_{t+1}\right)^T\left(\E_{x\sim\mathcal{D}}\left[V_{R_{\mathcal{H}}}^{\pi_{t}}(x)\right]-\E_{x\sim\mathcal{D}}\left[V_{R_{\mathcal{H}}}^{\pi_{t-1}}(x)\right]\right)$, we have
\begin{equation}
\label{eq:npg_B_3}
\begin{aligned}
&\left(\boldsymbol{\lambda}_t-\hat{\boldsymbol{\lambda}}_{t+1}\right)^T\left(\E_{x\sim\mathcal{D}}\left[V_{R_{\mathcal{H}}}^{\pi_{t}}(x)\right]-\E_{x\sim\mathcal{D}}\left[V_{R_{\mathcal{H}}}^{\pi_{t-1}}(x)\right]\right) \\
&=\E_{x\sim\mathcal{D}}\left[C_1\KL(\pi_{t}(\cdot|x)\| \hat{\pi}_{t}(\cdot|x))
+C_2\KL(\hat{\pi}_{t}(\cdot|x)\| \pi_{t-1}(\cdot|x))\right]
+|\mathcal{H}| R_{\max}^2\left(\frac{1}{2C_1}+\frac{1}{2C_2}\right)\|\boldsymbol{\lambda}_{t}-\hat{\boldsymbol{\lambda}}_{t+1}\|_{2}^2
\end{aligned}
\end{equation}
Combining \cref{eq:npg_B_1}, \cref{eq:npg_B_2}, and \cref{eq:npg_B_3}, we can upper bound $\mathrm{B}$ as 
\begin{equation}
\label{eq:npg_B_basic2}
\begin{aligned}
    \mathrm{B}
    \leq& \eta_{\lambda}\left(\|\boldsymbol{\lambda}^{\star}-\hat{\boldsymbol{\lambda}}_t\|_2^2
-\|\boldsymbol{\lambda}^{\star}-\hat{\boldsymbol{\lambda}}_{t+1}\|_2^2
-\|\hat{\boldsymbol{\lambda}}_{t+1}-\hat{\boldsymbol{\lambda}}_t\|_2^2\right)
+\eta_{\lambda}\left(\|\hat{\boldsymbol{\lambda}}_{t+1}-\hat{\boldsymbol{\lambda}}_t\|_2^2
-\|\hat{\boldsymbol{\lambda}}_{t+1}-\boldsymbol{\lambda}_{t}\|_2^2
-\|\boldsymbol{\lambda}_{t}-\hat{\boldsymbol{\lambda}}_t\|_2^2\right)\\
&+\E_{x\sim\mathcal{D}}\left[C_1\KL(\pi_{t}(\cdot|x)\| \hat{\pi}_{t}(\cdot|x))
+C_2\KL(\hat{\pi}_{t}(\cdot|x)\| \pi_{t-1}(\cdot|x))\right]
+|\mathcal{H}| R_{\max}^2\left(\frac{1}{2C_1}+\frac{1}{2C_2}\right)\|\boldsymbol{\lambda}_{t}-\hat{\boldsymbol{\lambda}}_{t+1}\|_{2}^2\\
=&\eta_{\lambda}\left(\|\boldsymbol{\lambda}^{\star}-\hat{\boldsymbol{\lambda}}_t\|_2^2
-\|\boldsymbol{\lambda}^{\star}-\hat{\boldsymbol{\lambda}}_{t+1}\|_2^2
-\|\boldsymbol{\lambda}_{t}-\hat{\boldsymbol{\lambda}}_{t+1}\|_2^2
-\|\boldsymbol{\lambda}_{t}-\hat{\boldsymbol{\lambda}}_t\|_2^2\right)\\
&+\E_{x\sim\mathcal{D}}\left[C_1\KL(\pi_{t}(\cdot|x)\| \hat{\pi}_{t}(\cdot|x))
+C_2\KL(\hat{\pi}_{t}(\cdot|x)\| \pi_{t-1}(\cdot|x))\right]
+|\mathcal{H}| R_{\max}^2\left(\frac{1}{2C_1}+\frac{1}{2C_2}\right)\|\boldsymbol{\lambda}_{t}-\hat{\boldsymbol{\lambda}}_{t+1}\|_{2}^2
\end{aligned}
\end{equation}

\subsection{Combining $\mathrm{A}$ and $\mathrm{B}$}
\label{sec:npg_combine_ab}
Substituting \cref{eq:npg_A_basic2} and \cref{eq:npg_B_basic2} into the RHS of \cref{eq:npg_AB-split}, we get 
\[
\begin{aligned}
&L(\pi^{\star},\boldsymbol{\lambda}_t) - L(\pi_t,\boldsymbol{\lambda}^{\star})\\
&\leq\E_{x\sim\mathcal{D}}\left[\eta_{\theta}\KL(\pi^{\star}(\cdot|x)\|\hat{\pi}_t(\cdot|x))
-(\eta_{\theta}+\beta)\KL(\pi^{\star}(\cdot|x)\|\hat{\pi}_{t+1}(\cdot|x))
{\color{red}-(\eta_{\theta}+\beta)\KL(\hat{\pi}_{t+1}(\cdot|x)\|\pi_{t}(\cdot|x))}
\right.\\
&\qquad\qquad
\left.
{\color{blue}-\eta_{\theta}\KL(\pi_{t}(\cdot|x)\|\hat{\pi}_t(\cdot|x))}
{\color{red}+ C\KL(\hat{\pi}_{t+1}(\cdot|x)\|\pi_t(\cdot|x))}\right]
{+2\gap(\varepsilon_{\mathrm{approx}},p_{\min})}\\
&\quad+\frac{|\mathcal{H}| R_{\max}^2}{C}\left({\color{magenta}\|\boldsymbol{\lambda}_{t}-\hat{\boldsymbol{\lambda}}_{t}\|_{2}^2}+\|\hat{\boldsymbol{\lambda}}_{t}-\boldsymbol{\lambda}_{t-1}\|_{2}^2\right)
+\eta_{\lambda}\left(\|\boldsymbol{\lambda}^{\star}-\hat{\boldsymbol{\lambda}}_t\|_2^2
-\|\boldsymbol{\lambda}^{\star}-\hat{\boldsymbol{\lambda}}_{t+1}\|_2^2
{\color{cyan}-\|\boldsymbol{\lambda}_{t}-\hat{\boldsymbol{\lambda}}_{t+1}\|_2^2}
{-\color{magenta}\|\boldsymbol{\lambda}_{t}-\hat{\boldsymbol{\lambda}}_t\|_2^2}\right)\\
&\quad+\E_{x\sim\mathcal{D}}\left[{\color{blue}C_1\KL(\pi_{t}(\cdot|x)\| \hat{\pi}_{t}(\cdot|x))}
+C_2\KL(\hat{\pi}_{t}(\cdot|x)\| \pi_{t-1}(\cdot|x))\right]
+{\color{cyan}|\mathcal{H}| R_{\max}^2\left(\frac{1}{2C_1}+\frac{1}{2C_2}\right)\|\boldsymbol{\lambda}_{t}-\hat{\boldsymbol{\lambda}}_{t+1}\|_{2}^2}\\
&\leq\E_{x\sim\mathcal{D}}\left[\eta_{\theta}\KL(\pi^{\star}(\cdot|x)\|\hat{\pi}_t(\cdot|x))
-(\eta_{\theta}+\beta)\KL(\pi^{\star}(\cdot|x)\|\hat{\pi}_{t+1}(\cdot|x))
-(\eta_{\theta}+\beta-C)\KL(\hat{\pi}_{t+1}(\cdot|x)\|\pi_{t}(\cdot|x))
\right.\\
&\qquad\qquad
\left.
-\left(\eta_{\theta}-C_1\right)\KL(\pi_{t}(\cdot|x)\|\hat{\pi}_t(\cdot|x))
+C_2\KL(\hat{\pi}_{t}(\cdot|x)\| \pi_{t-1}(\cdot|x))\right]\\
&\quad
-\left(\eta_{\lambda}-\frac{|\mathcal{H}| R_{\max}^2}{C}\right)\|\boldsymbol{\lambda}_{t}-\hat{\boldsymbol{\lambda}}_{t}\|_{2}^2+\frac{|\mathcal{H}| R_{\max}^2}{C}\|\hat{\boldsymbol{\lambda}}_{t}-\boldsymbol{\lambda}_{t-1}\|_{2}^2
+\eta_{\lambda}\|\boldsymbol{\lambda}^{\star}-\hat{\boldsymbol{\lambda}}_t\|_2^2
-\eta_{\lambda}\|\boldsymbol{\lambda}^{\star}-\hat{\boldsymbol{\lambda}}_{t+1}\|_2^2\\
&\quad
-\left(\eta_{\lambda}-|\mathcal{H}| R_{\max}^2\left(\frac{1}{2C_1}+\frac{1}{2C_2}\right)\right)\|\boldsymbol{\lambda}_{t}-\hat{\boldsymbol{\lambda}}_{t+1}\|_2^2
{+2\gap(\varepsilon_{\mathrm{approx}},p_{\min})},
\end{aligned}
\]
where terms of the same color can be combined. Recall \cref{eq:npg_L(pi*)-L(lambda*)>0} that $L(\pi^{\star},\boldsymbol{\lambda}_t) - L(\pi_t,\boldsymbol{\lambda}^{\star})\geq 0$. Substituting this into the above equation and rearranging the equation, we have
\begin{equation}
\label{eq:npg_Phi_t_contraction_0}
\begin{aligned}
&
(\eta_{\theta}+\beta)\E_{x\sim\mathcal{D}}\left[\KL(\pi^{\star}(\cdot|x)\|\hat{\pi}_{t+1}(\cdot|x))\right]
+(\eta_{\theta}+\beta-C)\E_{x\sim\mathcal{D}}\left[\KL(\hat{\pi}_{t+1}(\cdot|x)\|\pi_{t}(\cdot|x))\right]\\
&
+\eta_{\lambda}\|\boldsymbol{\lambda}^{\star}-\hat{\boldsymbol{\lambda}}_{t+1}\|_2^2
+\left(\eta_{\lambda}-|\mathcal{H}| R_{\max}^2\left(\frac{1}{2C_1}+\frac{1}{2C_2}\right)\right)\|\boldsymbol{\lambda}_{t}-\hat{\boldsymbol{\lambda}}_{t+1}\|_2^2\\
&\leq
\eta_{\theta}\E_{x\sim\mathcal{D}}\left[\KL(\pi^{\star}(\cdot|x)\|\hat{\pi}_t(\cdot|x))\right]
+C_2\E_{x\sim\mathcal{D}}\left[\KL(\hat{\pi}_{t}(\cdot|x)\| \pi_{t-1}(\cdot|x))\right]
+\eta_{\lambda}\|\boldsymbol{\lambda}^{\star}-\hat{\boldsymbol{\lambda}}_t\|_2^2
+\frac{|\mathcal{H}| R_{\max}^2}{C}\|\hat{\boldsymbol{\lambda}}_{t}-\boldsymbol{\lambda}_{t-1}\|_{2}^2
\\
&\quad
-\left(\eta_{\theta}-C_1\right)\E_{x\sim\mathcal{D}}\left[\KL(\pi_{t}(\cdot|x)\|\hat{\pi}_t(\cdot|x))\right]
-\left(\eta_{\lambda}-\frac{|\mathcal{H}| R_{\max}^2}{C}\right)\|\boldsymbol{\lambda}_{t}-\hat{\boldsymbol{\lambda}}_{t}\|_{2}^2
{+2\gap(\varepsilon_{\mathrm{approx}},p_{\min})}
\end{aligned}
\end{equation}
Note that for any $\delta,\theta>0$, we have
\begin{equation}
\label{eq:npg_split_lambdat-hatlambdat}
    \begin{aligned}
\|\boldsymbol{\lambda}_t-\hat{\boldsymbol{\lambda}}_t\|^2
&= \left\|\left(\boldsymbol{\lambda}_t-\hat{\boldsymbol{\lambda}}_{t+1}\right)+\left(\hat{\boldsymbol{\lambda}}_{t+1}-\boldsymbol{\lambda}^\star\right)+\left(\boldsymbol{\lambda}^\star-\hat{\boldsymbol{\lambda}}_t\right)\right\|^2\\
&\ge 
(1-\delta)\|\hat{\boldsymbol{\lambda}}_{t+1}-\boldsymbol{\lambda}^{\star}\|^2
+(1-\frac{1}{\delta})(1-\theta)\|\boldsymbol{\lambda}_t-\hat{\boldsymbol{\lambda}}_{t+1}\|^2
+(1-\frac{1}{\delta})(1-\frac{1}{\theta})\|\boldsymbol{\lambda}^{\star}-\hat{\boldsymbol{\lambda}}_t\|^2 \quad\text{(by \cref{lem:||a+b+c||_2^2_lower_bd})}.
\end{aligned}
\end{equation}
We set $\eta_{\lambda}>\frac{|\mathcal{H}| R_{\max}^2}{C}$. Substituting \cref{eq:npg_split_lambdat-hatlambdat} into the last term of the RHS of \cref{eq:npg_Phi_t_contraction_0} to get
\[
\begin{aligned}
&
(\eta_{\theta}+\beta)\E_{x\sim\mathcal{D}}\left[\KL(\pi^{\star}(\cdot|x)\|\hat{\pi}_{t+1}(\cdot|x))\right]
+(\eta_{\theta}+\beta-C)\E_{x\sim\mathcal{D}}\left[\KL(\hat{\pi}_{t+1}(\cdot|x)\|\pi_{t}(\cdot|x))\right]\\
&
{\color{red}+\eta_{\lambda}\|\boldsymbol{\lambda}^{\star}-\hat{\boldsymbol{\lambda}}_{t+1}\|_2^2}
+{\color{blue}\left(\eta_{\lambda}-|\mathcal{H}| R_{\max}^2\left(\frac{1}{2C_1}+\frac{1}{2C_2}\right)\right)\|\boldsymbol{\lambda}_{t}-\hat{\boldsymbol{\lambda}}_{t+1}\|_2^2}\\
&\leq
\eta_{\theta}\E_{x\sim\mathcal{D}}\left[\KL(\pi^{\star}(\cdot|x)\|\hat{\pi}_t(\cdot|x))\right]
+C_2\E_{x\sim\mathcal{D}}\left[\KL(\hat{\pi}_{t}(\cdot|x)\| \pi_{t-1}(\cdot|x))\right]
+{\color{cyan}\eta_{\lambda}\|\boldsymbol{\lambda}^{\star}-\hat{\boldsymbol{\lambda}}_t\|_2^2}
+\frac{|\mathcal{H}| R_{\max}^2}{C}\|\hat{\boldsymbol{\lambda}}_{t}-\boldsymbol{\lambda}_{t-1}\|_{2}^2
\\
&\quad
-\left(\eta_{\theta}-C_1\right)\E_{x\sim\mathcal{D}}\left[\KL(\pi_{t}(\cdot|x)\|\hat{\pi}_t(\cdot|x))\right]
{\color{red}-\left(\eta_{\lambda}-\frac{|\mathcal{H}| R_{\max}^2}{C}\right)(1-\delta)\|\hat{\boldsymbol{\lambda}}_{t+1}-\boldsymbol{\lambda}^{\star}\|^2}\\
&\quad
{\color{blue}-\left(\eta_{\lambda}-\frac{|\mathcal{H}| R_{\max}^2}{C}\right)
(1-\frac{1}{\delta})(1-\theta)\|\boldsymbol{\lambda}_t-\hat{\boldsymbol{\lambda}}_{t+1}\|^2}
{\color{cyan}-\left(\eta_{\lambda}-\frac{|\mathcal{H}| R_{\max}^2}{C}\right)(1-\frac{1}{\delta})(1-\frac{1}{\theta})\|\boldsymbol{\lambda}^{\star}-\hat{\boldsymbol{\lambda}}_t\|^2}\\
&\quad{+2\gap(\varepsilon_{\mathrm{approx}},p_{\min})},
\end{aligned}
\]
where terms of the same color can be combined. Rearranging the above equation, we have
\begin{equation}
\label{eq:npg_Phi_t_contraction_1}
\begin{aligned}
&
(\eta_{\theta}+\beta)\E_{x\sim\mathcal{D}}\left[\KL(\pi^{\star}(\cdot|x)\|\hat{\pi}_{t+1}(\cdot|x))\right]
+(\eta_{\theta}+\beta-C)\E_{x\sim\mathcal{D}}\left[\KL(\hat{\pi}_{t+1}(\cdot|x)\|\pi_{t}(\cdot|x))\right]\\
&
+\left(\eta_{\lambda}+\left(\eta_{\lambda}-\frac{|\mathcal{H}| R_{\max}^2}{C}\right)(1-\delta)\right)\|\boldsymbol{\lambda}^{\star}-\hat{\boldsymbol{\lambda}}_{t+1}\|_2^2\\
&
+\left(\eta_{\lambda}-|\mathcal{H}| R_{\max}^2\left(\frac{1}{2C_1}+\frac{1}{2C_2}\right)+\left(\eta_{\lambda}-\frac{|\mathcal{H}| R_{\max}^2}{C}\right)
(1-\frac{1}{\delta})(1-\theta)\right)\|\hat{\boldsymbol{\lambda}}_{t+1}-\boldsymbol{\lambda}_{t}\|_2^2\\
&\leq
\eta_{\theta}\E_{x\sim\mathcal{D}}\left[\KL(\pi^{\star}(\cdot|x)\|\hat{\pi}_t(\cdot|x))\right]
+C_2\E_{x\sim\mathcal{D}}\left[\KL(\hat{\pi}_{t}(\cdot|x)\| \pi_{t-1}(\cdot|x))\right]\\
&\quad
+\left(\eta_{\lambda}-\left(\eta_{\lambda}-\frac{|\mathcal{H}| R_{\max}^2}{C}\right)(1-\frac{1}{\delta})(1-\frac{1}{\theta})\right)\|\boldsymbol{\lambda}^{\star}-\hat{\boldsymbol{\lambda}}_t\|_2^2\\
&\quad
+\frac{|\mathcal{H}| R_{\max}^2}{C}\|\hat{\boldsymbol{\lambda}}_{t}-\boldsymbol{\lambda}_{t-1}\|_{2}^2
-\left(\eta_{\theta}-C_1\right)\E_{x\sim\mathcal{D}}\left[\KL(\pi_{t}(\cdot|x)\|\hat{\pi}_t(\cdot|x))\right]
{+2\gap(\varepsilon_{\mathrm{approx}},p_{\min})}\\
\end{aligned}
\end{equation}
Note the RHS of \cref{eq:npg_Phi_t_contraction_1} can be written as
\[
\begin{aligned}
&\frac{\eta_{\theta}}{\eta_{\theta}+\beta}(\eta_{\theta}+\beta)\E_{x\sim\mathcal{D}}\left[\KL(\pi^{\star}(\cdot|x)\|\hat{\pi}_t(\cdot|x))\right]
+\frac{C_2}{\eta_{\theta}+\beta-C}(\eta_{\theta}+\beta-C)\E_{x\sim\mathcal{D}}\left[\KL(\hat{\pi}_{t}(\cdot|x)\| \pi_{t-1}(\cdot|x))\right]\\
&\quad
+\frac{\eta_{\lambda}-\left(\eta_{\lambda}-\frac{|\mathcal{H}| R_{\max}^2}{C}\right)(1-\frac{1}{\delta})(1-\frac{1}{\theta})}{\eta_{\lambda}+\left(\eta_{\lambda}-\frac{|\mathcal{H}| R_{\max}^2}{C}\right)(1-\delta)}\left(\eta_{\lambda}+\left(\eta_{\lambda}-\frac{|\mathcal{H}| R_{\max}^2}{C}\right)(1-\delta)\right)\|\boldsymbol{\lambda}^{\star}-\hat{\boldsymbol{\lambda}}_t\|_2^2\\
&\quad
+\frac{\frac{|\mathcal{H}| R_{\max}^2}{C}}{\eta_{\lambda}-|\mathcal{H}| R_{\max}^2\left(\frac{1}{2C_1}+\frac{1}{2C_2}\right)+\left(\eta_{\lambda}-\frac{|\mathcal{H}| R_{\max}^2}{C}\right)
(1-\frac{1}{\delta})(1-\theta)}\times\\
&\qquad\left(\eta_{\lambda}-|\mathcal{H}| R_{\max}^2\left(\frac{1}{2C_1}+\frac{1}{2C_2}\right)+\left(\eta_{\lambda}-\frac{|\mathcal{H}| R_{\max}^2}{C}\right)
(1-\frac{1}{\delta})(1-\theta)\right)\|\hat{\boldsymbol{\lambda}}_{t}-\boldsymbol{\lambda}_{t-1}\|_{2}^2\\
&\quad-\left(\eta_{\theta}-C_1\right)\E_{x\sim\mathcal{D}}\left[\KL(\pi_{t}(\cdot|x)\|\hat{\pi}_t(\cdot|x))\right]
{+2\gap(\varepsilon_{\mathrm{approx}},p_{\min})}   
\end{aligned}
\]
Define $\Phi_t$ as the LHS of \cref{eq:npg_Phi_t_contraction_1}, i.e.,
\[
\begin{aligned}
 \Phi_{t+1}:=&
(\eta_{\theta}+\beta)\E_{x\sim\mathcal{D}}\left[\KL(\pi^{\star}(\cdot|x)\|\hat{\pi}_{t+1}(\cdot|x))\right]
+(\eta_{\theta}+\beta-C)\E_{x\sim\mathcal{D}}\left[\KL(\hat{\pi}_{t+1}(\cdot|x)\|\pi_{t}(\cdot|x))\right]\\
&
+\left(\eta_{\lambda}+\left(\eta_{\lambda}-\frac{|\mathcal{H}| R_{\max}^2}{C}\right)(1-\delta)\right)\|\boldsymbol{\lambda}^{\star}-\hat{\boldsymbol{\lambda}}_{t+1}\|_2^2\\
&
+\left(\eta_{\lambda}-|\mathcal{H}| R_{\max}^2\left(\frac{1}{2C_1}+\frac{1}{2C_2}\right)+\left(\eta_{\lambda}-\frac{|\mathcal{H}| R_{\max}^2}{C}\right)
(1-\frac{1}{\delta})(1-\theta)\right)\|\hat{\boldsymbol{\lambda}}_{t+1}-\boldsymbol{\lambda}_{t}\|_2^2
\end{aligned}
\]
If the following requirements are satisfied:
\begin{enumerate}
    \item Multipliers of all terms of LHS of \cref{eq:npg_Phi_t_contraction_1} are positive:
    \[
    \begin{aligned}
        &\eta_{\theta}+\beta>0, \\
        &\eta_{\theta}+\beta-C>0, \\
        &\eta_{\lambda}+\left(\eta_{\lambda}-\frac{|\mathcal{H}| R_{\max}^2}{C}\right)(1-\delta)>0, \\
        &\eta_{\lambda}-|\mathcal{H}| R_{\max}^2\left(\frac{1}{2C_1}+\frac{1}{2C_2}\right)+\left(\eta_{\lambda}-\frac{|\mathcal{H}| R_{\max}^2}{C}\right)(1-\frac{1}{\delta})(1-\theta)>0.\\
    \end{aligned}
    \]
    \item Multipliers of all terms of RHS of \cref{eq:npg_Phi_t_contraction_1} are positive:
    \[
    \begin{aligned}
        &\eta_{\theta}>0,\\
        &C_2>0,\\
        &\eta_{\lambda}-\left(\eta_{\lambda}-\frac{|\mathcal{H}| R_{\max}^2}{C}\right)(1-\frac{1}{\delta})(1-\frac{1}{\theta})>0,\\ &\frac{|\mathcal{H}| R_{\max}^2}{C}>0,\\
        &\eta_{\theta}-C_1>0.
    \end{aligned}
    \]
    \item Define 
    \[
    \begin{aligned}
      \rho: = &\max\left(
    \frac{\eta_{\theta}}{\eta_{\theta}+\beta},
    \frac{C_2}{ \eta_{\theta}+\beta-C},
    \frac{\eta_{\lambda}-\left(\eta_{\lambda}-\frac{|\mathcal{H}| R_{\max}^2}{C}\right)(1-\frac{1}{\delta})(1-\frac{1}{\theta})}{\eta_{\lambda}+\left(\eta_{\lambda}-\frac{|\mathcal{H}| R_{\max}^2}{C}\right)(1-\delta)},\right.\\
    &\qquad\qquad
    \left.\frac{\frac{|\mathcal{H}| R_{\max}^2}{C}}{\eta_{\lambda}-|\mathcal{H}| R_{\max}^2\left(\frac{1}{2C_1}+\frac{1}{2C_2}\right)+\left(\eta_{\lambda}-\frac{|\mathcal{H}| R_{\max}^2}{C}\right)
(1-\frac{1}{\delta})(1-\theta)}
    \right),      
    \end{aligned}
    \]
    then $\rho<1$.
\end{enumerate}
Then \cref{eq:npg_Phi_t_contraction_1} can be written as
\[
\begin{aligned}
 \Phi_{t+1} \leq& \rho\Phi_t -\left(\eta_{\theta}-C_1\right)\E_{x\sim\mathcal{D}}\left[\KL(\pi_{t}(\cdot|x)\|\hat{\pi}_t(\cdot|x))\right]{+2\gap(\varepsilon_{\mathrm{approx}},p_{\min})}
 \quad\text{(by the definition of $\rho$)}\\
 \leq& \rho\Phi_t {+2\gap(\varepsilon_{\mathrm{approx}},p_{\min})}
 \quad\text{(by $\eta-C_1>0$ and $\E_{x\sim\mathcal{D}}\left[\KL(\pi_{t}(\cdot|x)\|\hat{\pi}_t(\cdot|x))\right]>0$)}
\end{aligned}
\]
Iteratively apply the recursion, we have $\Phi_t\leq\rho^{t}\Phi_1$, where 
\[
\begin{aligned}
\Phi_1=&(\eta_{\theta}+\beta)\E_{x\sim\mathcal{D}}\left[\KL(\pi^{\star}(\cdot|x)\|\hat{\pi}_{1}(\cdot|x))\right]
+(\eta_{\theta}+\beta-C)\E_{x\sim\mathcal{D}}\left[\KL(\hat{\pi}_{1}(\cdot|x)\|\pi_{0}(\cdot|x))\right]\\
&
+\left(\eta_{\lambda}+\left(\eta_{\lambda}-\frac{|\mathcal{H}| R_{\max}^2}{C}\right)(1-\delta)\right)\|\boldsymbol{\lambda}^{\star}-\hat{\boldsymbol{\lambda}}_{1}\|_2^2\\
&
+\left(\eta_{\lambda}-|\mathcal{H}| R_{\max}^2\left(\frac{1}{2C_1}+\frac{1}{2C_2}\right)+\left(\eta_{\lambda}-\frac{|\mathcal{H}| R_{\max}^2}{C}\right)
(1-\frac{1}{\delta})(1-\theta)\right)\|\hat{\boldsymbol{\lambda}}_{1}-\boldsymbol{\lambda}_{0}\|_2^2.    
\end{aligned}
\]

Note that we initialize $\hat{\pi}_0$ having the same support set as $\pi_{\tref}$.  
Since we use a softmax parameterization over a finite action space, all policies have full support. Hence, the KL terms in $\Phi_1$ are finite and $\Phi_1$ is bounded.

Furthermore, we have
\[
\begin{aligned}
\E_{x\sim\mathcal{D}}\left[\KL(\pi^{\star}(\cdot|x)\|\hat{\pi}_{t}(\cdot|x))\right]
+\|\boldsymbol{\lambda}^{\star}-\hat{\boldsymbol{\lambda}}_{t}\|_2^2
\leq \rho^{t}\frac{\Phi_1}{\rho\min\left(\eta_{\theta}+\beta,\eta_{\lambda}+\left(\eta_{\lambda}-\frac{|\mathcal{H}| R_{\max}^2}{C}\right)(1-\delta)\right)}{\color{purple}+\frac{2(1-\rho^{t})\gap}{1-\rho}}
\end{aligned}
\]
and this shows the desired result.

\paragraph{Hyperparameters and Constants Selection}
Our next step is to choose hyperparameters $\eta_{\theta}$ and $\eta_{\lambda}$ as well as constants $C_1$, $C_2$, and $C$ to satisfy the requirements. 
For simplicity, with a little abuse of notations, we denote $h=|\mathcal{H}|$ and $R=R_{\max}$ in this parameter and constants selection section.
Let
\[
\eta_{\theta}=\eta_{\lambda}=\eta=3\sqrt{h}R,\quad
C_1=C_2=C=\sqrt{h}R,\quad
\frac{1}{2}<\delta<1,\quad
\frac{1}{2}<\theta<1.
\]
We will verify that this set of parameters satisfies the requirements.

\begin{enumerate}[start=0]
\item \textbf{Verification of }$\eta_{\lambda}>\frac{|\mathcal{H}| R_{\max}^2}{C}$. $\eta_\lambda = 3\sqrt{h}R\geq \sqrt{h}R=\frac{hR^2}{\sqrt{h}R}=\frac{hR^2}{C}$.

\item \textbf{Verifications that multipliers of all terms of LHS of \cref{eq:npg_Phi_t_contraction_1} are positive}.
    (1) Since $\eta_{\theta}>0$ and $\beta>0$, we have $\eta_{\theta}+\beta>0$. 
    (2) $\eta_{\theta}+\beta-C=\beta+2\sqrt{h}R>0$. 
    (3) Since $\eta_{\lambda}-hR^2/C=3\sqrt{h}R-hR^2/(\sqrt{h}R)=2\sqrt{h}R>0$ and $\delta<1$, we have $\left(\eta_{\lambda}-hR^2/C\right)(1-\delta)>0$. Hence $\eta_{\lambda}+\left(\eta_{\lambda}-hR^2/C\right)(1-\delta)>0$.
    (4) $\eta_{\lambda}-hR^2\left(1/(2C_1)+1/(2C_2)\right)+\left(\eta_{\lambda}-hR^2/C\right)(1-1/\delta)(1-\theta)
    =3\sqrt{h}R-hR^2/(\sqrt{h}R)+(2\sqrt{h}R-hR^2/(\sqrt{h}R))(1-1/\delta)(1-\theta)
    =\sqrt{h}R\left(2+\left(1-1/\delta\right)\left(1-\theta\right)\right)$.
    Since $1/2<\delta,\theta<1$, we have $-1<1-1/\delta<0$ and $0<1-\theta<1/2$, hence $-1/2<\left(1-1/\delta\right)\left(1-\theta\right)<0$. 
    Therefore, $2+\left(1-1/\delta\right)\left(1-\theta\right)>0$.

    \item \textbf{Verifications that the multipliers of all terms of the RHS of \cref{eq:npg_Phi_t_contraction_1} are positive}. 
    (1) $\eta_{\theta}>0$ by the definition of $\eta_{\theta}$. 
    (2) $C_2>0$ by the definition of $C_2$. 
    (3) $\eta_{\lambda}-\left(\eta_{\lambda}-hR^2/C\right)(1-1/\delta)(1-1/\theta)
    = 3\sqrt{h}R-(3\sqrt{h}R-hR^2/(\sqrt{h}R)(1-1/
    \delta)(1-1/\theta)
    =\sqrt{h}R(3-2(1-1/\delta)(1-1/\theta))$. 
    Since $1/2<\delta,\theta<1$, we have $-1<1-1/\theta<0$ and $-1<1-1/\delta<0$, hence $0<(1-1/\theta)(1-1/\delta)<1$.
    Therefore, we get $3-2(1-1/\delta)(1-1/\theta)>0$.
    (4) As $hR^2>0$ and $C=\sqrt{hR}>0$, we have $hR^2/C>0$.
    (5) $\eta_{\theta}-C_1=3\sqrt{hR}-\sqrt{h}R=2\sqrt{h}R>0$.

    \item 
    (1) Since $\eta_{\theta}>0$ and $\beta>0$, we have $\frac{\eta_{\theta}}{\eta_{\theta}+\beta}<1$.
    (2) We have $C_2/(\eta_{\theta}+\beta-C)=\sqrt{h}R/(3\sqrt{h}R-\beta-\sqrt{hR})<\sqrt{h}R/(3\sqrt{h}R-\sqrt{hR})=1/2$, where the inequality is because $\beta>0$.
    (3) Since $\eta_{\lambda}-hR^2/C=3\sqrt{hR}-hR^2/(\sqrt{h}R)=2\sqrt{h}R>0$, $(1-1/\delta)(1-1/\theta)>0$, and $1-\delta>0$, we have $\left(\eta_{\lambda}-hR^2/C\right)(1-1/\delta)(1-1/\theta)>0$ and $\left(\eta_{\lambda}-hR^2/C\right)(1-\delta)>0$, hence $\eta_{\lambda}-\left(\eta_{\lambda}-hR^2/C\right)(1-1/\delta)(1-1/\theta)<\eta_{\lambda}+\left(\eta_{\lambda}-hR^2/C\right)(1-\delta)$, and $(\eta_{\lambda}-\left(\eta_{\lambda}-hR^2/C\right)(1-1/\delta)(1-1/\theta))/(\eta_{\lambda}+\left(\eta_{\lambda}-hR^2/C\right)(1-\delta))<1$.
    (4)Plugging the parameters values into the last requirement, we have
    \[
\begin{aligned}
&\frac{{hR^2}/{C}}{\eta_{\lambda}-hR^2\left({1}/{2C_1}+{1}/{2C_2}\right)+(\eta_{\lambda}-{hR^2}/{C})(1-1/\delta)(1-\theta)}\\
&=\frac{\sqrt{h}R}{2\sqrt{h}R +2\sqrt{h}R(1-1/\delta)(1-\theta)}\\
&=\frac{1}{2+2(1-1/\delta)(1-\theta)}.
\end{aligned}
\]
As $1/2<\theta,\delta<1$, $-1/2<(1-1/\delta)(1-\theta)<0$. Therefore $1<2+2(1-1/\delta)(1-\theta)<2$ and ${1}/{2+2(1-1/\delta)(1-\theta)}<1$.
\end{enumerate}

If we further set $\delta=\theta=\frac{3}{4}$, then we can write $\rho$ and $\Phi_1$ as
 \begin{equation}
    \label{eq:npg_rho_valued} 
      \rho = \max\left(
    \frac{3\sqrt{\mathcal{H}}R_{\max}}{3\sqrt{\mathcal{H}}R_{\max}+\beta},
    \frac{\sqrt{\mathcal{H}}R_{\max}}{2\sqrt{\mathcal{H}}R_{\max}+\beta},
    \frac{50}{63}\right),      
 \end{equation}
 \begin{equation}
    \label{eq:npg_Phi_1_valued} 
\begin{aligned}
\Phi_1=&(3\sqrt{\mathcal{H}}R_{\max}+\beta)\E_{x\sim\mathcal{D}}\left[\KL(\pi^{\star}(\cdot|x)\|\hat{\pi}_{1}(\cdot|x))\right]
+(2\sqrt{\mathcal{H}}R_{\max}+\beta)\E_{x\sim\mathcal{D}}\left[\KL(\hat{\pi}_{1}(\cdot|x)\|\pi_{0}(\cdot|x))\right]\\
&
+\frac{7}{2}\sqrt{\mathcal{H}}R_{\max}\|\boldsymbol{\lambda}^{\star}-\hat{\boldsymbol{\lambda}}_{1}\|_2^2
+\frac{11}{6}\sqrt{\mathcal{H}}R_{\max}\|\hat{\boldsymbol{\lambda}}_{1}-\boldsymbol{\lambda}_{0}\|_2^2.    
\end{aligned}
 \end{equation}

\end{document}